\documentclass[accepted]{uai2025} 
                        
\usepackage[american]{babel}

\usepackage{natbib} 
\usepackage{mathtools} 
\usepackage{booktabs} 
\usepackage{tikz} 

\usepackage{amsmath}
\usepackage{amssymb}

\usepackage{amsthm}

\usepackage{xcolor}

\newcommand{\red}[1]{\textcolor{black}{#1}}

\usepackage{algorithm}
\usepackage{algpseudocode}
\usepackage{hyperref}

\usepackage{dsfont}
\theoremstyle{plain}
\newtheorem{theorem}{Theorem}[section]

\newtheorem{lemma}[theorem]{Lemma}

\theoremstyle{definition}

\newtheorem{assumption}[theorem]{Assumption}
\theoremstyle{remark}

\DeclareMathOperator*{\argmax}{arg\,max}

\newcommand{\thetac}{\theta^{\mathrm{c}}}
\newcommand{\Deltac}{\Delta^{\mathrm{c}}}
\newcommand{\Deltar}{\Delta^{\mathrm{r}}}
\newcommand{\thetar}{\theta^{\mathrm{r}}}
\newcommand{\muc}{\mu^{\mathrm{c}}}
\newcommand{\mur}{\mu^{\mathrm{r}}}

\newcommand{\truethetac}{\theta^{\mathrm{c}}}
\newcommand{\truethetar}{\theta^{\mathrm{r}}}
\newcommand{\hthetact}{\hat{\theta}_t^{\mathrm{c}}}
\newcommand{\hthetart}{\hat{\theta}_t^{\mathrm{r}}}
\newcommand{\EE}{\mathbb{E}} 
\newcommand{\PP}{\mathbb{P}} 
\newcommand{\cbr}[1]{\left\{#1\right\}}
\newcommand{\leo}{o} 
\newcommand*{\one}{\mathds{1}}
\def\cd{\cdot}
\allowdisplaybreaks

\title{Asymptotically Optimal \\ Linear Best Feasible Arm Identification with Fixed Budget}

\author[1]{\href{mailto:<jiebian@u.nus.edu>?Subject=Your UAI 2025 paper}{Jie Bian}{}}
\author[2,1]{Vincent Y. F. Tan}
\affil[1]{%
    Department of Electrical and Computer Engineering\\
    National University of Singapore\\
    Singapore
}
\affil[2]{%
    Department of Mathematics\\
    National University of Singapore\\
    Singapore
}  
  
\begin{document}
\maketitle

\begin{abstract}
The challenge of identifying the best feasible arm within a fixed budget has attracted considerable interest in recent years. However, a notable gap remains in the literature: the exact exponential rate at which the error probability approaches zero has yet to be established, even in the relatively simple setting of $K$-armed bandits with Gaussian noise. 
In this paper, we address this gap by examining the problem within the context of linear bandits. We introduce a novel algorithm for best feasible arm identification that guarantees an exponential decay in the error probability. Remarkably, the decay rate—characterized by the exponent—matches the theoretical lower bound derived using information-theoretic principles. 
Our approach leverages a posterior sampling framework embedded within a game-based sampling rule involving a min-learner and a max-learner. This strategy shares its foundations with Thompson sampling, but is specifically tailored to optimize the identification process under fixed-budget constraints. 
Furthermore, we validate the effectiveness of our algorithm through comprehensive empirical evaluations across various problem instances with different levels of complexity. The results corroborate our theoretical findings and demonstrate that our method outperforms several benchmark algorithms in terms of both accuracy and efficiency.
\end{abstract}

\vspace{-.1in}
\section{Introduction}\vspace{-.1in}
\label{Introduction}
\red{Many real-world decision-making systems, such as recommendation platforms, clinical trials, and autonomous agents, must make selections under feasibility constraints that reflect resource limitations, safety thresholds, or business requirements. For instance, a system may wish to recommend actions that maximize utility while adhering to pricing limits, fairness constraints, or operational risks. These constraints are often modeled as linear inequalities over contextual features, with parameters—such as cost or risk scores—that are unknown and must be inferred from data. Moreover, these estimates are inherently noisy due to model approximation, user behavior variability, and environmental uncertainty. We formalize this as follows: an arm \( z \in \mathbb{R}^d \) is considered feasible if its expected cost satisfies \( \langle \thetac, z \rangle \leq \tau \), where \( \thetac \in \mathbb{R}^d \) is an unknown cost vector and \( \tau \in \mathbb{R} \) is a known threshold. The goal is to identify the best feasible arm—the one with the highest expected reward among all feasible options—based on noisy observations of both reward and cost. This setting captures a wide range of constrained decision-making problems and offers a tractable framework for inference and optimization. It can be viewed as a constrained variant of the best arm identification (BAI) problem in multi-armed bandits (MAB). In classical BAI, the objective is to identify the arm with the highest expected reward under a fixed budget or confidence. However, feasibility constraints introduce new challenges, especially under a fixed sampling budget where confidence intervals cannot be reliably maintained. This motivates the development of new algorithms that jointly reason over reward and constraint uncertainty within limited resources.
}

The  MAB  problem is a central challenge in machine learning and statistical decision theory (\cite{lattimore2020bandit}). Here, a decision-maker, often referred to as the player, repeatedly selects one of several available arms. Each arm corresponds to a distinct action, and upon selection, the player receives a reward drawn from an unknown probability distribution associated with that arm.

The MAB problem can be categorized into two main settings: regret minimization and best arm identification. In regret minimization, the goal is to maximize cumulative reward, or equivalently, minimize cumulative regret, defined as the gap between the algorithm’s cumulative reward and that of always pulling the best arm at each time step. In contrast, best arm identification focuses on maximizing the probability of correctly identifying the arm with the highest expected reward (\cite{maron1997racing}), either within a predetermined budget (\cite{gabillon2012best, komiyama2022minimax}) or with a specified level of confidence (\cite{kalyanakrishnan2012pac, kuroki2020polynomial}).

MAB problems have broad applications, including online advertising, clinical trials, and recommendation systems (\cite{kuleshov2014algorithms,bouneffouf2019survey}). Despite their utility, these problems present key challenges, such as managing the exploration-exploitation trade-off, scalability across large action spaces, and incorporating contextual information.

\vspace{-.1in}
\subsection{Regret Minimization}\vspace{-.1in}

Regret minimization focuses on reducing the cumulative difference between the rewards obtained by the algorithm and those of the optimal strategy, requiring a careful balance between exploration and exploitation. In the classic $K$-armed bandit problem, algorithms such as UCB1 \cite{auer2002finite} and Thompson Sampling \cite{dc35850b-2ca1-314f-9e0d-470713436b17} are widely used. UCB1 selects arms based on an upper confidence bound that prioritizes arms with higher uncertainty until sufficient data is gathered, while Thompson Sampling maintains a posterior distribution over rewards and selects arms probabilistically based on reward samples. In the linear bandit setting, where rewards are linear functions of a known context vector with unknown parameters, LinUCB \cite{abbasi2011improved} extends UCB1 by incorporating parameter uncertainty into the confidence bounds, and LinTS \cite{agrawal2013thompson} applies a similar approach by sampling from the posterior distribution of the linear parameters to select arms.

\vspace{-.1in}
\subsection{Fixed Budget Best Arm Identification}\vspace{-.1in}

The BAI problem seeks to identify the arm with the highest expected reward, either under a fixed budget of arm pulls or by minimizing the number of trials required to achieve a specified level of confidence. Unlike regret minimization, this objective places a greater emphasis on exploration to ensure reliable identification of the optimal arm.

In the finite fixed budget setting, two algorithms are particularly noteworthy: Successive Rejects and Sequential Halving. The Successive Rejects algorithm \cite{audibert2010best} iteratively eliminates the arm with the lowest empirical mean reward after allocating a predefined number of trials to all arms in each phase. This process continues until only the best arm remains. The Sequential Halving algorithm \cite{karnin2013almost} adopts a similar phase-based approach but eliminates half of the remaining arms with the lowest observed rewards after each round. Subsequent rounds concentrate trials on the more promising arms, improving the efficiency of arm elimination.

Within the infinite fixed-budget setting, the primary objective of algorithms is to asymptotically achieve an exponential rate for the upper bound on regret as the time horizon $T$ approaches infinity. Algorithms in this category often prioritize efficient exploration-exploitation trade-offs to maximize long-term performance. One notable example is the Top-Two Thompson Sampling (TTTS) algorithm, introduced by \citet{russo2016simple}. The Top-Two Thompson Sampling algorithm introduces a critical tuning parameter, denoted by $\beta$. Achieving optimal rate of posterior convergence under this algorithm requires that $\beta$ precisely match the optimal allocation rate associated with the best arm. Only under this condition does the algorithm attain the optimal rate of posterior convergence. This dependency underscores the importance of parameter selection and adaptive tuning mechanisms for Bayesian algorithms operating in infinite horizon settings. Recently, new algorithm (\cite{li2024optimal}) called Pure Exploration with Projection-Free Sampling (PEPS) was developed; this algorithm does not need tuning of $\beta$ and achieves the optimal rate of posterior convergence when the parameter set is convex and bounded (\cite{kone2024pareto}).

\vspace{-.1in}
\subsection{Related Work}\vspace{-.1in}

The literature on MABs  has evolved significantly across three key dimensions: (1) unconstrained vs.\ constrained settings, (2) regret minimization vs.\ best arm identification objectives, and (3) fixed-confidence vs.\ fixed-budget paradigms. Our work bridges a critical gap in constrained BAI under fixed budgets by proposing a tuning-free algorithm with matching asymptotic bounds. Below, we contextualize our contributions within these axes of research.

\textbf{Unconstrained Best Arm Identification.} Early work in linear bandits focused on unconstrained settings. Recent breakthroughs, such as the min-max game framework by \citet{li2024optimal}, achieve minimax optimal error rates without tuning parameters. For structured bandits, \citet{azizi2021fixed} proposed an early fixed-budget method for linear BAI, though their theoretical guarantees fall short of tightness. \citet{yang2022minimax} offered a more refined analysis, but their upper bound retains a constant suboptimality gap compared to the lower bound. These advances, however, do not account for resource constraints.

\textbf{Constrained Bandits: Regret Minimization.} In constrained regret minimization, confidence intervals for rewards and costs drive algorithm design. Works such as \citet{amani2019linear, moradipari2021safe} and kernelized extensions \citep{zhou2022kernelized} ensure sublinear regret. Contextual settings further refine this approach \citep{pacchiano2024contextual}. While effective for minimizing cumulative regret, these methods are ill-suited for identifying the optimal feasible arm under fixed budgets.

\textbf{Constrained Best Arm Identification.} Under fixed-confidence settings, safety constraints dominate. \citet{katz2019top, camilleri2022active} maintain conservative safe arm sets via confidence intervals, while \citet{wang2022best, shang2023price} address pricing-specific constraints. These methods prioritize statistical efficiency over budget allocation, limiting their applicability to fixed-budget scenarios. \citet{Hou23} proposed an  optimal BAI algorithm that accounts for variance constraints. 

\textbf{Fixed-Budget Challenges.} The fixed-budget setting demands optimal resource allocation. Early efforts in \( K \)-armed bandits \citep{faizal2022constrained} and combinatorial settings \citep{tang2024pure} lack matching error bounds. Recent progress in linear bandits \citep{yang2022minimax} achieves minimax optimality, yet these ignore constraints entirely. A key unresolved challenge lies in balancing cost-awareness with asymptotic optimality under limited samples.

\textbf{Infinite-Budget and Parameter Sensitivity.} For infinite horizons, \citet{yang2025stochastically} proposed Top-Two Thompson Sampling with \( \beta \)-optimality, requiring careful parameter tuning. Parallel advances in unconstrained settings \citep{li2024optimal} eliminate tuning parameters via posterior sampling, but extending this to constrained problems remains open.

\textbf{Feasibility vs.\ Optimization.} Prior work such as \citet{katz2018feasible} identifies entire feasible arm sets, diverging from our goal of pinpointing the best feasible arm. Similarly, Pareto set identification \citep{kone2024pareto} focuses on multi-objective trade-offs rather than constrained single-objective optimization.

\vspace{-.1in}
\subsection{Main Contributions}\vspace{-.1in}
Our contributions in this paper are fourfold:

\begin{enumerate}
    \item We propose a novel algorithm BLFAIPS that achieves matching upper and lower bounds for the best feasible arm identification problem under a fixed budget. This is the first algorithm in this domain to achieve optimality in the exponent of the error probability.

    \item We derive two distinct forms of the exponential rate in the lower bound (Theorems~\ref{Thm:lb1} and~\ref{Thm:equ_lowerboundterm}). One form is derived from a Bayesian perspective, while the other arises from a frequentist viewpoint. We further prove the equivalence of these two formulations, thus unifying both theoretical perspectives.

    \item We improve the structure introduced in \cite{li2024optimal} by replacing the exponential weights algorithm with AdaHedge and eliminating the need for the doubling trick. Additionally, we relax the assumption that the reward is bounded, which was a requirement in their theoretical guarantees and proofs.

    \item We conduct empirical studies on synthetic and real-world datasets to evaluate the performance of our algorithm against several baseline methods on representative problem instances. Our results demonstrate that the proposed algorithm consistently outperforms or remains competitive with other state-of-the-art algorithms across different scenarios.

\end{enumerate}

\section{Problem Setup}
We consider the BAI  problem for linear bandits in the fixed budget setting with a budget $T$. In contrast to the conventional setting in which there are no constraints, we impose a constraint on the permissible arm to be selected at the end of the horizon $T$. The precise problem setting is as follows. There are two finite sets of arms -- a {\em training set} $\mathcal{X} \subset \mathbb{R}^d$ and a {\em testing set} $\mathcal{Z} \subset \mathbb{R}^d$. At each time $t \in [T]$, based on the learner's past selections and their corresponding outcomes, the learner chooses an arm $X_t \in \mathcal{X}$ and the environment reveals noisy versions of a reward (r) and a cost (c) which are respectively denoted as 
\begin{align}
    Y^{\mathrm{r}}_t &= \langle \theta^{\mathrm{r}}, X_t\rangle+\epsilon_t ,\quad \mbox{and} \label{eqn:reward}\\
    Y^{\mathrm{c}}_t &= \langle \theta^{\mathrm{c}}, X_t\rangle+\eta_t, \label{eqn:cost} 
\end{align}
where $\varepsilon_t$ and  $\eta_t$ are zero mean independent Gaussian noises with variances $\sigma^2$ and $\gamma^2$ respectively. In Eqns.~\eqref{eqn:reward} and~\eqref{eqn:cost}, $\theta^{\mathrm{r}} \in \Theta^{\mathrm{r}}$ and  $ \theta^{\mathrm{c}} \in \Theta^{\mathrm{c}}$ are {\em unknown} $d$-dimensional vectors. The decision $X_t $ at each time $t$ is based on the history $\mathcal{H}_{t-1} :=  \{  (X_s, Y_s^{\mathrm{r}}, Y_s^{\mathrm{c}} )\}_{s=1}^{t-1}$. Our goal is to select the arm that maximizes the mean reward subject to a constraint on the cost, measured according to the test vectors in $\mathcal{Z}$. More precisely, we denote $\mu^{\mathrm{r}}(z):=\langle \theta^{\mathrm{r}}, z \rangle$ and $\mu^{\mathrm{c}}(z):=\langle \theta^{\mathrm{c}}, z \rangle$ as the expected reward and cost of the testing arm vector $z\in \mathcal{Z}$. We define the {\em best feasible arm as}
\begin{align}
    z^*:= \argmax_{ z\in \mathcal{Z} :\mu^{\mathrm{c}}(z)\le\tau   }  \mu^{\mathrm{r}}(z), \nonumber
\end{align}
where $\tau\in\mathbb{R}$ is the known {\em threshold} on the cost.  At the end of the horizon, the learner recommends an arm $z_{\mathrm{out}} \in \mathcal{Z}$ which is its best guess of $z^*$. 
For ease of exposition later, we define the following three sets. 
\begin{itemize}[itemsep=0pt, parsep=0pt, topsep=0pt]
    \item Suboptimal arm set \\$\mathcal{S} := \{ z\in \mathcal{Z}: \mu^{\mathrm{r}}(z)\le \langle\theta^{\mathrm{r}}, z^* \rangle   \}$;
    \item Superoptimal arm set \\$\overline{\mathcal{S}} := \{ z\in \mathcal{Z}: \mu^{\mathrm{r}}(z)> \langle\theta^{\mathrm{r}}, z^* \rangle   \}$;
    \item Feasible arm set \\ $\mathcal{F} := \{ z \in \mathcal{Z}: \mu^{\mathrm{c}}(z)\le \tau \}$;
    \item Infeasible arm set  \\$\overline{\mathcal{F}} := \{ z \in \mathcal{Z}: \mu^{\mathrm{c}}(z)> \tau \}$
\end{itemize}
For convenient, denote $\mathcal{A}_1:= \overline{\mathcal{F}}\cap \overline{\mathcal{S}}, \mathcal{A}_2:= {\mathcal{F}}\cap {\mathcal{S}}, \mathcal{A}_3:= \overline{\mathcal{F}}\cap {\mathcal{S}} $. For any $z\in \mathcal{Z}$, define $\overline{\Theta}_z \subset \Theta^\mathrm{r} \times \Theta^\mathrm{c}$ such that for any $(\theta_1, \theta_2) \in \overline{\Theta}_z$, the arm $z$ is not the best feasible arm while $\thetar=\theta_1, \thetac=\theta_2$.


Furthermore we have the following assumptions:
\begin{assumption}
\label{ass1}
    $\Theta^\mathrm{r}, \Theta^\mathrm{c}$ are both bounded and closed with non-empty interiors such that $\max_{\theta \in \Theta^\mathrm{r}} \lVert \theta \rVert_2 \le R_1, \max_{\theta \in \Theta^\mathrm{c}} \lVert \theta \rVert_2 \le R_2~.$
\end{assumption}
\begin{assumption}
\label{ass2}
    The training and testing sets are bounded such that $\max_{x \in \mathcal{X}} \lVert x \rVert_2 \le L_1, \max_{z \in \mathcal{Z}} \lVert z \rVert_2 \le L_2~.$
\end{assumption}
\begin{assumption}
\label{ass3}
    The best feasible arm $z_1$ is unique and $\text{span}(\mathcal{Z}) \subset \text{span}(\mathcal{X})~.$
\end{assumption}

\textbf{Notations: }We define $\lVert x \rVert_A:= \sqrt{  x^\top A x}$ for any $d$-dimensional column vector $x$ and positive definite $d \times d$ matrix $A$. Given a finite set $\mathcal{X} \subset \mathbb{R}^d$, we define the  probability simplex as 
$\Delta_{\mathcal{X}} := \big\{ \lambda \in \mathbb{R}_{\geq 0}^{|\mathcal{X}|} : \sum_{i=1}^{|\mathcal{X}|} \lambda_i = 1 \big\}.$ For any $\lambda \in \Delta_{\mathcal{X}}$, we define the matrix $A(\lambda)$ as $A(\lambda) := \sum_{x \in \mathcal{X}} \lambda_x x x^\top.$ Denote two sequences of positive real numbers $\cbr{a_n}$ and $\cbr{b_n}$ to be equal to first order in the exponent as   $a_n\stackrel{\cdot}{=}b_n$ if $\lim_{n\to \infty} \frac{1}{n}\log\frac{a_n}{b_n}=0$.

\section{Best Linear Feasible Arm Identification with Posterior Sampling (BLFAIPS)}
\begin{algorithm}
\begin{algorithmic}[1]
\Require Finite set of arms $\mathcal{X} \subset \mathbb{R}^d$, $\mathcal{Z} \subset \mathbb{R}^d$, time horizon $T$, posterior reward and cost distribution variance parameter $\eta_\mathrm{r}$, $\eta_\mathrm{c}$, shrinking rate of exploration $\alpha=\frac{1}{4}$, AdaHedge initial learning rate $\eta_1:=\infty$, cost threshold constant $\tau \in \mathbb{R}$, initial cumulative mixability gap $\Delta_0=0$.
\State \textbf{Define} $\lambda^G = \arg\min_{\lambda \in \Delta_{\mathcal{X}}} \max_{x \in \mathcal{X}} \lVert x\rVert^2_{\mathcal{A}(\lambda)^{-1}}$
\State $\lambda_1 = \frac{1}{|\mathcal{X}|} \mathbf{1}$
\State Initialize $V_0 = I, S_0 = 0, p_1 = \mathcal{N}(0,V_0), \hat{\theta}_1$ arbitrarily
\For{$t = 1, 2, \dots, T$}
    \State $\gamma_t = t^{-\alpha}$
    \Comment{Posterior Sampling}
    \State $\hat{\mathcal{F}} := \{ z\in \mathcal{Z} \mid z^\top \hat{\theta}^{\mathrm{c}}_t \le \tau \}$
    \If{$|\hat{\mathcal{F}}| \neq \emptyset$}
        \State Compute $\hat{z}_t = \arg\max_{z \in \hat{\mathcal{F}}} z^\top \hat{\theta}^{\mathrm{r}}_t$
    \Else
        \State Uniformly sample $\hat{z}_t$ from $\mathcal{Z}$
    \EndIf
    \State Sample $(\theta^{\mathrm{r}}_t,  \theta^{\mathrm{c}}_t) \sim \mathcal{N}(\hat{\theta}^{\mathrm{r}}_t, \eta_r^{-1} V_{t-1}^{-1}) \otimes \mathcal{N}(\hat{\theta}^{\mathrm{c}}_t, \eta_c^{-1} V_{t-1}^{-1}) \mid \overline{\Theta}_{\hat{z}_t}$
    \State Sample $X_t \sim \lambda_t$ where $\lambda_t = (1-\gamma_t) \lambda_t + \gamma_t \lambda^G$
    \State Observe $y^{\mathrm{r}}_t = \langle \theta^{\mathrm{r}}_t, X_t \rangle + \epsilon_t$ where $\epsilon_t \sim \mathcal{N}(0,\sigma^2)$
    \State Observe $y^{\mathrm{c}}_t = \langle \theta^{\mathrm{c}}_t, X_t \rangle + \gamma_t$ where $\gamma_t \sim \mathcal{N}(0,\gamma^2)$
    \Comment{AdaHedge}
    \State Receive the loss vector $l_t \in \mathbb{R}^{|\mathcal{X}|}$ where 
     $l_{t,x} := -\Big( \frac{\lVert \theta_t^{\mathrm{r}} - \hat{\theta}_t^{\mathrm{r}} \rVert^2_{x x^\top}}{\sigma^2} + \frac{\lVert \theta_t^{\mathrm{c}} - \hat{\theta}_t^{\mathrm{c}} \rVert^2_{x x^\top}}{\gamma^2} \Big)$
    \State Compute the Hedge loss $h_t = \sum_{i=1}^{|\mathcal{X}|} \lambda_{t,i} l_{t,i}$
    \State Update $V_t = V_{t-1} + X_t X_t^\top$
    \State $S^\mathrm{c}_t = S^\mathrm{c}_{t-1} + x_t y^\mathrm{c}_t$, $S^\mathrm{r}_t = S^\mathrm{r}_{t-1} + x_t y^\mathrm{r}_t$
    \State $\hat{\theta}^\mathrm{r}_{t+1} = V_t^{-1} S^\mathrm{r}_t$, $\hat{\theta}^\mathrm{c}_{t+1} = V_t^{-1} S^\mathrm{c}_t$
    \State Update $\lambda_{t+1} \propto \lambda_t e^{-\eta_t l_t}$
    \State Compute mixed loss $$m_t := -\frac{1}{\eta_t} \log\left( \sum_{i \in \mathcal{X}} \lambda_{t,i} e^{-\eta_t l_{t,i}} \right)$$
    \State Compute the mixability gap $\delta_t := h_t - m_t$, $\Delta_t = \Delta_{t-1} + \delta_t$
    \State Update the AdaHedge learning rate $\eta_{t+1} = \frac{\log(|\mathcal{X}|)}{\Delta_t}$
\EndFor
\State Sample $(\theta^{\mathrm{r}}_{T+1}, \theta^{\mathrm{c}}_{T+1}) \sim \mathcal{N}(\hat{\theta}^\mathrm{r}_{T+1}, \sigma^2 V_T^{-1}) \otimes \mathcal{N}(\hat{\theta}^\mathrm{c}_{T+1}, \gamma^2 V_T^{-1}) \mid \Theta$
\State $\hat{\mathcal{F}} \gets \{ z \in \mathcal{Z} \mid z^\top \theta^{\mathrm{c}} \le \tau \}$
\If{$|\hat{\mathcal{F}}| \neq \emptyset$}
    \State Compute $\hat{z}_t \gets \arg\max_{z \in \hat{\mathcal{F}}} z^\top \hat{\theta}^{\mathrm{r}}_{T+1}$
\Else
    \State Uniformly sample $\hat{z}_t$ from $\mathcal{Z}$
\EndIf
\State \textbf{Output:} $\hat{z}_{\mathrm{out}} \gets \arg\max_{z \in \hat{\mathcal{F}}} z^\top \theta^{\mathrm{r}}$
\end{algorithmic}
\caption{Best Linear Feasible arm identification with Posterior Sampling (BLFAIPS)}
 \label{alg:1}
\end{algorithm}

We introduce a novel approach for linear feasible arm identification utilizing posterior sampling. This algorithm, called BLFAIPS (and whose pseudocode is in Algorithm~\ref{alg:1}), employs posterior sampling for the min-learner and AdaHedge for the max-learner. We call them  {\em min-learner} and {\em max-learner} since by Sion's minimax theorem, the exponential rate $\Gamma$ in the lower bound, as we will see in Theorem~\ref{Thm:equ_lowerboundterm}, can be represented by:
\begin{align*}
    &\Gamma:=\max_{w\in \Delta_\mathcal{X}}\inf_{(\theta_1, \theta_2)\in\overline{\Theta}_{z^*}}\! \frac{1}{2} \bigg( \frac{\lVert\theta_1\!-\!\thetar\rVert^2_{A(w)^{}}}{\sigma^2}\!+\!\frac{\lVert\theta_2\!-\!\thetac\rVert^2_{A(w)^{}}}{\gamma^2}\bigg)
    \\&=\max_{w\in \Delta_\mathcal{X}}\min_{p\in{\Delta(\overline{\Theta}_{z^*})}} \frac{1}{2} \EE\bigg[  \frac{\lVert\theta_1\!-\!\thetar\rVert^2_{A(w)^{}}}{\sigma^2}\!+\!\frac{\lVert\theta_2\!-\!\thetac\rVert^2_{A(w)^{}}}{\gamma^2}\bigg]
    \\&=\min_{p\in{\Delta(\overline{\Theta}_{z^*})}} \max_{w\in \Delta_\mathcal{X}} \frac{1}{2}\EE \bigg[ \frac{\lVert\theta_1\!-\!\thetar\rVert^2_{A(w)^{}}}{\sigma^2}\!+\!\frac{\lVert\theta_2\!-\!\thetac\rVert^2_{A(w)^{}}}{\gamma^2}\bigg],
\end{align*}
where the expectations above are with respect to $(\theta_1,\theta_2)\sim p$.
\red{The quantity \( \Gamma \) captures the hardness of the best feasible arm identification problem—it characterizes the exponential rate at which the optimal arm can be distinguished from all other (possibly infeasible) arms under reward and cost uncertainty. To match this fundamental limit, our algorithm is designed to mirror the structure of this bound: the max-learner adaptively selects sampling distributions to maximize distinguishability, while the min-learner simulates adversarial parameters that minimize it. Since both components achieve sublinear regret over time (though our focus is not on regret), the overall algorithm converges to the optimal error exponent in the long run.} A comprehensive description of the proposed algorithm is provided below.

\textbf{Initialization}
In Line 1, we define the exploration distribution using the G-optimal design, which is specifically designed to determine the weights of the arms by minimizing the maximum confidence over the arms. This design ensures that the exploration process effectively gathers sufficient information about the unknown linear coefficients, $\thetar$ and $\thetac$. In Lines 2-3, we initialize the weight for the AdaHedge algorithm with a uniform distribution, set the covariance matrix to the identity matrix, and initialize $S_0$ as the $d \times 1$ zero matrix, following the standard initialization procedure for Ridge regression.

\textbf{Min-learner}
From Lines 7 to 16, we perform posterior sampling for the min-learner. In Line 7, we first compute the empirical feasible arm set, and if this set is empty, we randomly select the best empirical feasible arm. If the set is not empty, we proceed in Line 9 to compare the empirical rewards to identify the empirical best feasible arm. In Line 13, we then perform posterior sampling of $(\thetar_t, \thetac_t)$ from the posterior distribution $\mathcal{N}(\hat{\theta}^{\mathrm{r}}_t, \eta_r^{-1} V_{t-1}^{-1}) \otimes \mathcal{N}(\hat{\thetac_t}, \eta_c^{-1} V_{t-1}^{-1})$ within the space $\overline{\Theta}_{\hat{z}_t}$, ensuring that $\hat{z}_t$ is not the best feasible arm. For the learning parameters, we choose $\eta_r:=\frac{\eta}{\sigma^2},\eta_c:=\frac{\eta}{\gamma^2}$ where $\eta:=\min \cbr{\frac{\sigma^2}{8L^2R_1^2},\frac{\gamma^2}{8L^2R_2^2}}.$
Following this, in Line~14, we draw the arm under the distribution $\lambda_t$, which combines the AdaHedge distribution and the G-optimal distribution, with the proportion of the G-optimal distribution decreasing exponentially over time. Finally, in Lines 15-16, we observe the reward and cost corresponding to the selected arm, which are utilized for updating the ridge estimator in Line 20.

\textbf{Max-learner}
From Lines 18 to 24, we employ AdaHedge for the max-learner instead of the traditional exponential weighting algorithm, enabling the learning rate parameter to automatically adapt to the environment without the need for further tuning. Specifically, in Line 18, we compute the loss for each arm \( x \) at time step \( t \) as 
$$
l_{t, x} := -\left(\frac{\lVert \thetar_t - \hat{\theta}^{\mathrm{r}}_t \rVert^2_{x x^\top}}{\sigma^2} + \frac{\lVert \thetac_t - \hat{\theta}^{\mathrm{c}}_t \rVert^2_{x x^\top}}{\gamma^2}\right),
$$
where the first term represents the loss from the reward and the second term represents the loss from the cost. This choice of loss function is derived from our lower bound analysis in Section~\ref{sec:4}. Following this, in the subsequent lines, we apply the AdaHedge algorithm to determine the distribution for the next time step and update the necessary parameters to calculate the learning rate for the next round.

\textbf{Recommendation}
From Lines 26 to 33, upon accumulating the relevant information, the algorithm recommends the arm that is most likely to be the optimal feasible arm. This procedure adheres to the principles of posterior sampling, closely following the approach outlined in Lines 7 to 13.

\red{\textbf{Min-max intuition} Our algorithm is guided by a game-theoretic interpretation of the information-theoretic lower bound. Intuitively, this corresponds to a zero-sum game between two players: the learner (max-player) selects a sampling distribution over arms to maximize information gain, while the adversary (min-player) chooses alternative parameters that are hardest to distinguish from the true ones. This formulation focuses exploration on the most informative regions of the parameter space. The structure mirrors the lower bound in Theorem~4.2, where the exponent is expressed as a max–min over sampling distributions and perturbations within the feasible set \( \overline{\Theta}_{z^*} \). We include a diagram in the appendix to illustrate this interaction and support readers less familiar with game-based bandit formulations.
}

\textbf{Novelty of our algorithm}
Compared to the linear best arm identification algorithm without constraints proposed by \citet{li2024optimal}, our algorithm, BLFAIPS, incorporates several key innovations:

First, we replace the exponential weights algorithm with AdaHedge (Lines 17-24). 
\red{By utilizing AdaHedge, our algorithm does not require the doubling trick employed in~\citet{li2024optimal} to achieve anytime adaptivity. Unlike standard Hedge-style algorithms that need manually scheduled learning rates or restarts, AdaHedge automatically adapts to the loss sequence, which leads to smoother convergence and more effective use of the fixed sampling budget.}

Second, we define a novel loss function (Line 18), derived from our lower bound theorem (Theorem~\ref{Thm:equ_lowerboundterm}). This loss function generalizes the one used in \citet{li2024optimal}, allowing for greater flexibility and improved theoretical guarantees within constrained settings.

Finally, in Line 13, our algorithm incorporates constraints directly into the posterior sampling process. This enhancement contrasts with the min-learner approach of \citet{li2024optimal}, enabling our algorithm to better account for feasibility constraints during decision-making.

These improvements collectively enhance both the theoretical and practical performance of our algorithm under resource-constrained scenarios.

\section{Theoretical Guarantees}
\label{sec:4}
In this section, we  state an information-theoretic lower bound on the error probability.  We provide intuitive explanations for each of the terms in the bound.

\vspace{-.1in}
\subsection{Lower Bound}\vspace{-.1in}
\begin{theorem} \label{Thm:lb1}
    Under the environment $\Theta^\mathrm{r}=\Theta^\mathrm{c}=\mathbb{R}^d$, let $\Pi^{\mathrm{r}}_t:=\mathcal{N}(\hat{\theta}_t^{\mathrm{r}}, \Sigma_t^{\mathrm{r}}), \Pi^{\mathrm{c}}_t:=\mathcal{N}(\hat{\theta}_t^{\mathrm{c}}, \Sigma_t^{\mathrm{c}})$ be the posterior distribution of the unknown $\thetar$ and $ \thetac$ at time step $t$, for any sampling rule where $V_w:=\sum_{i=1}^K w_{i} x_i x_i^{\top}$ and $w_i:=\lim_{n\rightarrow \infty}\frac{T_{i,n}}{n}$, where $W:=\cbr{w=(w_1, w_2, \ldots, w_K): \sum_i w_i=1, w_i\ge0, \, \forall \, i\in [K]}$, 
    \begin{align*}
        \limsup_{T\rightarrow\infty} -\frac{1}{T} \log \PP_{\Pi_T^{\mathrm{r}}, \Pi_T^{\mathrm{c}}}(z_{\text{out}}\neq z^*)\le\Gamma
    \end{align*}
    where  the hardness term is
    \begin{align*}
     \Gamma  := \max_{w\in W}\min_{z\in \mathcal{Z}} \min_{i=1,2,3,4} \cbr{f_i(w,z)   }, 
     \end{align*}
     and 
     \begin{align*}
     f_1(w,z) &:= \frac{(\Deltac(z))^2}{2\gamma^2\lVert z\rVert^2_{V_w^{-1}}} \cdot \one \cbr{z\in \mathcal{A}_1}, \\
    f_2(w,z) &:= \frac{(\Deltar(z))^2}{2\sigma^2\lVert z-z^* \rVert^2_{V_w^{-1}}} \cdot \one \cbr{z\in \mathcal{A}_2}, \\
    f_3(w,z) &:= \bigg( \frac{(\Deltac(z))^2}{2\gamma^2\lVert z\rVert^2_{V_w^{-1}}} + \frac{(\Deltar(z))^2}{2\sigma^2\lVert z-z^*\rVert^2_{V_w^{-1}}} \bigg) 
    \\&\cdot \one \cbr{z\in \mathcal{A}_3}, \\
    f_4(w,z) &:= \frac{(\Deltac(z))^2}{2\gamma^2\lVert z\rVert^2_{V_w^{-1}}} \cdot \one \cbr{z=z^*}.
\end{align*}
\end{theorem}
The complexity of the linear BAI problem under feasibility constraints is characterized by the hardness parameter $\Gamma$. This parameter is composed of four   terms, denoted by $f_1(w, z)$, $f_2(w, z)$, $f_3(w, z)$, and $f_4(w, z)$, each reflecting a specific source of difficulty in distinguishing arms based on feasibility and optimality conditions.

The term $f_1(w, z)$ captures the challenge of misidentifying superoptimal and infeasible arms as the best feasible arm. This occurs when arms with superior rewards, but violating feasibility constraints, are incorrectly favored due to insufficient exploration of feasibility conditions. The term $f_2(w, z)$ quantifies the difficulty arising from confusing feasible yet suboptimal arms with the optimal feasible arm. In such cases, the agent may exploit arms that satisfy constraints but do not yield the highest reward due to inadequate exploitation of known reward estimates.

Similarly, $f_3(w, z)$ corresponds to the hardness induced by incorrectly identifying infeasible and suboptimal arms as optimal. This issue arises when infeasibility is not adequately recognized during exploration, leading to wasted budget on arms that neither satisfy constraints nor provide optimal rewards. Finally, $f_4(w, z)$ represents the difficulty associated with mistakenly classifying the true best feasible arm as  being infeasible.

\begin{theorem}
\label{Thm:equ_lowerboundterm}
  Under the same assumptions as Theorem~\ref{Thm:lb1}, we establish the equivalence between two distinct expressions of the hardness term in the lower bound:
  \begin{align*}
      \Gamma\!:=\!\max_{w\in \Delta_\mathcal{X}}\!\inf_{(\theta_1, \theta_2)\in\overline{\Theta}_{z^*}}\! \frac{1}{2} \bigg(\! \frac{\lVert\theta_1-\thetar\rVert^2_{A(w)^{}}}{\sigma^2}\!+\!\frac{\lVert\theta_2-\thetac\rVert^2_{A(w)^{}}}{\gamma^2}\!\bigg)
  \end{align*}
\end{theorem}

    Compared to the equivalence of the two hardness terms without constraint, the difference here is that the KL-divergence in the reward and cost  are summarized with the respective rates since the variances of the noise in reward and cost are different.
    If we set $\tau\to \infty$, then our lower bound particularizes to the hardness term as in the sample complexity of pure exploration for linear bandits (\cite{jedra2020optimal}). This form of the exponential rate in the lower bound is especially useful, as it aligns with the structure of our upper bound, which is derived in the same form.

    \vspace{-.1in}
\subsection{Upper Bound}\vspace{-.1in}
\red{We now state the upper bound on the error probability. Theorem~\ref{upperbound} establishes that the error probability of our proposed algorithm BLFAIPS decays exponentially with rate (exponent) at least \( \Gamma \), where \( \Gamma \) is the same hardness constant appearing in the lower bound of Theorem~\ref{Thm:lb1}. This implies that BLFAIPS  is asymptotically optimal: no strategy can achieve  faster exponential decay rate. 
}
\begin{theorem}
    \label{upperbound}
    Under Assumptions~\ref{ass1},~\ref{ass2}, and~\ref{ass3}, with probability 1,
    \begin{align*}
        \liminf_{T \to \infty} -\frac{1}{T} \log \mathbb{P}_{(\thetar_{T+1},\thetac_{T+1})\sim p_{T+1}} (\hat{z}_\text{{out}} \neq z^*) \ge \Gamma, 
\quad 
    \end{align*}
    where $p_{T+1}:=\mathcal{N}(\hat{\theta}^\mathrm{r}_{T+1}, \sigma^2 V_{T}^{-1})\otimes \mathcal{N}(\hat{\theta}^\mathrm{c}_{T+1}, \gamma^2 V_{T}^{-1} )| \Theta_{}~.$
\end{theorem}
    To the best of our knowledge, our algorithm BLFAIPS is the first to demonstrate matching upper and lower bounds in the literature on BAI with constraints under a fixed budget, both in the $K$-armed bandit setting (\cite{faizal2022constrained}) and the linear bandit setting (\cite{tang2024pure}).

\vspace{-.1in}
\section{Proof Sketch of Theorem~\ref{upperbound}}\vspace{-.1in}
We provide a proof sketch for Theorem~\ref{upperbound}, which is the central result of this paper. To prove this theorem, we first introduce the following lemmas concerning good events:
\begin{lemma}
    \label{lem:Good_Event_1}
    Define  the good event  as
    \begin{multline*}
E_{1,\delta}:= \bigcap_{t=1}^T\bigg\{ \lVert \hat{\theta}_t^r - \theta^r  \rVert_{V_{t-1}} \le \sqrt{\beta_1\left(t, \frac{1}{\delta^2}\right)},   \text{and}\\
\lVert \hat{\theta}_t^c - \theta^c  \rVert_{V_{t-1}} \le \sqrt{\beta_2\left(t, \frac{1}{\delta^2}\right)} \bigg\}
\end{multline*}

    where $\beta_1(t,\frac{1}{\delta^2}):= (S_1 + \sigma\sqrt{2 \log(\frac{1}{\delta^2}) + d \log\left(\frac{d + t L^2}{d}\right)})^2$, $\beta_2(t,\frac{1}{\delta^2}):= (S_2 + \gamma\sqrt{2 \log(\frac{1}{\delta^2}) + d \log\left(\frac{d + t L^2}{d}\right)})^2$.
    
    Then with probability $1-2\delta$, good event $E_{1,\delta}$ holds.
\end{lemma}
This lemma, which establishes the concentration property of the ridge regression estimator over the entire time horizon, plays a crucial role in the linear bandit literature.

\begin{lemma}
\label{lem:Good_Event_2}
Define the good event:
\begin{multline*}
    E_{2,\delta} := \bigcap_{t=1}^T\Big\{ \max_{x\in \mathcal{X}} \lvert \langle \hat{\theta}_t^r, x \rangle \rvert \le B_1 
    ,~\max_{x\in \mathcal{X}} \lvert \langle \hat{\theta}_t^c, x \rangle \rvert \le B_2 \Big\}
\end{multline*}
where
\[
B_1 = L R_1 + L \sqrt{\beta_1\left(T, \frac{1}{\delta^2}\right)}, \;\; B_2 = L R_2 + L \sqrt{\beta_2\left(T, \frac{1}{\delta^2}\right)}
\]
Then,
\[
E_{1,\delta} \subseteq E_{2,\delta}
\]
\end{lemma}
This lemma demonstrates that, under the conditions of Lemma~\ref{lem:Good_Event_1} and the bounded $\Theta$ assumption, the empirical reward and cost at each time step are also bounded. This result is essential for applying the Azuma-Hoeffding inequality and deriving an upper bound on regret in our analysis.
\begin{lemma}[Lemma C.14 in \citet{li2024optimal}]
    \label{lem:GoodEvent_3}
    Define the event $E_{3,\delta}:= \cbr{V_t \ge \frac{3}{4}A(\lambda^G), \forall t\ge T_1(\delta), x\in \mathcal{X}}$ where $T_1(\delta):=\max_{x \in \mathcal{X}} \Big( \frac{6 \sqrt{\log(|\mathcal{X}|  \frac{T}{\delta})}}{\lambda_x^G} \Big)^4$.\footnote{We note that $A(\lambda^G)$ depends on $x$.}
    Then with probability $1-\delta$, the event $E_{3,\delta}$ holds.
\end{lemma}
This lemma establishes that, under the mixed G-optimal design allocation, the algorithm gathers sufficient information about the arms by a time point that is on the order of $O(\log T)$. This allows us to determine the time  at which the empirically best feasible arm coincides with the true one.

\begin{lemma}
    \label{lem:GoodEvent_4} 
    Define the good event 
    \[
    E_{4,\delta} := \left\{ \hat{z}_t = z^*, \quad \forall t > \max \left\{ T_1(\delta) + 1, T_2(\delta) \right\} \right\},
    \]
    where
    \begin{align*}
    T_2(\delta) := \max \Bigg\{ 
    &\left( \frac{m_\text{max} |\mathcal{X}|}{\Delta_r^\text{min}} \sqrt{d \beta_1(t, \frac{1}{\delta^2})} \right)^{\frac{8}{3}}, \\
    &\left( \frac{n_\text{max} |\mathcal{X}|}{\Delta_c^\text{min}} \sqrt{d \beta_2(t, \frac{1}{\delta^2})} \right)^{\frac{8}{3}} 
    \Bigg\}.
\end{align*}
    Then
    \[
    E_{1,\delta} \cap E_{3,\delta} \subseteq E_{4,\delta}.
    \]
\end{lemma}
This lemma demonstrates that, with the support of Lemmas~\ref{lem:Good_Event_1} and~\ref{lem:GoodEvent_3}, after a time point in the order of $O(\log T)$, the empirically best feasible arm is indeed the true best feasible arm. This result is crucial to our proof, as both the min-learner strategy (posterior sampling) and the max-learner strategy (AdaHedge) in our algorithm rely on the accuracy of identifying the empirically best feasible arm.

\begin{lemma}
\label{lem:GoodEvent_5}
Define
\[
E_{5,\infty} := \bigg\{\lim_{T \to \infty} \sup_{\theta_1 \in \Theta_1, \theta_2 \in \Theta_2} \frac{1}{T} \left| 
\log \frac{p_{T+1}(\theta_1, \theta_2)}{p_{T+1}(\truethetar, \truethetac)} \right.
\]
\[
\left. + \frac{T}{2} \bigg( \frac{\left\| \truethetar - \theta_1 \right\|^2_{\overline{V}_T}}{\sigma^2} 
+ \frac{\left\| \truethetac - \theta_2 \right\|^2_{\overline{V}_T}}{\gamma^2} \bigg)
\right| = 0\bigg\},
\]
with probability 1, the good event $E_{5,\infty}$ holds. This also implies that
\[
\frac{p_{T+1}(\theta_1, \theta_2)}{p_{T+1}(\truethetar, \truethetac)} \stackrel{\cdot}{=} \exp \left( -\frac{1}{2} \left( 
\frac{\left\| \truethetar\! -\! \theta_1 \right\|^2_{{V}_T}}{\sigma^2} 
\!+ \!\frac{\left\| \truethetac \!-\! \theta_2 \right\|^2_{{V}_T}}{\gamma^2} 
\right) \right).
\]
\end{lemma}
This lemma shows that as $T \to \infty$, the logarithm of the ratio between $p_{T+1}(\theta_1, \theta_2)$ and $p_{T+1}(\truethetar, \truethetac)$ is  asymptotically equal to the hardness term in the exponential lower bound, namely,  $\Gamma$. This result is crucial for applying the Laplace approximation to analyze the error probability in our proof.

\begin{lemma}
\label{lem:GoodEvent_6}
Define the good event 
\begin{align*}
&\! E_{6,\delta}\!:=\!\bigg\{
    \bigg| \max_{w \in \Delta_{\mathcal{X}}} \!\inf_{(\theta_1, \theta_2) \in \overline{\Theta}_{z^*}}\!\bigg(\!\frac{\lVert \theta_1 \!-\! \truethetar \rVert^2_{A(w)}}{2\sigma^2}\! +\! \frac{\lVert \theta_2\! -\! \truethetac \rVert^2_{A(w)}}{2\gamma^2}\!\bigg)
      \\*
&\;\;\; -\!\!\inf_{(\theta_1, \theta_2) \in \overline{\Theta}_{z^*}} \bigg(\frac{\lVert \theta_1\! - \!\truethetar \rVert^2_{\overline{V}_T}}{2\sigma^2} \!+ \!\frac{\lVert \theta_2 \!- \!\truethetac \rVert^2_{\overline{V}_T}}{2\gamma^2} \bigg)
    \bigg| \!\le\! o(1)\!
\bigg\}
\end{align*}
Then event \( E_{6,\frac{1}{T}} \) holds with probability at least \( 1 - \frac{28}{T} \), conditioned on events \( E_{1,\frac{1}{T}}, E_{2,\frac{1}{T}}, E_{3,\frac{1}{T}}, E_{4,\frac{1}{T}} \).
\end{lemma}
This lemma shows that when events $  E_{1,\frac{1}{T}}$ $ E_{2,\frac{1}{T}}$ $ E_{3,\frac{1}{T}}$, and $ E_{4,\frac{1}{T}} $ hold, the exponent in the lower bound $\Gamma$ is asymptotically equal to the Laplace approximation of the integral of the log-ratio, as stated in Lemma~\ref{lem:GoodEvent_5}.

Then the error probability of our algorithm is
\begin{align*}  &\mathbb{P}_{(\thetar_{T+1},\thetac_{T+1})\sim p_{T+1}} (\hat{z}_\text{{out}} \neq z^*)
\\&=\frac{\int_{(\theta_1,\theta_2)\in \overline{\Theta}_{z^*}}p_{T+1}(\theta_1,\theta_2)\,\mathrm{d}\theta_1\,\mathrm{d}\theta_2}{\int_{(\theta_1,\theta_2)\in {\Theta}_{z^*}}p_{T+1}(\theta_1,\theta_2)\,\mathrm{d}\theta_1\,\mathrm{d}\theta_2}
\\&=\frac{\int_{(\theta_1,\theta_2)\in \overline{\Theta}_{z^*}}p_{T+1}(\theta_1,\theta_2)/p_{T+1}(\thetar,\thetac)\,\mathrm{d}\theta_1\mathrm{d}\theta_2}{\int_{(\theta_1,\theta_2)\in {\Theta}_{z^*}}p_{T+1}(\theta_1,\theta_2)/p_{T+1}(\thetar,\thetac)\,\mathrm{d}\theta_1\, \mathrm{d}\theta_2}
\\&\stackrel{(a)}{\stackrel{\cdot}{=}}\frac{\int_{(\theta_1,\theta_2)\in \overline{\Theta}_{z^*}} \exp \left( M\right)\,\mathrm{d}\theta_1\mathrm{d}\theta_2}{\int_{(\theta_1,\theta_2)\in {\Theta}_{z^*}}\exp \left( M\right)\,\mathrm{d}\theta_1\mathrm{d}\theta_2},
\end{align*}
where $M = M(\theta_1,\theta_2):=-\frac{1}{2} \Big( 
\frac{\left\| \truethetar - \theta_1 \right\|^2_{{V}_T}}{\sigma^2} 
+ \frac{\left\| \truethetac - \theta_2 \right\|^2_{{V}_T}}{\gamma^2} 
\Big) $ and $(a)$ comes from good event $E_{5,\infty}$.
Furthermore, from the Laplace approximation (Lemma~\ref{lem:GaussianApproximation}) and the fact that 
\begin{align*}
    \inf_{(\theta_1,\theta_2)\in {\Theta}}  
\frac{\left\| \truethetar - \theta_1 \right\|^2_{{V}_T}}{\sigma^2} 
+ \frac{\left\| \truethetac - \theta_2 \right\|^2_{{V}_T}}{\gamma^2} =0~,
\end{align*}
we have
\begin{align*}  &\mathbb{P}_{(\thetar_{T+1},\thetac_{T+1})\sim p_{T+1}} (\hat{z}_\text{{out}} \neq z^*)
\\&\stackrel{\cdot}{=}\exp \left( -\frac{T}{2} \inf_{(\theta_1,\theta_2)\in \overline{\Theta}_{z^*}}\bigg( 
\frac{\left\| \truethetar \!-\! \theta_1 \right\|^2_{\overline{V}_T}}{\sigma^2} 
\!+\! \frac{\left\| \truethetac \!-\! \theta_2 \right\|^2_{\overline{V}_T}}{\gamma^2} 
\bigg) \right).
\end{align*}
Then combining with the good event $E_{6,\delta}$, we have
\begin{align*}  \mathbb{P}_{(\thetar_{T+1},\thetac_{T+1})\sim p_{T+1}} (\hat{z}_\text{{out}} \neq z^*)\stackrel{\cdot}{=}
\exp \left( -T \Gamma \right).
\end{align*}
To summarize, with the choice $\delta=\frac{1}{T}$, with probability 1, 
\begin{align*}
        \lim_{T \to \infty} -\frac{1}{T} \log \mathbb{P}_{(\thetar_{T+1},\thetac_{T+1})\sim p_{T+1}} (\hat{z}_\text{{out}} \neq z^*) = \Gamma~.
    \end{align*}

\vspace{-.1in}
\section{Empirical Studies}\vspace{-.1in}
For the empirical studies, we compare our proposed algorithm BLFAIPS against three baselines: a modified Linear Thompson Sampling algorithm (\cite{agrawal2013thompson}), which selects the empirically best feasible arm at each round; the Linear $\beta$-Top-Two Thompson Sampling algorithm, where $\beta$ is set to match the allocation rate of the best feasible arm; and an Oracle baseline, which pulls arms according to the optimal allocation rate derived from the lower bound established in Theorem~\ref{Thm:lb1}; the Linear Top-Two Thompson Sampling algorithm with the optimal $\beta$ under constraints, which can be viewed as an extension from $K$-armed bandits to linear bandits of the algorithm in \citet{yang2022minimax}. Due to space constraints, we show one representative plot here in the main paper. Other plots for other parameters are shown in Appendix~\ref{sec:empirical_plots}.

\vspace{-.1in}
\subsection{"End of Optimism" Instance}\vspace{-.1in}
We perform our first  experiment on the ubiquitous ``End of Optimism'' instance  (\cite{lattimore2017end}) plus an additional feasibility constraint. This can be viewed as the Soare's Instance as considered in \cite{li2024optimal} with linear constraints. To be more specific, in this instance, we choose $\thetar=[1, 0]^\top$, $\thetac=[0, 1]^\top$, $\tau=0.5$, there are  five arms: $[1, 0]^\top, [0, 0.15]^\top$, $[0, 1]^\top$, $[1.2, 1.2]^\top$, $ [\cos(\alpha), \sin(\alpha)]^\top$, where $[1, 0]^\top$ is the best feasible arm, $[0, 0.15]^\top$ is the suboptimal feasible arm, $[0, 1]^\top$ is the suboptimal and infeasible arm, $[1.2, 1.2]^\top$ is the superoptimal and infeasible arm. Figures~\ref{fig:1},~\ref{fig:2}, and~\ref{fig:3} (in Appendix~\ref{sec:eoo}) show the  accuracies  of identifying the best feasible arm over time for these algorithms when $\alpha=0.1, 0.2, 0.3$. The plots also present confidence intervals, represented by a range of plus or minus two standard deviations  from the average. We choose the time horizon $T=2,000$ and run each algorithm over 50 repetitions.

\begin{figure}[h!]
    \centering
    \includegraphics[width=1\linewidth]{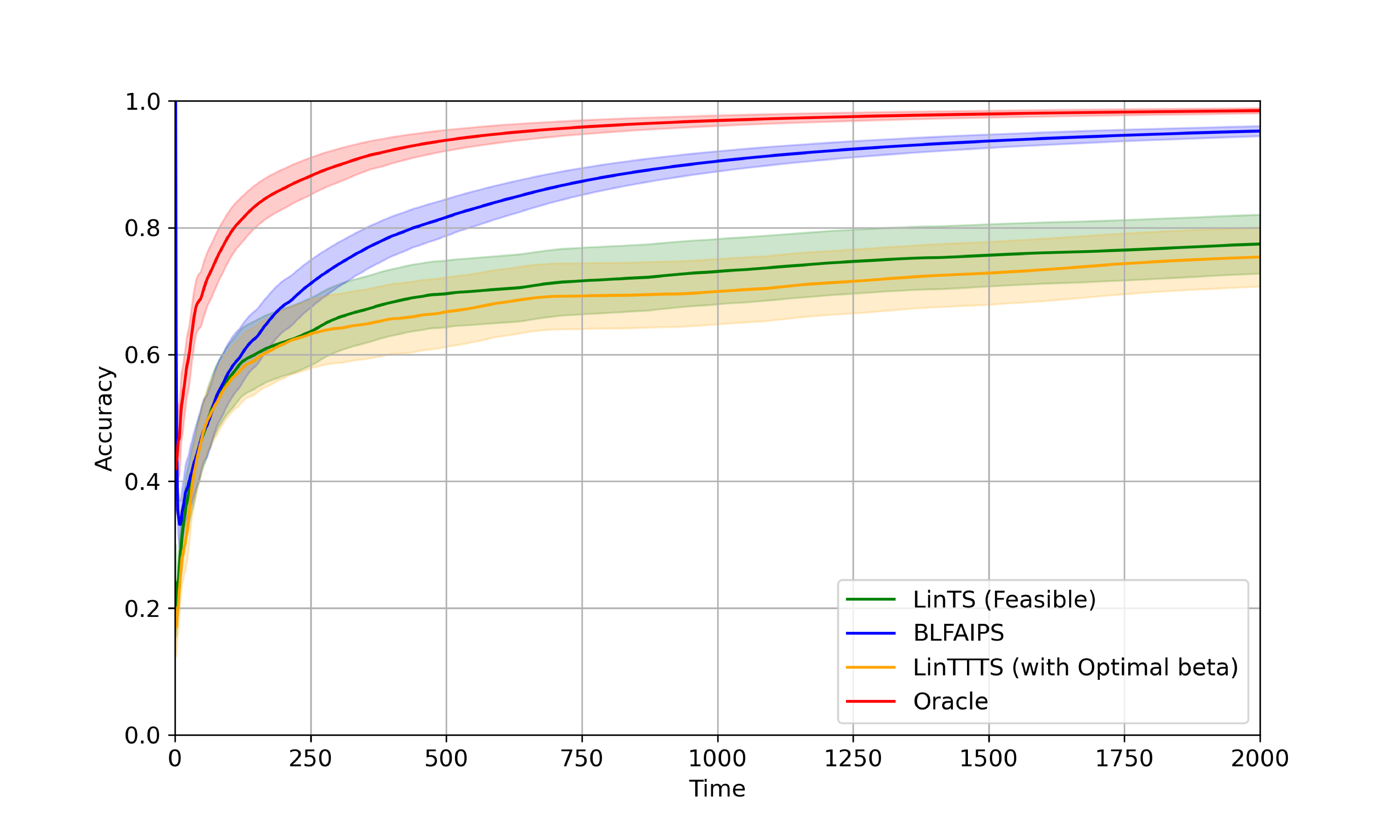}
    \caption{$\alpha=0.1$}
    \label{fig:1}
\end{figure}

The experimental results, as illustrated by the plots, demonstrate that our proposed algorithm, BLFAIPS, consistently outperforms alternative approaches, including the constrained versions of Linear Thompson Sampling   and  Linear $\beta$-Top-two Thompson Sampling. This performance advantage persists even when the $\beta$ parameter is fine-tuned to match the true allocation rate of the given instance.

\vspace{-.1in}
\subsection{Random Instances}\vspace{-.1in}
We next evaluate the performance of our algorithm on random instances within a $d$-dimensional unit ball. For these experiments, we randomly sampled $K = 5, 20, 50$ arms from $d = 2, 20, 50$-dimensional spaces. The parameters were set as follows: $\thetar = [1, 0, \dots, 0]^\top \in \mathbb{R}^d$, $\thetac = [0, \dots, 0, 1]^\top \in \mathbb{R}^d$, and $\tau = 0.5$. Figures~\ref{fig:d2K5},~\ref{fig:d2K50}, and~\ref{fig:d2K20} (in Appendix~\ref{sec:random})  present the accuracies over time for varying values of $K$, while Figures~\ref{fig:d50K20},~\ref{fig:d2K20},~\ref{fig:d20K20} illustrate the accuracies over time for different dimensionalities $d$.

\begin{figure}[h!]
    \centering
    \includegraphics[width=1\linewidth]{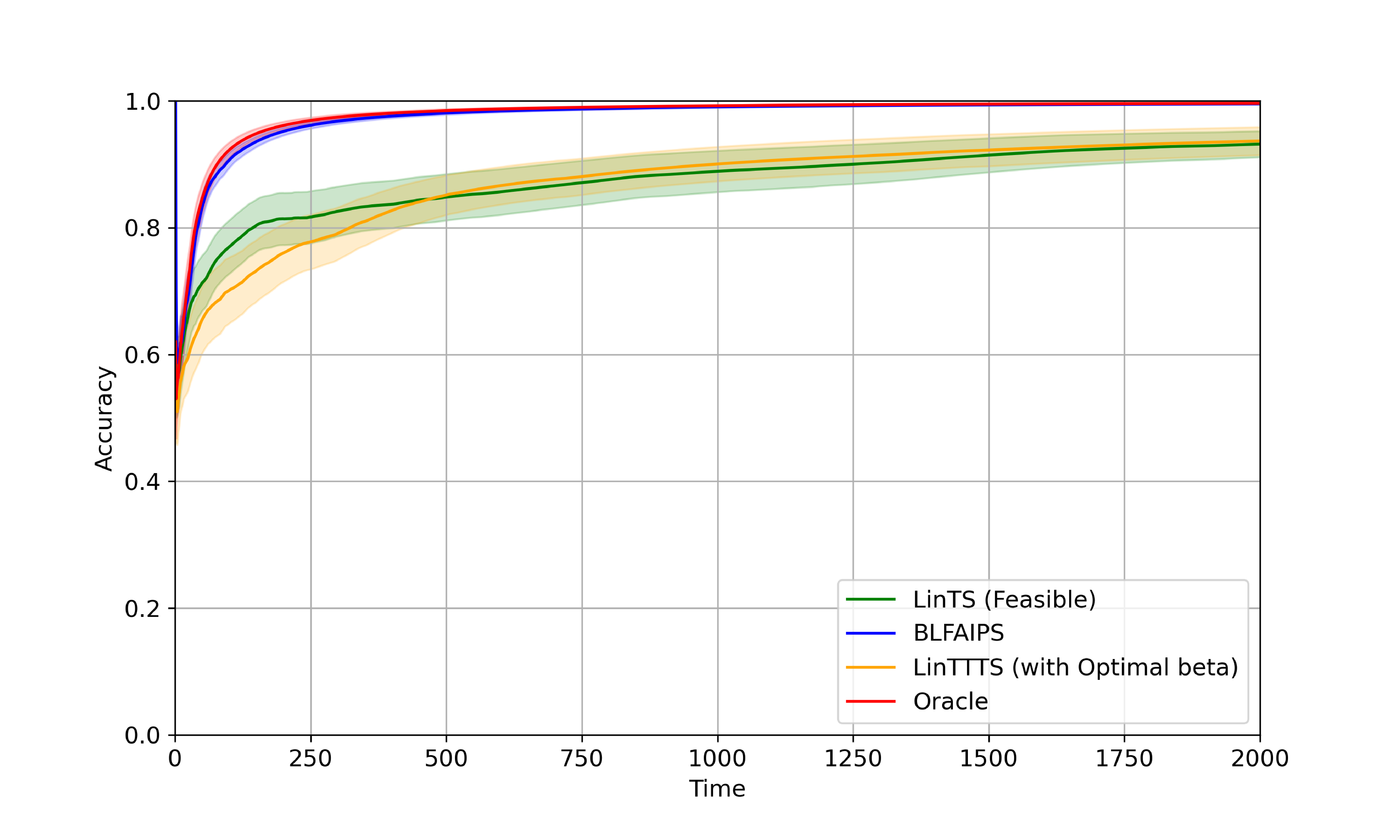}
    \caption{$d=2, K=5$}
    \label{fig:d2K5}

    \includegraphics[width=1\linewidth]{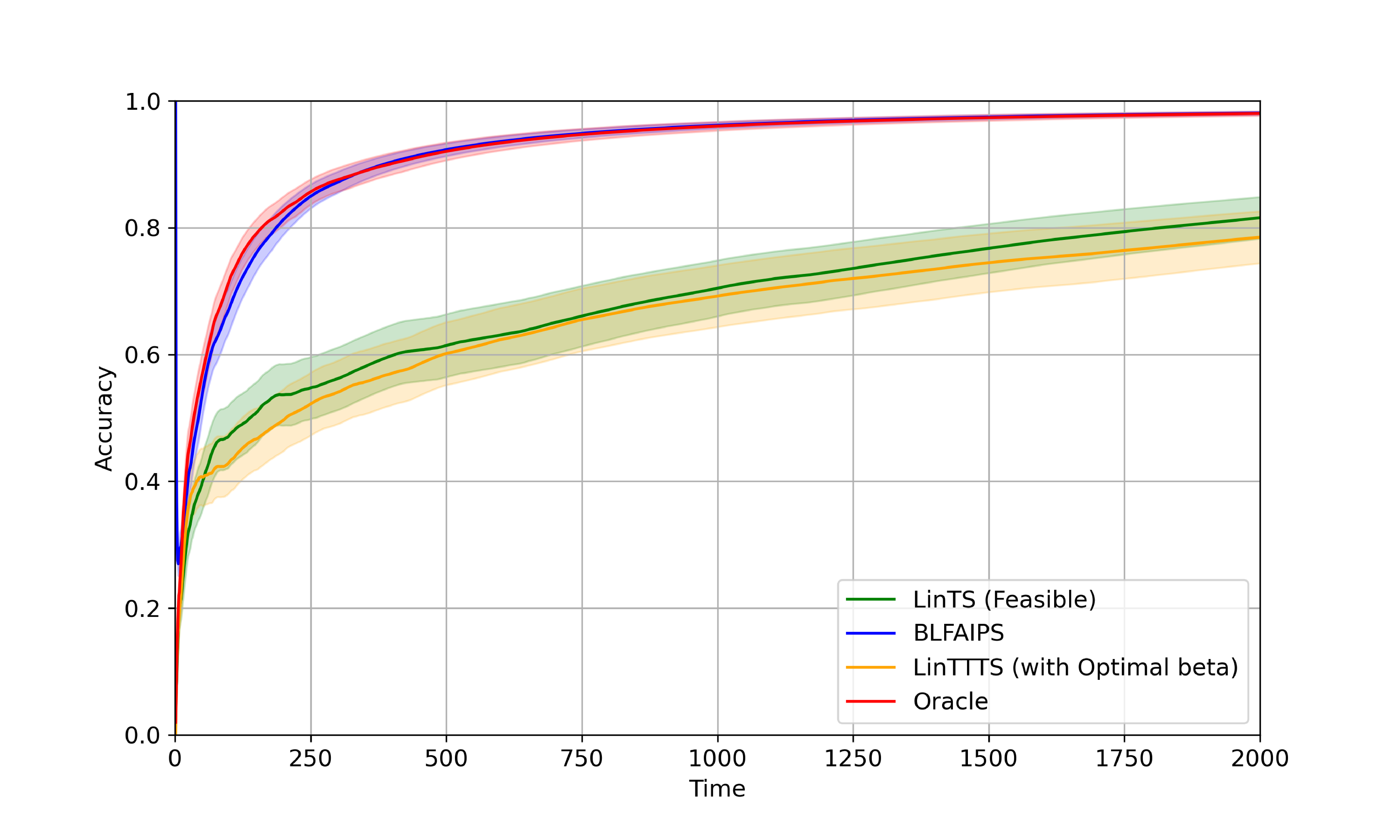}
    \caption{$d=2, K=50$}
    \label{fig:d2K50}

    \includegraphics[width=1\linewidth]{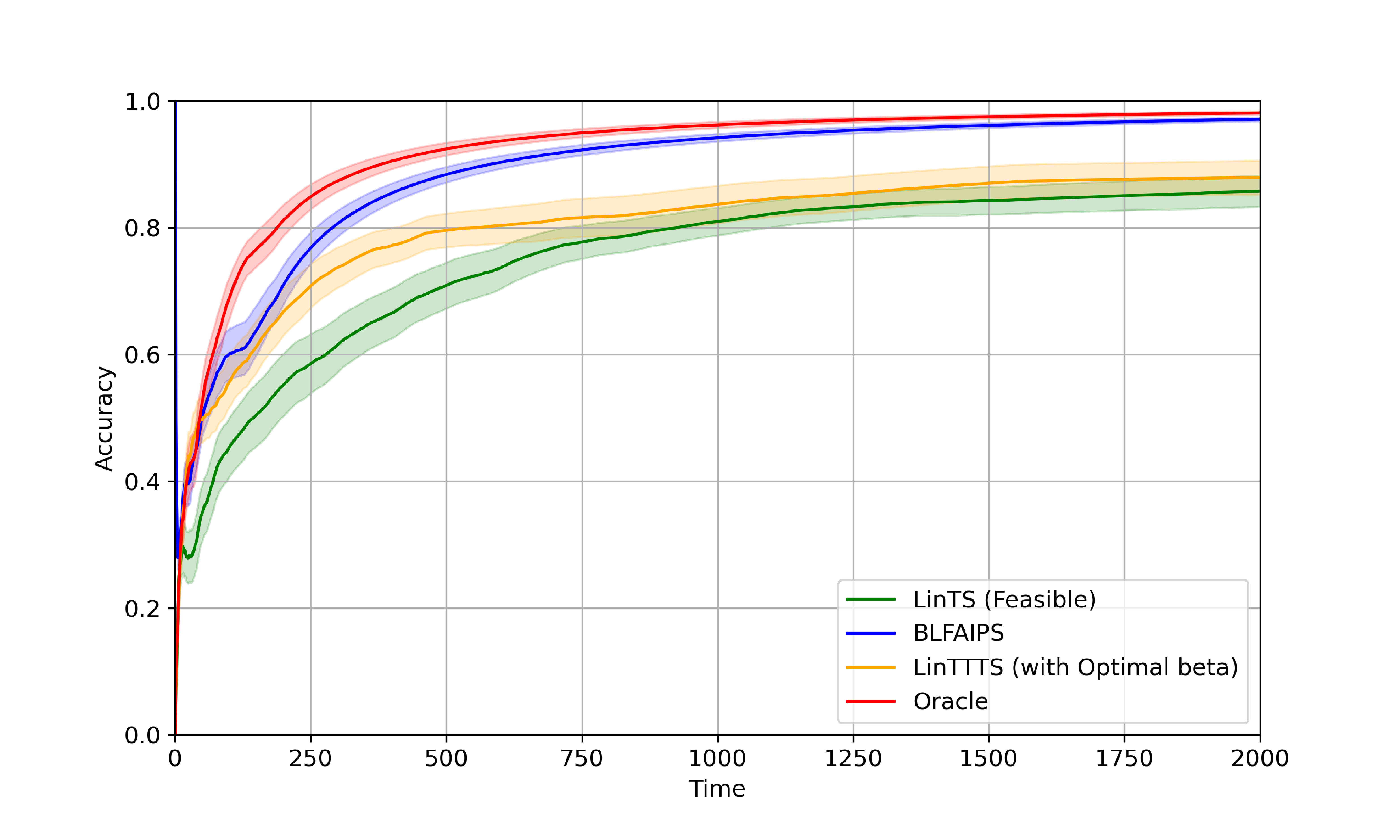}
    \caption{$d=50, K=20$}
    \label{fig:d50K20}
    \vspace{-.1in}
\end{figure}
The results indicate that our proposed algorithm, BLFAIPS, consistently exhibits faster convergence compared to competing algorithms across varying numbers of arms $K$ and dimensionalities $d$ in random instances.

\vspace{-.1in}
\subsection{Comparison with PEPS}\vspace{-.1in}

We also compare our algorithm with the PEPS algorithm introduced by \citet{li2024optimal}, which uses AdaHedge but does not incorporate the doubling trick. While their algorithm PEPS performs well empirically, they did not provide a formal upper bound guaranteeing its asymptotic optimality. In contrast, our Theorem~\ref{upperbound} establishes such a guarantee, matching the information-theoretic lower bound (Theorem~\ref{Thm:lb1}) and closing a key theoretical gap in the literature.

We evaluate both algorithms (and others) in the unconstrained setting using the same configuration as in the random instance experiments with $d=2$ and $K=5$, except that the constraint threshold $\tau$ is set to $\infty$. When the sampling budget \( T \) is known in advance, PEPS and BLFAIPS achieve comparable performance. However, when \( T \) is {\em unknown}, as is common in many real-world applications, PEPS tends to waste almost half its budget during its initial exploration phase, leading to strictly inferior results. In such cases, BLFAIPS achieves superior performance as shown in Figure~\ref{fig:comparison_with_PEPS}.

\begin{figure}[h!]
    \centering
    \includegraphics[width=1\linewidth]{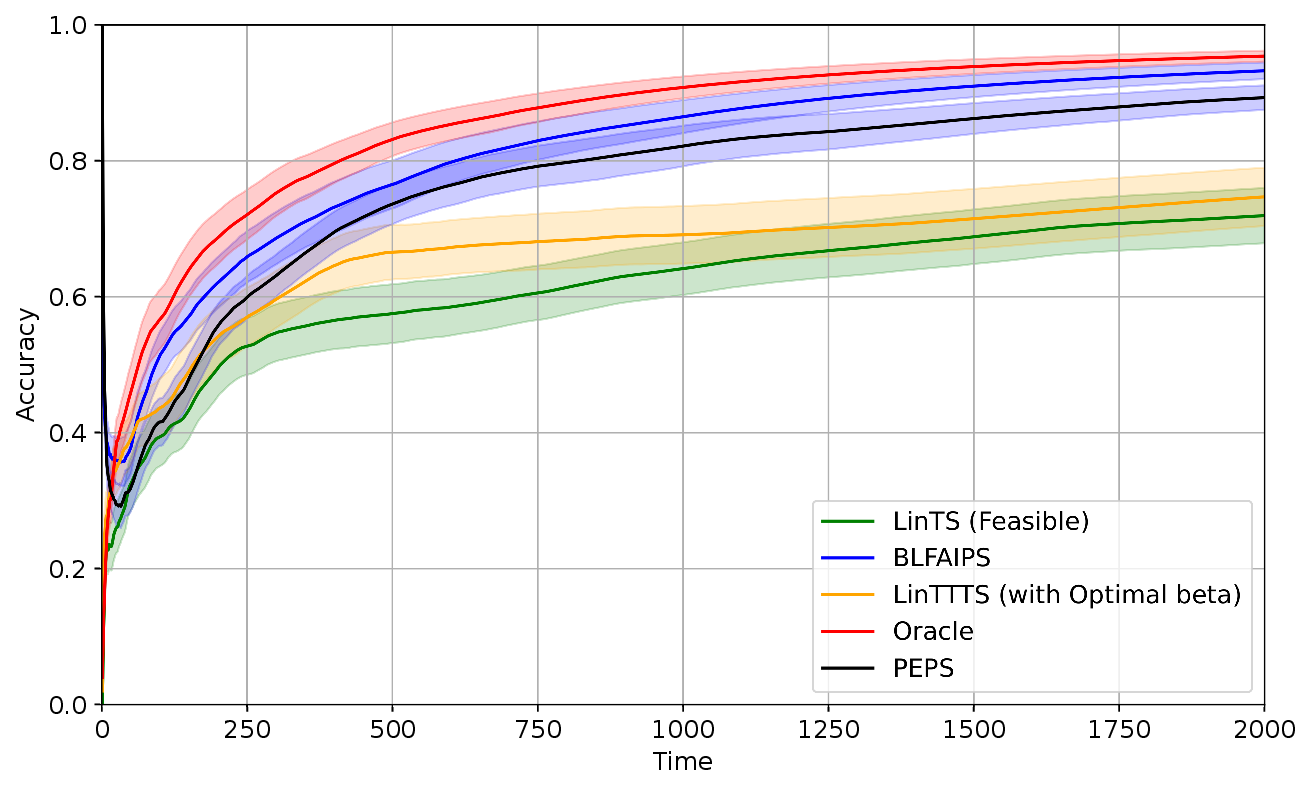}
    \caption{Comparison of BLFAIPS to PEPS for unknown~$T$}
    \label{fig:comparison_with_PEPS}
\end{figure}

\vspace{-.1in}
\subsection{Evaluation on a Real-World Dataset}\vspace{-.1in}
We also conducted an experiment utilizing a real-world dataset: the MovieLens 10M dataset augmented with IMDb ratings, based on an open-source preprocessing script~\cite{striatum2017}. Our objective was to identify the highest-rated movie (on a 1-to-5 scale) among the top 20 movies in MovieLens, subject to the constraint that the IMDb rating as of April 2025 remains below \(\tau=7.5\) (on a 1-to-10 scale).

To assess the feasibility-aware performance of our algorithm, we calculated the cumulative accuracies over 5,000 time steps for various algorithms. Accuracy is defined as the proportion of instances where the selected movie satisfies the IMDb constraint and is simultaneously the highest-rated based on user ratings.

Since this is a real-world dataset, the true problem instance is unknown. Thus,  it not possible to determine an optimal solution or a lower bound. We compared our approach against baseline configurations using the same parameters, which were set as follows: \(\sigma=1\), \(\gamma=1\), \(L=\sqrt{3}\), \(R_1=\sqrt{10}\), and \(R_2=\sqrt{20}\). Each algorithm was executed 50 times over the 5,000 time steps.

As illustrated in Figure~\ref{fig:realworld}, our proposed method  BLFAIPS significantly outperforms both the constrained versions of {Linear Feasible Thompson Sampling} and {Linear TTTS} with \(\beta=0.5\). It consistently achieves higher accuracy while maintaining similar error bars, showing the accuracy, stability, and robustness of BLFAIPS. 

\begin{figure}[h!]
    \centering
    \includegraphics[width=1\linewidth]{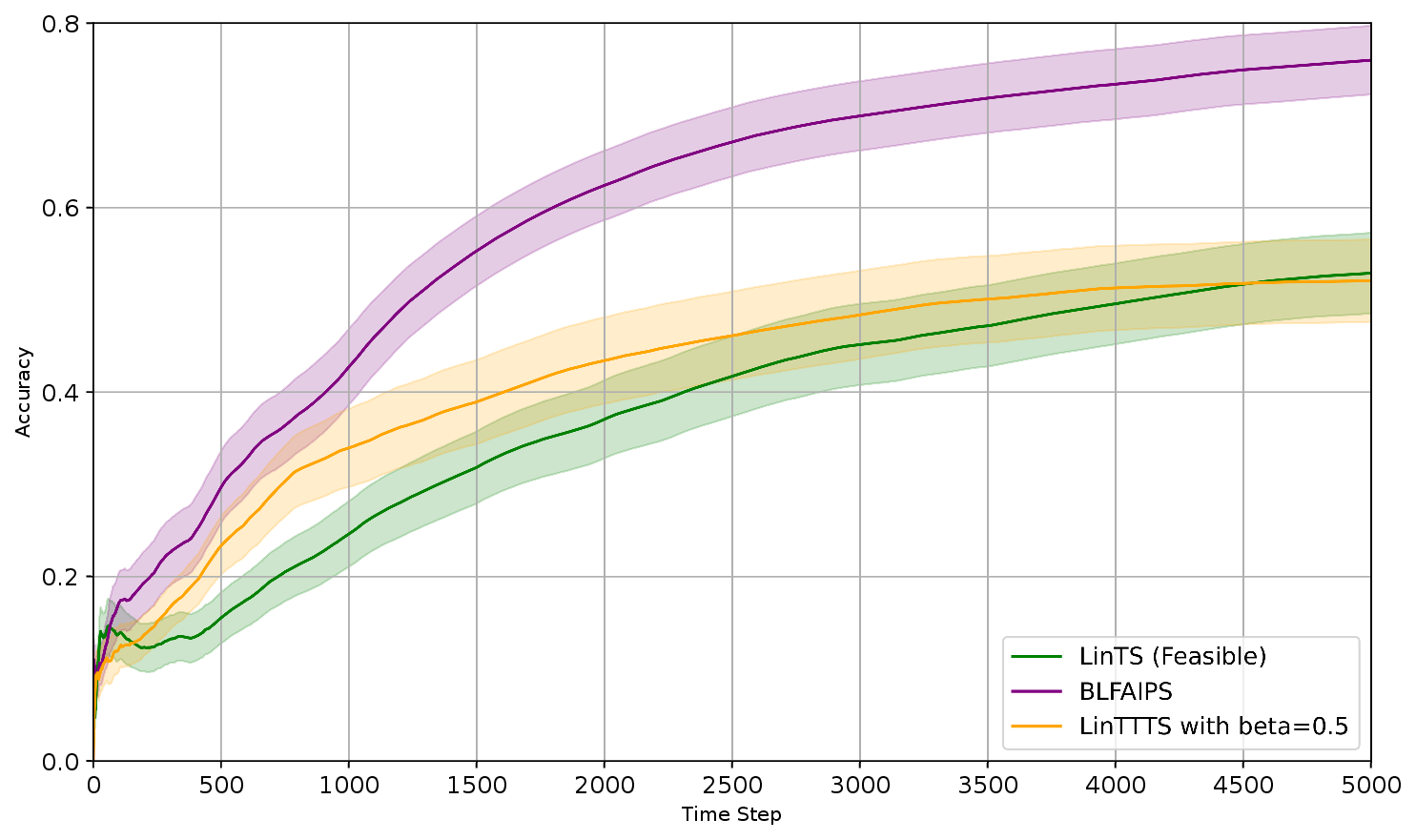}
    \caption{Comparison  of accuracies of various algorithms  on the MovieLens dataset}
    \label{fig:realworld}
\end{figure}

\vspace{-.1in}
\section{Conclusion and Future Work}\vspace{-.1in}

In this paper, we addressed the best feasible arm identification problem under a fixed budget in linear bandits. We proposed a novel algorithm that matches upper and lower bounds on the exponential rate of the error probability. To the best of our knowledge, it is the first to attain such optimality in this domain. Our contributions include two equivalent formulations of the lower bound rate from Bayesian and frequentist perspectives, structural improvements to existing methods, and empirical results showing the superiority of our approach over several baselines.

Our theoretical findings provide a basis for further exploration in constrained bandit problems. While we have established matching bounds under specific conditions, extending these results to more general settings, such as non-convex or high-dimensional parameter spaces, remains valuable. Refining guarantees under weaker assumptions, especially regarding noise and cost constraints, could further broaden applicability.

We hope our work will inspire advances in constrained best arm identification and related problems within the multi-armed bandit framework.

\paragraph{Acknowledgements} 
This research is supported by the National Research Foundation, Singapore under its
AI Singapore Programme (AISG Award No: AISG2-PhD-2023-08-044T-J), and is part of the programme
DesCartes which is supported by the National Research Foundation, Prime Minister’s Office, Singapore
under its Campus for Research Excellence and Technological Enterprise (CREATE) programme.


\bibliography{uai2025-template}

\begin{thebibliography}{36}
\providecommand{\natexlab}[1]{#1}
\providecommand{\url}[1]{\texttt{#1}}
\expandafter\ifx\csname urlstyle\endcsname\relax
  \providecommand{\doi}[1]{doi: #1}\else
  \providecommand{\doi}{doi: \begingroup \urlstyle{rm}\Url}\fi

\bibitem[Abbasi-Yadkori et~al.(2011)Abbasi-Yadkori, P{\'a}l, and
  Szepesv{\'a}ri]{abbasi2011improved}
Yasin Abbasi-Yadkori, D{\'a}vid P{\'a}l, and Csaba Szepesv{\'a}ri.
\newblock Improved algorithms for linear stochastic bandits.
\newblock \emph{Advances in Neural Information Processing Systems}, 24, 2011.

\bibitem[Agrawal and Goyal(2013)]{agrawal2013thompson}
Shipra Agrawal and Navin Goyal.
\newblock Thompson sampling for contextual bandits with linear payoffs.
\newblock In \emph{International Conference on Machine Learning}, pages
  127--135. PMLR, 2013.

\bibitem[Amani et~al.(2019)Amani, Alizadeh, and Thrampoulidis]{amani2019linear}
Sanae Amani, Mahnoosh Alizadeh, and Christos Thrampoulidis.
\newblock Linear stochastic bandits under safety constraints.
\newblock \emph{Advances in Neural Information Processing Systems}, 32, 2019.

\bibitem[Audibert and Bubeck(2010)]{audibert2010best}
Jean-Yves Audibert and S{\'e}bastien Bubeck.
\newblock Best arm identification in multi-armed bandits.
\newblock In \emph{Proc. of 23rd Conference on Learning Theory}, page~13, 2010.

\bibitem[Auer et~al.(2002)Auer, Cesa-Bianchi, and Fischer]{auer2002finite}
Peter Auer, Nicolo Cesa-Bianchi, and Paul Fischer.
\newblock Finite-time analysis of the multiarmed bandit problem.
\newblock \emph{Machine Learning}, 47:\penalty0 235--256, 2002.

\bibitem[Azizi et~al.(2021)Azizi, Kveton, and Ghavamzadeh]{azizi2021fixed}
Mohammad~Javad Azizi, Branislav Kveton, and Mohammad Ghavamzadeh.
\newblock Fixed-budget best-arm identification in structured bandits.
\newblock \emph{arXiv preprint arXiv:2106.04763}, 2021.

\bibitem[Bouneffouf and Rish(2019)]{bouneffouf2019survey}
Djallel Bouneffouf and Irina Rish.
\newblock A survey on practical applications of multi-armed and contextual
  bandits.
\newblock \emph{arXiv preprint arXiv:1904.10040}, 2019.

\bibitem[Camilleri et~al.(2022)]{camilleri2022active}
Romain Camilleri et~al.
\newblock Active learning with safety constraints.
\newblock \emph{Advances in Neural Information Processing Systems},
  35:\penalty0 33201--33214, 2022.

\bibitem[De~Rooij et~al.(2014)De~Rooij, Van~Erven, Gr{\"u}nwald, and
  Koolen]{de2014follow}
Steven De~Rooij, Tim Van~Erven, Peter~D Gr{\"u}nwald, and Wouter~M Koolen.
\newblock Follow the leader if you can, hedge if you must.
\newblock \emph{The Journal of Machine Learning Research}, 15\penalty0
  (1):\penalty0 1281--1316, 2014.

\bibitem[Faizal and Nair(2022)]{faizal2022constrained}
Fathima~Zarin Faizal and Jayakrishnan Nair.
\newblock Constrained pure exploration multi-armed bandits with a fixed budget.
\newblock \emph{arXiv preprint arXiv:2211.14768}, 2022.

\bibitem[Gabillon et~al.(2012)Gabillon, Ghavamzadeh, and
  Lazaric]{gabillon2012best}
Victor Gabillon, Mohammad Ghavamzadeh, and Alessandro Lazaric.
\newblock Best arm identification: A unified approach to fixed budget and fixed
  confidence.
\newblock \emph{Advances in Neural Information Processing Systems}, 25, 2012.

\bibitem[Hou et~al.(2023)Hou, Tan, and Zhong]{Hou23}
Yunlong Hou, Vincent Y.~F. Tan, and Zixin Zhong.
\newblock Almost optimal variance-constrained best arm identification.
\newblock \emph{IEEE Transactions on Information Theory}, 69\penalty0
  (4):\penalty0 2603--2634, 2023.

\bibitem[Jedra and Proutiere(2020)]{jedra2020optimal}
Yassir Jedra and Alexandre Proutiere.
\newblock Optimal best-arm identification in linear bandits.
\newblock \emph{Advances in Neural Information Processing Systems},
  33:\penalty0 10007--10017, 2020.

\bibitem[Kalyanakrishnan et~al.(2012)Kalyanakrishnan, Tewari, Auer, and
  Stone]{kalyanakrishnan2012pac}
Shivaram Kalyanakrishnan, Ambuj Tewari, Peter Auer, and Peter Stone.
\newblock {PAC} subset selection in stochastic multi-armed bandits.
\newblock In \emph{International Conference on Machine Learning}, volume~12,
  pages 655--662, 2012.

\bibitem[Karnin et~al.(2013)Karnin, Koren, and Somekh]{karnin2013almost}
Zohar Karnin, Tomer Koren, and Oren Somekh.
\newblock Almost optimal exploration in multi-armed bandits.
\newblock In \emph{International Conference on Machine Learning}, pages
  1238--1246. PMLR, 2013.

\bibitem[Katz-Samuels and Scott(2018)]{katz2018feasible}
Julian Katz-Samuels and Clay Scott.
\newblock Feasible arm identification.
\newblock In \emph{International Conference on Machine Learning}, pages
  2535--2543. PMLR, 2018.

\bibitem[Katz-Samuels and Scott(2019)]{katz2019top}
Julian Katz-Samuels and Clayton Scott.
\newblock Top feasible arm identification.
\newblock In \emph{The 22nd International Conference on Artificial Intelligence
  and Statistics}, pages 1593--1601. PMLR, 2019.

\bibitem[Komiyama et~al.(2022)Komiyama, Tsuchiya, and
  Honda]{komiyama2022minimax}
Junpei Komiyama, Taira Tsuchiya, and Junya Honda.
\newblock Minimax optimal algorithms for fixed-budget best arm identification.
\newblock \emph{Advances in Neural Information Processing Systems},
  35:\penalty0 10393--10404, 2022.

\bibitem[Kone et~al.(2024)Kone, Jourdan, and Kaufmann]{kone2024pareto}
Cyrille Kone, Marc Jourdan, and Emilie Kaufmann.
\newblock Pareto set identification with posterior sampling.
\newblock \emph{arXiv preprint arXiv:2411.04939}, 2024.

\bibitem[Kuleshov and Precup(2014)]{kuleshov2014algorithms}
Volodymyr Kuleshov and Doina Precup.
\newblock Algorithms for multi-armed bandit problems.
\newblock \emph{arXiv preprint arXiv:1402.6028}, 2014.

\bibitem[Kuroki et~al.(2020)Kuroki, Xu, Miyauchi, Honda, and
  Sugiyama]{kuroki2020polynomial}
Yuko Kuroki, Liyuan Xu, Atsushi Miyauchi, Junya Honda, and Masashi Sugiyama.
\newblock Polynomial-time algorithms for multiple-arm identification with
  full-bandit feedback.
\newblock \emph{Neural Computation}, 32\penalty0 (9):\penalty0 1733--1773,
  2020.

\bibitem[Lab(2017)]{striatum2017}
NTU~Computational Lab.
\newblock Striatum: A contextual bandit library.
\newblock \url{https://github.com/ntucllab/striatum}, 2017.
\newblock Accessed: 2025-06-02.

\bibitem[Lattimore and Szepesvari(2017)]{lattimore2017end}
Tor Lattimore and Csaba Szepesvari.
\newblock The end of optimism? an asymptotic analysis of finite-armed linear
  bandits.
\newblock In \emph{Artificial Intelligence and Statistics}, pages 728--737.
  PMLR, 2017.

\bibitem[Lattimore and Szepesv{\'a}ri(2020)]{lattimore2020bandit}
Tor Lattimore and Csaba Szepesv{\'a}ri.
\newblock \emph{Bandit Algorithms}.
\newblock Cambridge University Press, 2020.

\bibitem[Li et~al.(2024)Li, Jamieson, and Jain]{li2024optimal}
Zhaoqi Li, Kevin Jamieson, and Lalit Jain.
\newblock Optimal exploration is no harder than {Thompson} sampling.
\newblock In \emph{International Conference on Artificial Intelligence and
  Statistics}, pages 1684--1692. PMLR, 2024.

\bibitem[Maron and Moore(1997)]{maron1997racing}
Oden Maron and Andrew~W Moore.
\newblock The racing algorithm: Model selection for lazy learners.
\newblock \emph{Artificial Intelligence Review}, 11:\penalty0 193--225, 1997.

\bibitem[Moradipari et~al.(2021)Moradipari, Amani, Alizadeh, and
  Thrampoulidis]{moradipari2021safe}
Ahmadreza Moradipari, Sanae Amani, Mahnoosh Alizadeh, and Christos
  Thrampoulidis.
\newblock Safe linear {Thompson} sampling with side information.
\newblock \emph{IEEE Transactions on Signal Processing}, 69:\penalty0
  3755--3767, 2021.

\bibitem[Pacchiano et~al.(2024)Pacchiano, Ghavamzadeh, and
  Bartlett]{pacchiano2024contextual}
Aldo Pacchiano, Mohammad Ghavamzadeh, and Peter Bartlett.
\newblock Contextual bandits with stage-wise constraints.
\newblock \emph{arXiv preprint arXiv:2401.08016}, 2024.

\bibitem[Russo(2016)]{russo2016simple}
Daniel Russo.
\newblock Simple {Bayesian} algorithms for best arm identification.
\newblock In \emph{Conference on Learning Theory}, pages 1417--1418. PMLR,
  2016.

\bibitem[Shang et~al.(2023)Shang, Colin, Barlier, and
  Cherkaoui]{shang2023price}
Xuedong Shang, Igor Colin, Merwan Barlier, and Hamza Cherkaoui.
\newblock Price of safety in linear best arm identification.
\newblock \emph{arXiv preprint arXiv:2309.08709}, 2023.

\bibitem[Tang et~al.(2024)Tang, Jain, Nayyar, and Nuzzo]{tang2024pure}
Dengwang Tang, Rahul Jain, Ashutosh Nayyar, and Pierluigi Nuzzo.
\newblock Pure exploration for constrained best mixed arm identification with a
  fixed budget.
\newblock \emph{arXiv preprint arXiv:2405.15090}, 2024.

\bibitem[Thompson(1933)]{dc35850b-2ca1-314f-9e0d-470713436b17}
William~R. Thompson.
\newblock On the likelihood that one unknown probability exceeds another in
  view of the evidence of two samples.
\newblock \emph{Biometrika}, 25\penalty0 (3/4):\penalty0 285--294, 1933.
\newblock ISSN 00063444.
\newblock URL \url{http://www.jstor.org/stable/2332286}.

\bibitem[Wang et~al.(2022)Wang, Wagenmaker, and Jamieson]{wang2022best}
Zhenlin Wang, Andrew~J Wagenmaker, and Kevin Jamieson.
\newblock Best arm identification with safety constraints.
\newblock In \emph{International Conference on Artificial Intelligence and
  Statistics}, pages 9114--9146. PMLR, 2022.

\bibitem[Yang and Tan(2022)]{yang2022minimax}
Junwen Yang and Vincent Tan.
\newblock Minimax optimal fixed-budget best arm identification in linear
  bandits.
\newblock \emph{Advances in Neural Information Processing Systems},
  35:\penalty0 12253--12266, 2022.

\bibitem[Yang et~al.(2025)Yang, Gao, Li, and Wang]{yang2025stochastically}
Le~Yang, Siyang Gao, Cheng Li, and Yi~Wang.
\newblock Stochastically constrained best arm identification with {Thompson}
  sampling.
\newblock \emph{Automatica}, 176:\penalty0 112223, Jun 2025.

\bibitem[Zhou and Ji(2022)]{zhou2022kernelized}
Xingyu Zhou and Bo~Ji.
\newblock On kernelized multi-armed bandits with constraints.
\newblock \emph{Advances in Neural Information Processing Systems},
  35:\penalty0 14--26, 2022.

\end{thebibliography}

\newpage

\onecolumn

\title{Asymptotically Optimal Linear Best Feasible Arm Identification with Fixed Budget\\(Supplementary Material)}
\maketitle

\appendix
\section{Table of Notation}
\begin{table}[h!]
\renewcommand{\arraystretch}{1.5} 
\begin{tabular}{p{0.3\textwidth} | p{0.6\textwidth}}
\hline
$B_1= 2L R_1+L \sqrt{\beta_1(T, \frac{1}{\delta^2})}$ & Upper bound of $\max_{x \in \mathcal{X}} \max_{t \leq T} \lvert \langle x, \hat{\theta}^{\mathrm{r}}_t \rangle \rvert
$ \\
$B_2=2L R_2+L \sqrt{\beta_2(T, \frac{1}{\delta^2})}$ & Upper bound of $\max_{x \in \mathcal{X}} \max_{t \leq T} \lvert \langle x, \hat{\theta}^{\mathrm{c}}_t \rangle \rvert
$ \\
$T_1(\delta)$ & Time while the algorithm collecting enough information for the exploration
  \\
  $\Deltar_{\text{min}}:=\min_{z\in \mathcal{Z}} (z^*-z)^\top \thetar $ & minimum gap of the reward for $\mathcal{Z}$
  \\
    $\Deltac_{\text{min}}:=\min_{z\in \mathcal{Z}} |\tau-z^\top \thetac|$ & minimum gap of the cost for $\mathcal{Z}$
    \\
    $\overline{V}_T:=\frac{1}{T} (I+\sum_{t=1}^{T} X_t X_t^\top)$ & Empirical variance matrix after $T$ rounds
    \\
    $p_{T+1} := \mathcal{N}\left(\hat{\theta}^\mathrm{r}_{T+1}, \, \eta_\mathrm{r}^{-1} V_{T}^{-1} \right) \otimes$ &posterior distribution for recommendation at time $T$\\  $\mathcal{N}\left(\hat{\theta}^\mathrm{c}_{T+1}, \, \eta\mathrm{c}^{-1} V_{T}^{-1} \right) \, \Big| \, \Theta$
    \\
    $\eta_\mathrm{r}$ &learning rate of reward for min-learner
    \\
    $\eta_\mathrm{c}$ &learning rate of cost for min-learner
    \\
    $p_t := \mathcal{N}(\hat{\theta}^{\mathrm{r}}_t, \eta_r^{-1} V_{t-1}^{-1}) \otimes $ &posterior distribution for sampling at time $t$\\  
    $\mathcal{N}(\hat{\thetac_t}, \eta_c^{-1} V_{t-1}^{-1} ) | \overline{\Theta}_{\hat{z}_t}$
    \\
    $\overline{p}_T:=\frac{1}{T} \sum_{t=1}^T p_t$ & Average posterior distribution for sampling over $T$ rounds
    \\
    $\tilde{p}_t := \mathcal{N}(\hat{\theta}^{\mathrm{r}}_t, \eta_r^{-1} V_{t-1}^{-1}) \otimes $ &optimal posterior distribution for sampling at time $t$\\  
    $\mathcal{N}(\hat{\thetac_t}, \eta_c^{-1} V_{t-1}^{-1} ) | \overline{\Theta}_{z^*}$
    \\
\hline
\end{tabular}
\end{table}

\section{Proof of Theorem~\ref{Thm:lb1}}
\begin{proof}
    We use the standard posterior convergence analysis to derive this theorem. Without loss of generality, we assume that the training and testing arm sets $\mathcal{X}, \mathcal{Z}$ both span $\mathbb{R}^d$. 
Denote the data as $\mathcal{D} = \{(X_i, Y_i)\}_{i=1}^n$ and $X_i \in \mathbb{R}^d$. Define 
$$
X =\begin{bmatrix}
X_1^\top \\ \vdots\\ X_n^\top 
\end{bmatrix} \in\mathbb{R}^{n\times d}
$$

We assume the prior distribution for $\thetar, \thetac$ are multivariate Gaussian:
\[
\thetar \sim \mathcal{N}(\thetar_0, {\lambda}^{-1} I) \quad\mbox{and}\quad
\thetac \sim \mathcal{N}(\thetac_0, {\lambda}^{-1} I)
\]

The likelihood for each reward and cost $Y^{\mathrm{r}}_i, Y^{\mathrm{c}}_i$ are:
\[
P(Y^{\mathrm{r}}_i \mid \thetar) = \mathcal{N}(Y_i^{\mathrm{r}} \mid \langle{\thetar}, X_i \rangle, \sigma^2)
\quad\mbox{and}\quad
P(Y^{\mathrm{c}}_i \mid \thetac) = \mathcal{N}(Y_i^{\mathrm{c}} \mid \langle{\thetac}, X_i \rangle, \gamma^2)
\]
For $n$ observations, the likelihood is:
\[
P(\mathcal{D} \mid \thetar) = \prod_{i=1}^n \mathcal{N}(Y^{\mathrm{r}}_i \mid\langle{\thetar}, X_i \rangle, \sigma^2)
\quad\mbox{and}\quad
P(\mathcal{D} \mid \thetac) = \prod_{i=1}^n \mathcal{N}(Y^{\mathrm{c}}_i \mid\langle{\thetac}, X_i \rangle, \gamma^2)
\]
Let $X \in \mathbb{R}^{n \times d}$ be the feature matrix where each row is $X_i^\top$. Then the vector of rewards and costs $Y^{\mathrm{r}} = (Y^{\mathrm{r}}_1, Y^{\mathrm{r}}_2, \ldots, Y^{\mathrm{r}}_n)^\top, Y^{\mathrm{c}} = (Y^{\mathrm{c}}_1, Y^{\mathrm{c}}_2, \ldots, Y^{\mathrm{c}}_n)^\top$ follows:
\[
Y^{\mathrm{r}} \mid \thetar \sim \mathcal{N}(X \thetar, \sigma^2 I)
\quad\mbox{and}\quad
Y^{\mathrm{c}} \mid \thetac \sim \mathcal{N}(X \thetac, \gamma^2 I)
\]
Using Bayes' theorem, the posterior distribution is:
\[
P(\theta^* \mid \mathcal{D}) \propto P(\mathcal{D} \mid \theta^*) P(\theta^*)
\]
Since both the prior and likelihood are Gaussian, the posterior is also Gaussian, the posterior mean and covariance of the reward are
  \[
  \Sigma^{\mathrm{r}}_n = \left( \lambda I + \frac{1}{\sigma^2} X^\top X \right)^{-1}
\quad\mbox{and}\quad
  \theta^{\mathrm{r}}_n = \Sigma^{\mathrm{r}}_n \left( \lambda \theta^{\mathrm{r}}_0 + \frac{1}{\sigma^2} X^\top Y^{\mathrm{r}} \right)
  \]
where the posterior distribution is
\[
\thetar \mid \mathcal{D} \sim \mathcal{N}(\theta^{\mathrm{r}}_n, \Sigma^{\mathrm{r}}_n)
\]
The posterior mean and covariance of the cost are:
 \[
  \Sigma^{\mathrm{c}}_n = \left( \lambda I + \frac{1}{\gamma^2} X^\top X \right)^{-1}
\quad\mbox{and}\quad
  \theta^{\mathrm{c}}_n = \Sigma^{\mathrm{c}}_n \left( \lambda \theta^{\mathrm{c}}_0 + \frac{1}{\gamma^2} X^\top Y^{\mathrm{c}} \right)
  \]
 where the posterior distribution is
\[
\thetac \mid \mathcal{D} \sim \mathcal{N}(\theta^{\mathrm{c}}_n, \Sigma^{\mathrm{c}}_n).
\]
Assume $\theta^{\mathrm{r}}_0, \theta^{\mathrm{c}}_0$ are zero vectors, then 
\begin{align*}
    \Sigma^{\mathrm{r}}_n &=\sigma^2\left( \lambda \sigma^2 I +  X^\top X \right)^{-1},\\
    \theta^{\mathrm{r}}_n &= \left( \lambda \sigma^2 I +  X^\top X \right)^{-1}   X^\top Y^{\mathrm{r}} ,\\
    \Sigma^{\mathrm{c}}_n&=\gamma^2\left( \lambda \gamma^2 I +  X^\top X \right)^{-1},\\
    \theta^{\mathrm{c}}_n &= \left( \lambda \gamma^2 I +  X^\top X \right)^{-1}   X^\top Y^{\mathrm{c}} .
\end{align*}
Hence the posterior distributions of the expected reward and cost of arm $z\in \mathcal{Z}$ are:
\[
\mur(z) \mid \mathcal{D} \sim \mathcal{N}(\langle \thetar_n, z\rangle, \sigma^2 \lVert z \rVert^2_{V^{-1}_{nr}})
\quad\mbox{and}\quad
\muc(z) \mid \mathcal{D} \sim \mathcal{N}(\langle \thetac_n, z\rangle, \gamma^2 \lVert z \rVert^2_{V^{-1}_{nc}}),
\]
where $V_{nr}:=\lambda \sigma^2 I +  X^\top X$, $V_{nc}:=\lambda \gamma^2 I +  X^\top X$. Let $\mu_n^{\mathrm{r}}(z), \mu_n^{\mathrm{c}}(z)$ denote the sample reward and cost of arm $z\in \mathcal{Z}$ at step $n$ drawn from the posterior distribution $\mur(z)|\mathcal{D}, \muc(z)|\mathcal{D}$. 
Therefore for any suboptimal arm $z\neq z^*$,
\begin{align*}
    \mur(z)-\mu(z^*) \mid \mathcal{D} \sim \mathcal{N}(\langle \thetar_n, z-z^*\rangle, \sigma^2 \lVert z-z^* \rVert^2_{V^{-1}_{nr}}).
\end{align*}

For any arm $z\in\mathcal{Z}$, 
\begin{align*}
    \tau-\muc(z) \mid \mathcal{D} \sim \mathcal{N}(\tau-\langle \thetac_n, z\rangle, \gamma^2 \lVert z \rVert^2_{V^{-1}_{nc}}).
\end{align*}
Denote the reward gap and sample reward gap for suboptimal arm $z\neq z^*$ as $\Deltar(z):=\langle \thetar, z^*-z\rangle, \Deltar_n(z):=\langle \thetar_n, z^*-z\rangle$ respectively,  define the cost gap and sample cost gap for arm $z\in\mathcal{Z}$ by $\Deltac(z):= |\tau-\langle \thetac, z\rangle|, \Deltac_n(z):= |\tau-\langle \thetac_n, z\rangle|$ respectively.

Since 
\begin{align*}
    1-\PP(z_{\text{out}}\neq z^*)&=\PP(\cbr{\mu_n^{\mathrm{c}}(z^*)>\tau}\cup \cbr{\exists z\neq z^*, \mu_n^{\mathrm{r}}(z)\ge\mu_n^{\mathrm{r}}(z^*) \mbox{ and } \mu_n^{\mathrm{c}}(z)\le \tau})
    \\&\stackrel{\cdot}{=}\max_{z\in \mathcal{Z}} \cbr{\PP(\mur_{n}(z)\ge \mur_{n}(z^*), \muc_{n}(z)\le \tau)\cd \one \cbr{z\neq z^*}, \PP (\muc_{n}(z)>\tau) \cd \one \cbr{z=z^*}},
\end{align*}
then the problem reduces to derive the asymptotic expression within the maximum term. The following lemma about Gaussian distribution will be useful in our proofs: 

\begin{lemma}
\label{lemma:gaussiancdf}
    For Gaussian distribution $X \sim \mathcal{N}(\mu, \sigma^2)$ with $\mu\le 0$,
    \begin{align*}
       \frac{x}{1+x^2}\frac{1}{\sqrt{2\pi}}\exp\Big(-\frac{\mu^2}{2\sigma^2}\Big) \le\PP(X\ge0)\le \frac{1}{2}\exp\Big(-\frac{\mu^2}{2\sigma^2}\Big)
    \end{align*}
    where $x=-\frac{ \mu}{\sigma}$.
\end{lemma}

With the above lemma, first if $z=z^*$ we obtain the following lemma.
\begin{lemma}
\label{lem:asp_1}
    $\PP(\muc_{n}(z^*)> \tau)\stackrel{\cdot}{=}\exp\left(-\frac{(\Deltac_n(z^*))^2}{2\gamma^2 \lVert z^*\rVert^2_{V^{-1}_{nc}}}\right)$.
\end{lemma}
\begin{proof}
    Since $\muc_{n}(z^*)\sim \mathcal{N}(\langle \thetac_n, z^*\rangle, \gamma^2 \lVert z^* \rVert^2_{V^{-1}_{nc}})$, applying Lemma~\ref{lemma:gaussiancdf} we have
    \begin{align*}
        \frac{x}{1+x^2}\frac{1}{\sqrt{2\pi}}\exp \left( -\frac{(\langle \thetac_n, z^*\rangle-\tau)^2}{2\gamma^2 \lVert z^* \rVert^2_{V^{-1}_{nc}}} \right)\le \PP(\muc_{n}(z^*)> \tau) \le \frac{1}{2} \exp \left( -\frac{(\langle \thetac_n, z^*\rangle-\tau)^2}{2\gamma^2 \lVert z^* \rVert^2_{V^{-1}_{nc}}} \right)
    \end{align*}
    where $x=\frac{\tau-\langle \thetac_n, z^*\rangle}{\gamma \lVert z^* \rVert_{V^{-1}_{nc}}}$ here.
    Then 
    \begin{align*}
       \frac{1}{n} \log \frac{1}{\sqrt{2\pi}} + \frac{1}{n} \log \frac{x}{1 + x^2}
 \le\frac{1}{n}\log \left(\frac{\PP(\muc_{n}(z^*)> \tau)}{\exp \left( -\frac{(\langle \thetac_n, z^*\rangle-\tau)^2}{2\gamma^2 \lVert z^* \rVert^2_{V^{-1}_{nc}}} \right)}\right)\le \frac{1}{n} \log \frac{1}{2}.
    \end{align*}
    Notice $\lVert z^*\rVert_{V_{nc}^{-1}}=\frac{\lVert z^*\rVert_{V_{\pi}^{-1}}}{\sqrt{n}}$ where $V_{\pi c}:=\frac{\lambda \gamma^2}{n} I+\sum_{i=1}^K \frac{T_{n}(z)}{n}z z^{\top}$ and $T_{n}(z)$ is the total number of pulls of arm $z$ over $n$ rounds. Since there must exist one arm $j^*$ such that $T_{n}(j^*)\ge \frac{n}{K}$, then 
\begin{align*}
   V_{\min}:=\frac{\lambda\gamma^2}{n}I+ \frac{{j^*}{j^*}^{\top}}{K}\lesssim V_{\pi c}\lesssim \frac{\lambda\gamma^2}{n}I+\sum_{i=1}^K z z^\top:=V_{\max}.
\end{align*}
Hence $\lim_{n\rightarrow \infty} \frac{1}{n}\log \frac{x}{1+x^2}=0$ since $\lim_{n\rightarrow \infty} \frac{1}{n}\log n=0$, we have
\begin{align*}
   \lim_{n\rightarrow \infty}\frac{1}{n}\log \left(\frac{\PP(\muc_{n}(z^*)> \tau)}{\exp \left( -\frac{(\langle \thetac_n, z^*\rangle-\tau)^2}{2\gamma^2 \lVert z^* \rVert^2_{V^{-1}_{nc}}} \right)}\right)=0.
\end{align*}
\end{proof}

\begin{lemma}
\label{lem:asp2}
For $z\in {\mathcal{F}}\cap {\mathcal{S}}$,
    $\PP(\mur_{n}(z)\ge \mur_{n}(z^*), \muc_{n}(z)\le \tau)\stackrel{\cdot}{=}\exp\left(-\frac{(\Deltar_n(z))^2}{2\sigma^2 \lVert z-z^* \rVert^2_{V^{-1}_{nr}}}\right)$.
\end{lemma}

\begin{proof}
First $\PP(\muc_{\infty}(z)\le \tau)=1$ since $z\in \mathcal{F}$, then
 by Lemma~\ref{lemma:gaussiancdf},
\begin{align*}
   \frac{x}{1+x^2} \frac{1}{\sqrt{2\pi}}\exp\left( -\frac{(\Deltar_{n}(z))^2}{2\sigma^2\lVert z-z^* \rVert^2_{V^{-1}_{nr}}}\right)\le\PP(\mur_{n}(z)\ge \mur_{n}(z^*) )\le \frac{1}{2}\exp \left( -\frac{(\Deltar_n(z))^2}{2\sigma^2 \lVert z-z^* \rVert^2_{V^{-1}_{nr}}}\right),
\end{align*}
where $x=\frac{\Deltar_{n}(z)}{\sigma\lVert z-z^*\rVert_{V^{-1}_{nr}}}$.

Then
\begin{align*}
    \frac{1}{n}\log \frac{1}{\sqrt{2\pi}}+\frac{1}{n}\log\frac{x}{1+x^2}\le\frac{1}{n} \log \left(\frac{\PP(\mur_{n}(z)\ge \mur_{n}(z^*) \mid \mathcal{D})}{\exp\left(-\frac{(\langle \thetar_n, z-z^*\rangle)^2}{2\sigma^2 \lVert z-z^* \rVert^2_{V^{-1}_n}}\right)}\right)\le \frac{1}{n} \log \frac{1}{2}.
\end{align*}
Notice that $\lVert z-z^*\rVert_{V_{nr}^{-1}}=\frac{\lVert z-z^*\rVert_{V_{\pi}^{-1}}}{\sqrt{n}}$ where $V_{\pi r}:=\frac{\lambda \sigma^2}{n} I+\sum_{i=1}^K \frac{T_{n}(z)}{n}z z^{\top}$ and $T_{n}(z)$ is the total number of pulls of arm $z$ over $n$ rounds. Since there must exist one arm $j^*$ such that $T_{n}(j^*)\ge \frac{n}{K}$, then 
\begin{align*}
   V_{\min}:=\frac{\lambda\sigma^2}{n}I+ \frac{{j^*}{j^*}^{\top}}{K}\lesssim V_{\pi r}\lesssim \frac{\lambda\sigma^2}{n}I+\sum_{i=1}^K z z^\top:=V_{\max}.
\end{align*}
Hence $\lim_{n\rightarrow \infty} \frac{1}{n}\log \frac{x}{1+x^2}=0$ since $\lim_{n\rightarrow \infty} \frac{1}{n}\log n=0$, we have
\begin{align*}
    \lim_{n\rightarrow \infty} \frac{1}{n} \log \left(\frac{\PP(\mur_{n}(z)\ge \mur_{n}(z) \mid \mathcal{D})}{\exp\left(-\frac{(\langle \thetar_n, z-z^*\rangle)^2}{2\sigma^2 \lVert z-z^* \rVert^2_{V^{-1}_{nr}}}\right)}\right)=0.
\end{align*}
Hence 
\begin{align*}
    \PP(\mur_{n}(z)\ge \mur_{n}(z^*), \muc_{n}(z)\le \tau)\stackrel{\cdot}{=}\exp\left(-\frac{(\Deltar_n(z))^2}{2\sigma^2 \lVert z-z^* \rVert^2_{V^{-1}_{nr}}}\right).
\end{align*}
\end{proof}

\begin{lemma}
\label{lem:asp3}
    For $z\in \overline{\mathcal{F}}\cap \overline{\mathcal{S}}$,
    $\PP(\mur_{n}(z)\ge \mur_{n}(z^*), \muc_{n}(z)\le \tau)\stackrel{\cdot}{=}\exp\left(-\frac{(\Deltac_n(z))^2}{2\gamma^2 \lVert z \rVert^2_{V^{-1}_{nc}}}\right)$
\end{lemma}
\begin{proof}
    First $\PP(\mur_{\infty}(z)\ge \mur_{\infty}(z^*))=1$ since $i\in \mathcal{S}$, then with the similar procedure as in Lemma~\ref{lem:asp2},
\begin{align*}
   \PP(\muc_{n}(z)\le \tau )\stackrel{\cdot}{=} \frac{1}{2}\exp \left( -\frac{(\Deltac_{n}(z))^2}{2\gamma^2 \lVert z \rVert^2_{V^{-1}_{nc}}}\right).
\end{align*}
Hence $$\PP(\mur_{n}(z)\ge \mur_{n}(z^*), \muc_{n}(z)\le \tau)\stackrel{\cdot}{=}\exp\left(-\frac{(\Deltac_n(z))^2}{2\gamma^2 \lVert z \rVert^2_{V^{-1}_{nc}}}\right).$$
\end{proof}

\begin{lemma}
\label{lem:asp4}
    For $i\in \overline{\mathcal{F}}\cap {\mathcal{S}}$,
    $\PP(\mur_{n,i}\ge \mur_{n,1}, \muc_{n,i}\le \tau)\stackrel{\cdot}{=}\exp\left(-\frac{(\Deltac_n(z))^2}{2\gamma^2 \lVert z \rVert^2_{V^{-1}_{nc}}}\right)\cd\exp\left(-\frac{(\Deltar_n(z))^2}{2\sigma^2 \lVert z-z^* \rVert^2_{V^{-1}_{nr}}}\right)$
\end{lemma}
\begin{proof}
    Since $\PP(\mur_{n,i}\ge \mur_{n,1}, \muc_{n,i}\le \tau)= \PP(\mur_{n,i}\ge \mur_{n,1}) \cd \PP(\muc_{n,i}\le \tau)$, then with the similar procedure as in Lemma~\ref{lem:asp2}, this lemma can be proved.
\end{proof}

For convenience, denote $\mathcal{A}_1:= \overline{\mathcal{F}}\cap \overline{\mathcal{S}}, \mathcal{A}_2:= {\mathcal{F}}\cap {\mathcal{S}}, \mathcal{A}_3:= \overline{\mathcal{F}}\cap {\mathcal{S}} $. 

To summarize,
\begin{align*}
    &1-\PP(z_{\text{out}}\neq z^*) \\*
    & \stackrel{\cdot}{=}\max_{z\in \mathcal{Z}} \cbr{\PP(\mur_{n}(z)\ge \mur_{n}(z^*), \muc_{n}(z)\le \tau)\cd \one \cbr{z\neq z^*}, \PP (\muc_{n}(z)>\tau) \cd \one \cbr{z=z^*}}\\
    & \stackrel{\cdot}{=} \max_{z\in \mathcal{Z}}  \bigg\{
        \exp\left(-\frac{(\Deltac_n(z^*))^2}{2\gamma^2 \lVert z^*\rVert^2_{V^{-1}_{nc}}}\right), \exp\left(-\frac{(\Deltac_n(z))^2}{2\gamma^2 \lVert z \rVert^2_{V^{-1}_{nc}}}\right) \one\{z \in \mathcal{A}_1\}, \\*
        & \quad \exp\left(-\frac{(\Deltar_n(z))^2}{2\sigma^2 \lVert z - z^* \rVert^2_{V^{-1}_{nr}}}\right) \one\{z \in \mathcal{A}_2\},  \exp\left(-\frac{(\Deltac_n(z))^2}{2\gamma^2 \lVert z \rVert^2_{V^{-1}_{nc}}}\right) \cdot 
        \exp\left(-\frac{(\Deltar_n(z))^2}{2\sigma^2 \lVert z - z^* \rVert^2_{V^{-1}_{nr}}}\right) \one\{z \in \mathcal{A}_3\} 
        \bigg\}.
\end{align*}

Hence 
\begin{align*}
    1-\PP(z_{\text{out}}\neq z^*) & \stackrel{\cdot}{=} \exp \bigg(-n \min_{z\in\mathcal{Z}} \bigg\{ \frac{(\Deltac_n(z))^2}{2\gamma^2\lVert z\rVert^2_{V_w^{-1}}} \one \cbr{z\in \mathcal{A}_1}, 
      \frac{(\Deltar_n(z))^2}{2\sigma^2\lVert z-z^*\rVert^2_{V_w^{-1}}} \one \cbr{z\in \mathcal{A}_2}, 
    \\& \qquad\quad \bigg( \frac{(\Deltac_n(z))^2}{2\gamma^2\lVert z\rVert^2_{V_w^{-1}}}+\frac{(\Deltar_n(z))^2}{2\sigma^2\lVert z-z^*\rVert^2_{V_w^{-1}}}\bigg) \one \cbr{i\in \mathcal{A}_3} ,
     \frac{(\Deltac_n(z))^2}{2\gamma^2\lVert z^*\rVert^2_{V_w^{-1}}}
    \bigg\}
    \bigg)
    \\&\ge \exp \bigg(-n \max_{w\in W}\min_{z\in \mathcal{Z}} \bigg\{ \frac{(\Deltac(z))^2}{2\gamma^2\lVert z\rVert^2_{V_w^{-1}}} \one \cbr{z\in \mathcal{A}_1}, 
    \frac{(\Deltar(z))^2}{2\sigma^2\lVert z-z^*\rVert^2_{V_w^{-1}}} \one \cbr{z\in \mathcal{A}_2}, 
    \\& \qquad\quad \bigg( \frac{(\Deltac(z))^2}{2\gamma^2\lVert x_i\rVert^2_{V_w^{-1}}}+\frac{(\Deltar(z))^2}{2\sigma^2\lVert z-z^*\rVert^2_{V_w^{-1}}}\bigg) \one \cbr{z\in \mathcal{A}_3},
      \frac{(\Deltac(z^*))^2}{2\gamma^2\lVert z^*\rVert^2_{V_w^{-1}}}\bigg\}
    \bigg)
    \\&=\exp(-n\Gamma)
\end{align*}

for any sampling rule where $V_w:=\sum_{i=1}^K w_{i} x_i x_i^{\top}$ and $w_i:=\lim_{n\rightarrow \infty}\frac{T_{i,n}}{n}$.

Let $W:=\cbr{w=(w_1, w_2, \ldots, w_K): \sum_i w_i=1, w_i\ge0 ,\, \forall\, i\in [K]}$. The exponential rate in the lower bound can be written as
\begin{align*}
    \Gamma&:= \max_{w\in W}\min_{z\in \mathcal{Z}} \bigg\{ \frac{(\Deltac(z))^2}{2\gamma^2\lVert z\rVert^2_{V_w^{-1}}} \one \cbr{z\in \mathcal{A}_1}, 
    \frac{(\Deltar(z))^2}{2\sigma^2\lVert z-z^* \rVert^2_{V_w^{-1}}} \one \cbr{z\in \mathcal{A}_2}, 
    \\& \quad \bigg( \frac{(\Deltac(z))^2}{2\gamma^2\lVert z\rVert^2_{V_w^{-1}}}+\frac{(\Deltar(z))^2}{2\sigma^2\lVert z-z^*\rVert^2_{V_w^{-1}}}\bigg) \one \cbr{z\in \mathcal{A}_3},
   \frac{(\Deltac(z^*))^2}{2\gamma^2\lVert z^*\rVert^2_{V_w^{-1}}}\bigg\}.
\end{align*}
\end{proof}

\section{Proof of Theorem~\ref{Thm:equ_lowerboundterm}}

\begin{proof}
    We start from the right-hand side of the equation in Theorem~\ref{Thm:equ_lowerboundterm}, this optimization problem is equivalent to the following:
    \begin{align*}
       &\min_{\theta_1, \theta_2} \frac{1}{2} \bigg(  \frac{\lVert\theta_1-\thetar\rVert^2_{V_w^{}}}{\sigma^2}+\frac{\lVert\theta_2-\thetac\rVert^2_{V_w^{}}}{\gamma^2}\bigg)
        \\&\textbf{s.t. }\quad (\theta_1,\theta_2)\in \overline{\Theta}_{z^*}
    \end{align*}
    By the definition of feasibility, this optimization problem is equal to:
    \begin{align*}
        &\min_{\theta_1, \theta_2} \frac{1}{2} \bigg(  \frac{\lVert\theta_1-\thetar\rVert^2_{V_w^{}}}{\sigma^2}+\frac{\lVert\theta_2-\thetac\rVert^2_{V_w^{}}}{\gamma^2}\bigg)
        \\*&\textbf{s.t. }\quad \theta_2^\top z^*>\tau \mbox{ or }\exists z\in \mathcal{Z}\backslash z^*: \theta_1^\top z\ge \theta_1^\top z^* \mbox{ and } \theta_2^\top z \le \tau 
    \end{align*}
Furthermore this problem can be divided into two optimization problem and the minimum of the solutions of this two problem will equal to the solution of the original optimization problem, the two subproblems are:
\begin{align*}
        &\min_{\theta_1, \theta_2} \frac{1}{2} \bigg(  \frac{\lVert\theta_1-\thetar\rVert^2_{V_w^{}}}{\sigma^2}+\frac{\lVert\theta_2-\thetac\rVert^2_{V_w^{}}}{\gamma^2}\bigg)
        \\&\textbf{s.t. }\quad \exists z\in \mathcal{Z}\backslash z^*: \theta_1^\top z\ge \theta_1^\top z^* \mbox{ and } \theta_2^\top z \le \tau 
    \end{align*}
and       
\begin{align}&\min_{\theta_1, \theta_2} \frac{1}{2} \bigg(  \frac{\lVert\theta_1-\thetar\rVert^2_{V_w^{}}}{\sigma^2}+\frac{\lVert\theta_2-\thetac\rVert^2_{V_w^{}}}{\gamma^2}\bigg)
        \\&\textbf{s.t. }\quad \theta_2^\top z^*>\tau  
    \end{align}

Then the Lagrange multiplier of the first suboptimization problem is 
\[
\mathcal{L}(\theta_1, \theta_2, \mu, \nu) = 
\frac{1}{2} \left(\frac{\lVert\theta_1-\thetar\rVert^2_{V_w^{}}}{\sigma^2}+\frac{\lVert\theta_2-\thetac\rVert^2_{V_w^{}}}{\gamma^2}
\right)
- \mu (\theta_1^\top z - \theta_1^\top z^*) 
- \nu (\tau - \theta_2^\top z).
\]

Derivative with respect to \(\theta_1\):
\[
\frac{\partial \mathcal{L}}{\partial \theta_1} = \frac{1}{\sigma^2} V_w (\theta_1 - \theta^r) - \mu (z - z^*),
\]
and solving $\frac{\partial \mathcal{L}}{\partial \theta_1}=0$ gives:
\begin{align}
\label{eq:1}
    \theta_1 - \thetar = {\mu}{\sigma^2} V_w^{-1} (z - z^*).
\end{align}

Derivative with respect to \(\theta_2\):
\[
\frac{\partial \mathcal{L}}{\partial \theta_2} = \frac{1}{\gamma^2} V_w (\theta_2 - \thetac) + \nu z,
\]
and solving $\frac{\partial \mathcal{L}}{\partial \theta_2}=0$ gives:
\[
\theta_2 - \thetac =- {\nu}{\gamma^2} V_w^{-1} z.
\]

Derivative with respect to \(\mu\):
\[
\frac{\partial \mathcal{L}}{\partial \mu} = -\theta_1^\top z + \theta_1^\top z^*,
\]
and the condition is:
\[
\theta_1^\top z = \theta_1^\top z^*.
\]

Derivative with respect to \(\nu\):
\[
\frac{\partial \mathcal{L}}{\partial \nu} = -\tau + \theta_2^\top z,
\]
and the condition is:
\[
\theta_2^\top z = \tau.
\]

From Eqn.~\eqref{eq:1},   we have
\begin{align*}
    \lVert\thetar-\theta_1\rVert^2_{V_w}={\mu^2}{\sigma^4}\lVert z-z^* \rVert^2_{V_w^{-1}},
\end{align*}
then
\begin{align}
\label{eq:3}
    {\mu}{\sigma^2}=\frac{\lVert\thetar-\theta_1\rVert_{V_w}}{\lVert z-z^* \rVert_{V_w^{-1}}}~.
\end{align}
Plugging  Eqn.~\eqref{eq:3} into Eqn.~\eqref{eq:1},  we have
\begin{align*}
    (z-z^*)^\top (\theta_1-\thetar)=\lVert z-z^* \rVert_{V_w^{-1}} \lVert\thetar-\theta_1\rVert_{V_w}
\end{align*}
or
\begin{align}
\label{eq:13}
    \lVert\thetar-\theta_1\rVert^2_{V_w}=\frac{((z-z^*)^\top (\theta_1-\thetar))^2}{\lVert z-z^* \rVert^2_{V_w^{-1}}}
\end{align}
Similarly, we   also have
\begin{align*}
    z^\top (\thetac-\theta_2)=\lVert\theta_2-\thetac \rVert_{V_w} \lVert z \rVert_{V_w^{-1}}.\nonumber
\end{align*}
Substitute $\tau=\theta_2^\top z$ into it,
\begin{align}
    z^\top \thetac-\tau=\lVert\theta_2-\thetac \rVert_{V_w} \lVert z \rVert_{V_w^{-1}}\nonumber
\end{align}
or
\begin{align}
\label{eq:16}
    \lVert\theta_2-\thetac \rVert^2_{V_w}=\frac{(\tau-z^\top \thetac)^2}{\lVert z \rVert^2_{V_w^{-1}}}.
\end{align}

With Eqns.~\eqref{eq:13} and~\eqref{eq:16}, 
the optimization problem becomes
\begin{align}
        &\min_{\theta_1, \theta_2} \frac{1}{2} \bigg(  \frac{((z-z^*)^\top (\thetar-\theta_1))^2}{\sigma^2 \lVert z-z^* \rVert^2_{V_w^{-1}}}+\frac{(\tau-z^\top \thetac)^2}{\gamma^2\lVert z \rVert^2_{V_w^{-1}}}\bigg)\label{eq:20}
        \\&\textbf{s.t. }\quad \exists z\in \mathcal{Z}\backslash z^*: \theta_1^\top z\ge \theta_1^\top z^* \mbox{ and } \theta_2^\top z \le \tau. \notag
    \end{align}

If $z\in\mathcal{A}_1$, to minimize the objective function, $\theta_1=\thetar$, the minimum objective function  value will then be $\frac{(\Deltac(z))^2}{2\gamma^2\lVert z\rVert^2_{V_w^{-1}}} \one \cbr{z\in \mathcal{A}_1}$;

if $z\in\mathcal{A}_2$, to minimize the objective function, $\theta_2=\thetac$ since in this case $\tau=\theta_2^\top z=(\thetac)^\top z$, which will make the second term in Eqn.~\eqref{eq:20} to be zero. On the other hand for the first term, since
\begin{align*}
    (z-z^*)^\top (\theta_1-\thetar)=(z-z^*)^\top \theta_1+(z^*-z)^\top \thetar=(z-z^*)^\top \theta_1+\Deltar(z),
\end{align*}
then with the constraint $\theta_1^\top z\ge \theta_1^\top z^*$, the minimum will occur at $\theta_1$ such that $\theta_1^\top z= \theta_1^\top z^*$. Hence the minimum objective value will then be $\frac{(\Deltar(z))^2}{2\sigma^2\lVert z-z^* \rVert^2_{V_w^{-1}}} \one \cbr{z\in \mathcal{A}_2}$.
Similarly, we can see if $z\in\mathcal{A}_3$, the minimum will be $( \frac{(\Deltac(z))^2}{2\gamma^2\lVert z\rVert^2_{V_w^{-1}}}+\frac{(\Deltar(z))^2}{2\sigma^2\lVert z-z^*\rVert^2_{V_w^{-1}}}) \one \cbr{z\in \mathcal{A}_3}$.

For the second optimization problem, recall that it is
\begin{align*}
        &\min_{\theta_1, \theta_2} \frac{1}{2} \bigg(  \frac{\lVert\theta_1-\thetar\rVert^2_{V_w^{}}}{\sigma^2}+\frac{\lVert\theta_2-\thetac\rVert^2_{V_w^{}}}{\gamma^2}\bigg)
        \\&\textbf{s.t. }\quad \theta_2^\top z^*>\tau . 
    \end{align*}
First since the first term and second term are independent, the minimum occurs when $\theta_1=\thetar$, then we assume $\theta_2^\top z^*=\tau+\alpha$, where $\alpha\ge0$, then the Langrangian is 
\begin{align*}
    \mathcal{L}(\theta_2, \mu)=\frac{\lVert\theta_2-\thetac\rVert^2_{V_w^{}}}{2\gamma^2}+\mu (\tau-\theta_2^\top z^*+\alpha).
\end{align*}
Then 
\begin{align*}
    \frac{\partial \mathcal{L}}{\partial \theta_2}=\frac{1}{\gamma^2} V_w (\theta_2 - \thetac) -\mu z^*=0.
\end{align*}
therefore
\begin{align*}
    \theta_2-\thetac=\gamma^2 \mu V_w^{-1}z^*,
\end{align*}
which implies that 
\begin{align*}
    \lVert \theta_2-\thetac \rVert^2_{V_w}=\gamma^4 \mu^2 \lVert z^*\rVert^2_{{V_w}^{-1}},
\end{align*}
and which leads to
\begin{align*}
    \gamma^2\mu=\frac{\lVert \theta_2-\thetac \rVert_{V_w}}{\lVert z^*\rVert_{{V_w}^{-1}}}.
\end{align*}

Hence
\begin{align*}
    (\theta_2-\thetac)^\top z^*={\lVert \theta_2-\thetac \rVert_{V_w}}{\lVert z^*\rVert_{{V_w}^{-1}}}.
\end{align*}
The optimization problem then becomes 
\begin{align*}
        &\min_{\theta_2} \frac{((\theta_2-\thetac)^\top z^*)^2}{2\gamma^2\lVert z^*\rVert^2_{{V_w}^{-1}}}
        \\&\textbf{s.t. }\quad \theta_2^\top z^*=\tau+\alpha  ,
    \end{align*}
since $(\theta_2-\thetac)^\top z^*=\tau-(\thetac)^\top z^*+\alpha=\Deltac(z^*)+\alpha$, then when we choose $\alpha\rightarrow 0$, we get the minimum value $\frac{(\Deltac(z^*))^2}{2\gamma^2\lVert z^*\rVert^2_{{V_w}^{-1}}}$.

To summarize, the original optimization problem is equivalent to 
\begin{align*}
    &\min_{z\in \mathcal{Z}} \bigg\{ \frac{(\Deltac(z))^2}{2\gamma^2\lVert z\rVert^2_{V_w^{-1}}} \one \cbr{z\in \mathcal{A}_1}, 
    \frac{(\Deltar(z))^2}{2\sigma^2\lVert z-z^* \rVert^2_{V_w^{-1}}} \one \cbr{z\in \mathcal{A}_2}, 
    \\& \quad \bigg( \frac{(\Deltac(z))^2}{2\gamma^2\lVert z\rVert^2_{V_w^{-1}}}+\frac{(\Deltar(z))^2}{2\sigma^2\lVert z-z^*\rVert^2_{V_w^{-1}}}\bigg) \one \cbr{z\in \mathcal{A}_3},
   \frac{(\Deltac(z^*))^2}{2\gamma^2\lVert z^*\rVert^2_{V_w^{-1}}}\bigg\}.
\end{align*}
Hence the theorem is proved.

\section{Proof of Theorem~\ref{upperbound}}
During the proof of Theorem~\ref{upperbound}, we need the help of the good event lemmas in Section~\ref{sec:goodevents}. When these good events hold, the error probability of our algorithm is
\begin{align*}  &\mathbb{P}_{(\thetar_{T+1},\thetac_{T+1})\sim p_{T+1}} (\hat{z}_\text{{out}} \neq z^*)
\\&\stackrel{\cdot}{=}\frac{\int_{(\theta_1,\theta_2)\in \overline{\Theta}_{z^*}}p_{T+1}(\theta_1,\theta_2)\,\mathrm{d}\theta_1\mathrm{d}\theta_2}{\int_{(\theta_1,\theta_2)\in {\Theta}_{z^*}}p_{T+1}(\theta_1,\theta_2)\,\mathrm{d}\theta_1\mathrm{d}\theta_2}
\\&=\frac{\int_{(\theta_1,\theta_2)\in \overline{\Theta}_{z^*}}p_{T+1}(\theta_1,\theta_2)/p_{T+1}(\thetar,\thetac)\,\mathrm{d}\theta_1\mathrm{d}\theta_2}{\int_{(\theta_1,\theta_2)\in {\Theta}_{z^*}}p_{T+1}(\theta_1,\theta_2)/p_{T+1}(\thetar,\thetac)\,\mathrm{d}\theta_1\mathrm{d}\theta_2}
\\&\stackrel{(a)}{\stackrel{\cdot}{=}}\frac{\int\limits_{(\theta_1,\theta_2)\in \overline{\Theta}_{z^*}} \exp \left( M\right)\,\mathrm{d}\theta_1\mathrm{d}\theta_2}{\int\limits_{(\theta_1,\theta_2)\in {\Theta}_{z^*}}\exp \left( M\right)\,\mathrm{d}\theta_1\mathrm{d}\theta_2}
\end{align*}
where $M:=-\frac{1}{2} \left( 
\frac{\left\| \truethetar - \theta_1 \right\|^2_{{V}_T}}{\sigma^2} 
+ \frac{\left\| \truethetac - \theta_2 \right\|^2_{{V}_T}}{\gamma^2} 
\right) $ and $(a)$ comes from the definition of the good event $E_{5,\infty}$. Notice that when $T$ is enough large, at the last round the probability that all arms are empirically identified as infeasible is asymptotically equal to zero according to the good event in Lemma~\ref{lem:GoodEvent_4}.
Furthermore, using the Laplace approximation (stated in Lemma~\ref{lem:GaussianApproximation}) and the fact that 
\begin{align*}
    \inf_{(\theta_1,\theta_2)\in {\Theta}}  
\frac{\left\| \truethetar - \theta_1 \right\|^2_{{V}_T}}{\sigma^2} 
+ \frac{\left\| \truethetac - \theta_2 \right\|^2_{{V}_T}}{\gamma^2} =0~,
\end{align*}
we have
\begin{align*}   \mathbb{P}_{(\thetar_{T+1},\thetac_{T+1})\sim p_{T+1}} (\hat{z}_\text{{out}} \neq z^*) \stackrel{\cdot}{=}\exp \left( -\frac{T}{2} \inf_{(\theta_1,\theta_2)\in \overline{\Theta}_{z^*}}\left( 
\frac{\left\| \truethetar - \theta_1 \right\|^2_{\overline{V}_T}}{\sigma^2} 
+ \frac{\left\| \truethetac - \theta_2 \right\|^2_{\overline{V}_T}}{\gamma^2} 
\right) \right).
\end{align*}
Then combining with the good event $E_{6,\delta}$, we have
\begin{align*}  \mathbb{P}_{(\thetar_{T+1},\thetac_{T+1})\sim p_{T+1}} (\hat{z}_\text{{out}} \neq z^*)\stackrel{\cdot}{=}
\exp \left( -T \Gamma \right).
\end{align*}
To summarize, with the choice $\delta=\frac{1}{T}$, with probability 1, 
\begin{align*}
        \lim_{T \to \infty} -\frac{1}{T} \log \mathbb{P}_{(\thetar_{T+1},\thetac_{T+1})\sim p_{T+1}} (\hat{z}_\text{{out}} \neq z^*) = \Gamma~.
    \end{align*}
\section{Good Event Lemmas}
\label{sec:goodevents}
\begin{lemma}
    Define the good event 1 as
    \begin{align*}
        E_{1,\delta}:= \cbr{\lVert \hthetart-\thetar  \rVert_{V_{t-1}}\le \sqrt{\beta_1(t, \frac{1}{\delta^2})} \quad\text{and}\quad \lVert \hthetact-\thetac  \rVert_{V_{t-1}}\le \sqrt{\beta_2(t, \frac{1}{\delta^2}) }\quad\forall t\in[T]}
    \end{align*}
    where $\beta_1(t,\frac{1}{\delta^2}):= (S_1 + \sigma\sqrt{2 \log(\frac{1}{\delta^2}) + d \log\left(\frac{d + t L^2}{d}\right)})^2$, $\beta_2(t,\frac{1}{\delta^2}):= (S_2 + \gamma\sqrt{2 \log(\frac{1}{\delta^2}) + d \log\left(\frac{d + t L^2}{d}\right)})^2$.
    Then with probability $1-2\delta$, good event $E_{1,\delta}$ holds.
\end{lemma}
\begin{proof}
    This lemma is proved directly with union bound with Lemmas~\ref{Ellip_potential_reward} and~\ref{Ellip_potential_cost}.
\end{proof}
\begin{lemma}
Define good event
\begin{align*}
    E_{2,\delta}:= \cbr{\max_{x\in \mathcal{X}} \lvert \langle \hthetart, x\rangle\rvert \le B_{1} \quad\text{and}\quad \max_{x\in \mathcal{X}} \lvert \langle \hthetact, x\rangle\rvert \le B_{2} \quad \forall t\in[ T]}
\end{align*}
where $B_1= L R_1+L \sqrt{\beta_1(T, \frac{1}{\delta^2})}$, $B_2=L R_2+L \sqrt{\beta_2(T, \frac{1}{\delta^2})}$. 
Then $$E_{1,\delta}\subseteq E_{2,\delta}$$
\end{lemma}

\begin{proof}
     Follow the proof of Lemma D.2 in \cite{li2024optimal}, when event $E_{1,\delta}$ holds,
    $\text{for any } x \in \mathcal{X}, $
\begin{align*}
   \lvert \langle x, \hthetart \rangle \rvert&= \lvert\langle x, \thetar \rangle\rvert + \lvert\langle x, \hthetart - \thetar \rangle \rvert\\
    &\leq L R_1 + \|x\|_{V_{t-1}^{-1}} \|\hthetart - \thetar\|_{V_{t-1}} \\
    &\leq L R_1 + \|x\|_{V_{t-1}^{-1}} \sqrt{\beta_1(T, \frac{1}{\delta^2})}
    \\&\leq L R_1 + L \sqrt{\beta_1(T, \frac{1}{\delta^2})}
\end{align*}
and
\begin{align*}
    \lvert\langle x, \hthetact \rangle \rvert &= \lvert\langle x, \thetac \rangle \rvert + \lvert\langle x, \hthetact - \thetac \rangle \rvert \\
    &\leq L R_2 + \|x\|_{V_{t-1}^{-1}} \|\hthetact - \thetac\|_{V_{t-1}} \\
    &\leq L R_2 + \|x\|_{V_{t-1}^{-1}} \sqrt{\beta_2(T, \frac{1}{\delta^2})}
    \\&\leq L R_2 + L \sqrt{\beta_2(T, \frac{1}{\delta^2})}.
\end{align*}
Hence $\max_{x\in \mathcal{X}} \lvert \langle \hthetart, x\rangle\rvert \le B_{1}$ and $\max_{x\in \mathcal{X}} \lvert \langle \hthetact, x\rangle\rvert \le B_{2}$ for all $t\in[T]$.
\end{proof}

\begin{lemma}[Lemma C.14 in \citet{li2024optimal}]
    Define the event $E_{3,\delta}:= \cbr{V_t \ge \frac{3}{4}A(\lambda^G), \forall t\ge T_1(\delta), x\in \mathcal{X}}$ where $T_1(\delta):=\max_{x \in \mathcal{X}} \left( \frac{6 \sqrt{\log(|\mathcal{X}|  \frac{T}{\delta})}}{\lambda_x^G} \right)^4.$
    Then with probability $1-\delta$, the event $E_{3,\delta}$ holds.
\end{lemma}

\begin{lemma}
    Define the good event $$E_{4,\delta}:=\cbr{\hat{z}_t=z^*,\quad \forall t> \max \cbr{T_1(\delta)+1, T_2(\delta)}}$$ where $T_2(\delta):= \max \cbr{(\frac{m_\text{max} |\mathcal{X}| }{\Deltar_\text{min}} \sqrt{d\beta_1 (t, \frac{1}{\delta^2})})^{\frac{8}{3}}, (\frac{ n_\text{max}|\mathcal{X}|}{\Deltac_\text{min}}\sqrt{d\beta_2 (t, \frac{1}{\delta^2})})^{\frac{8}{3}}}.$ 
    Then $E_{1,\delta}\cap E_{3,\delta} \subseteq E_{4,\delta}.$
\end{lemma}
\begin{proof}
    First from Lemma C.15 in \cite{li2024optimal} we have under $E_{1,\delta}\cap E_{3,\delta}$, for any $t> T_1(\delta)+1$, we have for any $x\in \mathcal{X}$,
    \begin{align}
    \label{eq:concentration:reward}
        \lvert\langle x, \hthetart - \thetar \rangle\rvert \leq \sqrt{\frac{d}{t^{3/4}} \beta_1(t, \frac{1}{\delta^2})}.
    \end{align}
    and
    \begin{align*}
        \lvert\langle x, \hthetact - \thetac \rangle \rvert\leq \sqrt{\frac{d}{t^{3/4}} \beta_2(t, \frac{1}{\delta^2})}.
    \end{align*}
    Since $\mathcal{X}$ spans $\mathbb{R}^d$,  there exists $z^*-z=\sum_{x\in \mathcal{X}} m(x) x$, $z=\sum_{x\in \mathcal{X}} n(x) x$ for any $z\in \mathcal{Z}$ where $0\le m(x), n(x)<\infty$.
Then 
\begin{align*}
    \lvert\langle z^*-z, \hthetart - \thetar \rangle\rvert &=\bigg\lvert\Big\langle \sum_{x\in \mathcal{X}} m(x) x, \hthetart - \thetar \Big\rangle\bigg\rvert
    \\&\le \sum_{x\in \mathcal{X}} m(x) \frac{d}{t^{3/4}} \beta_1(t, \frac{1}{\delta^2})
    \\&\le m_\text{max} |\mathcal{X}|\sqrt{\frac{d }{t^{3/4}} \beta_1(t, \frac{1}{\delta^2})}
\end{align*}
and
\begin{align*}
    \lvert\langle z, \hthetart - \thetar \rangle\rvert &=\bigg\lvert\Big\langle \sum_{x\in \mathcal{X}} n(x) x, \hthetart - \thetar\Big\rangle\bigg\rvert
    \\&\le \sum_{x\in \mathcal{X}} n(x) \frac{d}{t^{3/4}} \beta_2(t, \frac{1}{\delta^2})
    \\&\le n_\text{max} |\mathcal{X}|\sqrt{\frac{d }{t^{3/4}} \beta_2(t, \frac{1}{\delta^2})},
\end{align*}
where $m_\text{max}:=\max_{x\in\mathcal{X}} m(x), n_\text{max}:=\max_{x\in\mathcal{X}} n(x)~.$

Therefore when $t> \max \cbr{(\frac{m_\text{max} |\mathcal{X}| }{\Deltar_\text{min}} \sqrt{d\beta_1 (t, \frac{1}{\delta^2})})^{\frac{8}{3}}, (\frac{ n_\text{max}|\mathcal{X}|}{\Deltac_\text{min}}\sqrt{d\beta_2 (t, \frac{1}{\delta^2})})^{\frac{8}{3}}}$, 
for all $x\in\mathcal{X}$,
\begin{align*}
    \lvert\langle z^*-z, \hthetart - \thetar \rangle\rvert< \Deltar_\text{min}
\end{align*}
and
\begin{align*}
    \lvert\langle z, \hthetart - \thetar \rangle\rvert< \Deltac_\text{min}~.
\end{align*}
Hence $\hat{z}_t=z^*$ and 
$E_{1,\delta}\cap E_{3,\delta} \subseteq E_{4,\delta}$.
\end{proof}

\begin{lemma}
    
    Define $E_{5,\infty}:=\lim_{T\to \infty}\sup_{\theta_1 \in \Theta_1, \theta_2\in\Theta_2} \frac{1}{T} \left| 
\log \frac{p_{T+1}(\theta_1, \theta_2)}{p_{T+1}(\truethetar,\truethetac)} 
+ \frac{T}{2} (\frac{\left\| \truethetar - \theta_1 \right\|^2_{\overline{V}_T}}{\sigma^2}+\frac{\left\| \truethetac - \theta_2 \right\|^2_{\overline{V}_T}}{\gamma^2})  
\right| =0~,$  with probability 1, the good event $E_{5,\infty}$ holds, which also means $\frac{p_{T+1}(\theta_1, \theta_2)}{p_{T+1}(\truethetar,\truethetac)}\stackrel{\cdot}{=}\exp \bigg( -\frac{1}{2} (\frac{\left\| \truethetar - \theta_1 \right\|^2_{{V}_T}}{\sigma^2}+\frac{\left\| \truethetac - \theta_2 \right\|^2_{{V}_T}}{\gamma^2})\bigg)~.$
\end{lemma}
\begin{proof}
  Since $p_{T+1}:=\mathcal{N}(\hat{\theta}^\mathrm{r}_{T+1}, \sigma^2 V_{T}^{-1})\otimes \mathcal{N}(\hat{\theta}^\mathrm{c}_{T+1}, \gamma^2 V_{T}^{-1}) | \Theta_{}$, then 
    \begin{align*}
\frac{p_{T+1}(\theta_1, \theta_2)}{p_{T+1}(\truethetar,\truethetac)}&= \exp\bigg(-\frac{1}{2} \Big(\frac{\lVert\theta_1-\hat{\theta}^\mathrm{r}_{T+1}\rVert^2_{V_{T}}}{\sigma^2}+\frac{\lVert\theta_2-\hat{\theta}^\mathrm{c}_{T+1}\rVert^2_{V_{T}}}{\gamma^2}-\frac{\lVert\truethetar-\hat{\theta}^\mathrm{r}_{T+1}\rVert^2_{V_{T}}}{\sigma^2}-\frac{\lVert\truethetac-\hat{\theta}^\mathrm{r}_{T+1}\rVert^2_{V_{T}}}{\gamma^2}\Big) \bigg).
    \end{align*} 

Hence, 
\begin{align*}
    &\log \frac{p_{T+1}(\theta_1, \theta_2)}{p_{T+1}(\truethetar, \truethetac)} 
    + \frac{T}{2} \left( \frac{\left\| \truethetar - \theta_1 \right\|^2_{\overline{V}_T}}{\sigma^2} + \frac{\left\| \truethetac - \theta_2 \right\|^2_{\overline{V}_T}}{\gamma^2} \right)  \\*
    & \quad -\frac{1}{2} \left( \frac{\lVert \theta_1 - \hat{\theta}^\mathrm{r}_{T+1} \rVert^2_{V_T} - \lVert \truethetar - \hat{\theta}^\mathrm{r}_{T+1} \rVert^2_{V_T} - \lVert \truethetar - \theta_1 \rVert^2_{V_T}}{\sigma^2}  + \frac{\lVert \theta_2 - \hat{\theta}^\mathrm{c}_{T+1} \rVert^2_{V_T} - \lVert \truethetac - \hat{\theta}^\mathrm{r}_{T+1} \rVert^2_{V_T} - \lVert \truethetac - \theta_2 \rVert^2_{V_T}}{\gamma^2} \right).
\end{align*}
Following the proof of Lemma C.3 in \cite{li2024optimal}, when \( T \to \infty \), with probability \( 1 - 2\delta \),
\begin{align*}
    & \left| \lVert \truethetar - \hat{\theta}^{\mathrm{r}}_{T+1} \rVert^2_{V_T} 
    - \lVert \theta_1 - \hat{\theta}^{\mathrm{r}}_{T+1} \rVert^2_{V_T} + \lVert \truethetar - \theta_1 \rVert^2_{V_T} \right| \\
    &= \left| - 2 \sum_{s=1}^{T} \epsilon_s x_s^\top (\truethetar - \theta_1) \right| \\
    &\leq O\bigg(\sigma L R_1 \sqrt{T} \sqrt{2d \log \left( \frac{d + T L^2}{d \delta} \right)}\bigg).
\end{align*}

Similarly, we have:
\begin{align*}
    & \left| \lVert \truethetac - \hat{\theta}^{\mathrm{c}}_{T+1} \rVert^2_{V_T} 
    - \lVert \theta_2 - \hat{\theta}^{\mathrm{c}}_{T+1} \rVert^2_{V_T} + \lVert \truethetac - \theta_2 \rVert^2_{V_T} \right| \\
    &= \left| - 2 \sum_{s=1}^{T} \eta_s x_s^\top (\truethetac - \theta_2) \right| \\
    &\leq O\bigg(\gamma L R_2 \sqrt{T} \sqrt{2d \log \left( \frac{d + T L^2}{d \delta} \right)}\bigg).
\end{align*}

With the choice \( \delta = \frac{1}{T} \), with probability 1, the event \( E_{5,\infty} \) holds.
\end{proof}

\begin{lemma}

Define the good event 
\begin{align*}
E_{6,\delta}&:=\bigg\{
    \bigg| \max_{w \in \Delta_{\mathcal{X}}} \inf_{(\theta_1, \theta_2) \in \overline{\Theta}_{z^*}} \frac{1}{2} \left( \frac{\lVert \theta_1 - \truethetar \rVert^2_{A(w)}}{\sigma^2} + \frac{\lVert \theta_2 - \truethetac \rVert^2_{A(w)}}{\gamma^2} \right)\\*
    &\qquad- \inf_{(\theta_1, \theta_2) \in \overline{\Theta}_{z^*}} \frac{1}{2} \left( \frac{\lVert \theta_1 - \truethetar \rVert^2_{\overline{V}_T}}{\sigma^2} + \frac{\lVert \theta_2 - \truethetac \rVert^2_{\overline{V}_T}}{\gamma^2} \right)
    \bigg| \le o(1)
\bigg\}    .
\end{align*}
Event \( E_{6,\frac{1}{T}} \) holds with probability at least \( 1 - \frac{28}{T}\), conditioned on events \( E_{1,\frac{1}{T}}, E_{2,\frac{1}{T}}, E_{3,\frac{1}{T}}, E_{4,\frac{1}{T}} \).
\end{lemma}

\begin{proof}
We decompose the simple regret with the following terms:
\begin{align}
&\max_{\lambda \in \Delta_{\mathcal{X}}} \mathbb{E}_{\theta \sim \tilde{p}_{T}} 
\left[ \frac{\|\theta_1 - \theta^\mathrm{r}\|_{A(\lambda)}^2}{\sigma^2} 
+ \frac{\|\theta_2 - \theta^\mathrm{c}\|_{A(\lambda)}^2}{\gamma^2}  \right] 
- \min_{p \in \mathcal{P}(\overline{\Theta}_{z^*})} \mathbb{E}_{\theta \sim p} 
\left[ \frac{\|\theta_1 - \theta^\mathrm{r}\|_{\overline{V}_T}^2}{\sigma^2} 
+ \frac{\|\theta_2 - \theta^\mathrm{c}\|_{\overline{V}_T}^2}{\gamma^2}  \right] \nonumber \\
&= \max_{\lambda \in \Delta_{\mathcal{X}}} \mathbb{E}_{\theta \sim \tilde{p}_{T}} 
\left[ \frac{\|\theta_1 - \theta^\mathrm{r}\|_{A(\lambda)}^2}{\sigma^2} 
+ \frac{\|\theta_2 - \theta^\mathrm{c}\|_{A(\lambda)}^2}{\gamma^2}  \right] 
- \inf_{(\theta_1, \theta_2) \in \overline{\Theta}_{z^*}} 
\left(\frac{\|\theta_1 - \theta^\mathrm{r}\|_{\overline{V}_T}^2}{\sigma^2} 
+ \frac{\|\theta_2 - \theta^\mathrm{c}\|_{\overline{V}_T}^2}{\gamma^2}\right) \nonumber \\
&=\max_{\lambda \in \Delta_{\mathcal{X}}} \mathbb{E}_{\theta \sim \tilde{p}_{T}} 
\left[ \frac{\|\theta_1 - \theta^\mathrm{r}\|_{A(\lambda)}^2}{\sigma^2} 
+ \frac{\|\theta_2 - \theta^\mathrm{c}\|_{A(\lambda)}^2}{\gamma^2}  \right] \quad - \frac{1}{T} \inf_{(\theta_1, \theta_2) \in \overline{\Theta}_{z^*}} 
\left(\frac{\|\theta_1 - \theta^\mathrm{r}\|_{V_{T}}^2}{\sigma^2} 
+ \frac{\|\theta_2 - \theta^\mathrm{c}\|_{V_{T}}^2}{\gamma^2}\right) + F_6\nonumber  \\
&= \underbrace{\max_{\lambda \in \Delta_{\mathcal{X}}} \mathbb{E}_{\theta \sim \tilde{p}_{T}} 
\left[ \frac{\|\theta_1 - \theta^\mathrm{r}\|_{A(\lambda)}^2}{\sigma^2} 
+ \frac{\|\theta_2 - \theta^\mathrm{c}\|_{A(\lambda)}^2}{\gamma^2}  \right] 
- \max_{\lambda \in \Delta_{\mathcal{X}}} \frac{1}{T} 
\sum_{t=1}^{T} \mathbb{E}_{(\theta_1, \theta_2) \sim p_t} 
\left[ \frac{\|\theta_1 - \hthetart\|_{A(\lambda)}^2}{\sigma^2} 
+ \frac{\|\theta_2 - \hthetact\|_{{A}(\lambda)}^2}{\gamma^2}  \right] \nonumber }_{F_1} \\
& + \underbrace{ \max_{\lambda \in \Delta_{\mathcal{X}}} \frac{1}{T} 
\sum_{t=1}^{T} \mathbb{E}_{(\theta_1, \theta_2) \sim p_t} 
\left[ \frac{\|\theta_1 - \hthetart\|_{A(\lambda)}^2}{\sigma^2} 
+ \frac{\|\theta_2 - \hthetact\|_{A(\lambda)}^2}{\gamma^2}  \right] 
- \frac{1}{T} \sum_{t=1}^{T} \mathbb{E}_{(\theta_1, \theta_2) \sim p_t} 
\left[ \frac{\|\theta_1 - \hthetart\|_{A(\tilde{\lambda}_t)}^2}{\sigma^2} 
+ \frac{\|\theta_2 - \hthetact\|_{A(\tilde{\lambda}_t)}^2}{\gamma^2}  \right] \nonumber}_{F_2} \\
&+\underbrace{\frac{1}{T} \sum_{t=1}^{T} \mathbb{E}_{(\theta_1, \theta_2) \sim p_t} 
\left[ \frac{\|\theta_1 - \hthetart\|_{A(\tilde{\lambda}_t)}^2}{\sigma^2} 
+ \frac{\|\theta_2 - \hthetact\|_{A(\tilde{\lambda}_t)}^2}{\gamma^2}  \right]-\frac{1}{T} \sum_{t=1}^{T} \mathbb{E}_{(\theta_1, \theta_2) \sim \tilde{p}_t} 
\left[ \frac{\|\theta_1 - \hthetart\|_{A(\tilde{\lambda}_t)}^2}{\sigma^2} 
+ \frac{\|\theta_2 - \hthetact\|_{A(\tilde{\lambda}_t)}^2}{\gamma^2}  \right]}_{F_3}\nonumber \\
& + \underbrace{\frac{1}{T} \sum_{t=1}^{T} \mathbb{E}_{(\theta_1, \theta_2) \sim \tilde{p}_t} 
\left[ \frac{\|\theta_1 - \hthetart\|_{A(\tilde{\lambda}_t)}^2}{\sigma^2} 
+ \frac{\|\theta_2 - \hthetact\|_{A(\tilde{\lambda}_t)}^2}{\gamma^2}  \right] 
- \frac{1}{T} \sum_{t=1}^{T} \mathbb{E}_{(\theta_1, \theta_2) \sim \tilde{p}_t} 
\left[ \frac{\|\theta_1 - \hthetart\|_{X_t X_t^\top}^2}{\sigma^2} 
+ \frac{\|\theta_2 - \hthetact\|_{X_t X_t^\top}^2}{\gamma^2}  \right] \nonumber}_{F_4} \\
&+\underbrace{\frac{1}{T} \sum_{t=1}^{T} \mathbb{E}_{(\theta_1, \theta_2) \sim \tilde{p}_t} 
\left[ \frac{\|\theta_1 \!-\! \hthetart\|_{X_t X_t^\top}^2}{\sigma^2} 
\!+\! \frac{\|\theta_2 \!-\! \hthetact\|_{X_t X_t^\top}^2}{\gamma^2}  \right]\quad \!-\! \frac{1}{T} \inf_{(\theta_1, \theta_2) \in \overline{\Theta}_{z^*}} 
\left(\frac{\|\theta_1 \!-\! \hat{\theta}^{\mathrm{r}}_{T + 1}\|_{V_{T}}^2}{\sigma^2} 
\!+\! \frac{\|\theta_2 \!-\! \hat{\theta}^{\mathrm{c}}_{T + 1}\|_{V_{T}}^2}{\gamma^2}\right)}_{F_5} 
\!+\! F_6. \nonumber 
\end{align}
where
\begin{align}
F_6 
&:= \frac{1}{T} \inf_{(\theta_1, \theta_2) \in \overline{\Theta}_{z^*}} 
\left(\frac{\|\theta_1 - \hat{\theta}_{T+1}\|_{V_{T}}^2}{\sigma^2} 
+ \frac{\|\theta_2 - \hat{\theta}_{T+1}\|_{{V}_{T}}^2}{\gamma^2}\right) 
- \inf_{(\theta_1, \theta_2) \in \overline{\Theta}_{z^*}} 
\left(\frac{\|\theta_1 - \thetar\|_{\overline{V}_T}^2}{\sigma^2} 
+ \frac{\|\theta_2- \thetac\|_{\overline{V}_T}^2}{\gamma^2}\right).\nonumber 
\end{align}
By combining Lemmas~\ref{lem:Regret:F1},~\ref{lem:Regret_F2},~\ref{lem:Regret_F3},~\ref{lem:Regret_F4},~\ref{lem:Regret_F5}, and~\ref{lem:Regret_F6} with the union bound, this  completes the proof of the lemma.
\end{proof}

\end{proof}

\section{Lemmas for the proof of Lemma~\ref{lem:GoodEvent_6}}

\begin{lemma}
\label{lem:Regret:F1}
When the good events $E_{1,\delta}, E_{3,\delta}$ both hold,
    \begin{align*}       & \max_{\lambda \in \Delta_{\mathcal{X}}} \mathbb{E}_{\theta \sim \overline{p}_{T}} 
\left[ \frac{\|\theta_1 - \theta^\mathrm{r}\|_{A(\lambda)}^2}{\sigma^2} 
+ \frac{\|\theta_2 - \theta^\mathrm{c}\|_{{A}(\lambda)}^2}{\gamma^2}  \right]  \\*
&\qquad- \max_{\lambda \in \Delta_{\mathcal{X}}} \frac{1}{T} 
\sum_{t=1}^{T} \mathbb{E}_{(\theta_1, \theta_2) \sim p_t} 
\left[ \frac{\|\theta_1 - \hthetart\|_{A(\lambda)}^2}{\sigma^2} 
+ \frac{\|\theta_2 - \hthetact\|_{{A}(\lambda)}^2}{\gamma^2}  \right] \le o(1).
    \end{align*}
\end{lemma}

\begin{proof}
Consider,
\begin{align*}
        F_1:&=\max_{\lambda \in \Delta_{\mathcal{X}}} \mathbb{E}_{\theta \sim \overline{p}_{T}} 
\left[ \frac{\|\theta_1 - \theta^\mathrm{r}\|_{A(\lambda)}^2}{\sigma^2} 
+ \frac{\|\theta_2 - \theta^\mathrm{c}\|_{{A}(\lambda)}^2}{\gamma^2} \right] \\ 
&\qquad- \max_{\lambda \in \Delta_{\mathcal{X}}} \frac{1}{T} 
\sum_{t=1}^{T} \mathbb{E}_{(\theta_1, \theta_2) \sim p_t} 
\left[ \frac{\|\theta_1 - \hthetart\|_{A(\lambda)}^2}{\sigma^2} 
+ \frac{\|\theta_2 - \hthetact\|_{{A}(\lambda)}^2}{\gamma^2} \right]
\\&=\max_{\lambda \in \Delta_{\mathcal{X}}} \frac{1}{T}\sum_{t=1}^T\mathbb{E}_{\theta \sim {p}_{t}} 
\left[ \frac{\|\theta_1 - \theta^\mathrm{r}\|_{A(\lambda)}^2}{\sigma^2} + \frac{\|\theta_2 - \theta^\mathrm{c}\|_{{A}(\lambda)}^2}{\gamma^2} \right] \\ 
&\qquad 
- \max_{\lambda \in \Delta_{\mathcal{X}}} \frac{1}{T} 
\sum_{t=1}^{T} \mathbb{E}_{(\theta_1, \theta_2) \sim p_t} 
\left[ \frac{\|\theta_1 - \hthetart\|_{A(\lambda)}^2}{\sigma^2} 
+ \frac{\|\theta_2 - \hthetact\|_{{A}(\lambda)}^2}{\gamma^2} \right]
\\&\le \max_{\lambda \in \Delta_{\mathcal{X}}} \frac{1}{T}\sum_{t=1}^T\mathbb{E}_{\theta \sim {p}_{t}} 
\left[ \frac{\|\theta_1 - \theta^\mathrm{r}\|_{A(\lambda)}^2-\|\theta_1 - \hthetart\|_{A(\lambda)}^2}{\sigma^2} 
+ \frac{\|\theta_2 - \theta^\mathrm{c}\|_{{A}(\lambda)}^2-\|\theta_2 - \hthetact\|_{{A}(\lambda)}^2}{\gamma^2} \right] .
    \end{align*}
When the event $E_{1,\delta}$ and $E_{3,\delta}$ both hold,
\begin{align*}
\| \theta_1 - \theta^\mathrm{r} \|^2_{A(\lambda)} - \| \theta_1 - \hat{\theta}_t^\mathrm{r} \|^2_{A(\lambda)} 
&= (\theta_1 - \theta^\mathrm{r})^{\top} A(\lambda) (\theta_1 - \theta^\mathrm{r}) - (\theta_1 - \hat{\theta}_t^\mathrm{r})^{\top} A(\lambda) (\theta_1 - \hat{\theta}_t^\mathrm{r}) \\
&= (\thetar + \hat{\theta}_t^\mathrm{r} - 2\theta_1)^{\top} A(\lambda) (\thetar - \hat{\theta}_t^\mathrm{r}) \\
&= \sum_{x \in \mathcal{X}} \lambda_x (\thetar + \hat{\theta}_t^\mathrm{r} - 2\theta_1)^{\top} x x^{\top} (\thetar - \hat{\theta}_t^\mathrm{r}) \\
&\le \max_{x \in \mathcal{X}} (\thetar + \hat{\theta}_t^\mathrm{r} - 2\theta_1)^{\top} x x^{\top} (\thetar - \hat{\theta}_t^\mathrm{r})\\
&\le (3L R_1+B_1)\max_{x\in\mathcal{X}} x^\top(\thetar-\hat{\theta}_t^\mathrm{r}).
\end{align*}

From Eqn.~\eqref{eq:concentration:reward}, we have $\max_{x\in \mathcal{X}}\lvert\langle x, \hthetart - \thetar \rangle\rvert \leq \sqrt{\frac{d}{t^{3/4}} \beta_1(t, \frac{1}{\delta^2})}$ for $t>T_1(\delta)+1$; from event $E_{1, \delta}$, we have $\max_{x\in \mathcal{X}}\lvert\langle x, \hthetart - \thetar \rangle\rvert \leq L \sqrt{\beta_1(t, \frac{1}{\delta^2})}$ for $t\ge 1$.
Hence
\begin{align*}
    \max_{x\in\mathcal{X}} \frac{1}{T} \sum_{t=1}^T x^\top(\thetar-\hat{\theta}_t^\mathrm{r})&=\max_{x\in\mathcal{X}} \frac{1}{T} \sum_{t=1}^{T_1(\delta)+1} x^\top(\thetar-\hat{\theta}_t^\mathrm{r})+\max_{x\in\mathcal{X}} \frac{1}{T} \sum_{t=T_1(\delta)+2}^T x^\top(\thetar-\hat{\theta}_t^\mathrm{r})
    \\&\le \frac{(T_1(\delta)+1)L\sqrt{\beta_1(t,\frac{1}{\delta^2})}}{T}+\frac{1}{T} \sum_{t=T_1(\delta)+2}^T \sqrt{\frac{d}{t^{3/4}} \beta_1(t, \frac{1}{\delta^2})}
    \\&\le \frac{(T_1(\delta)+1)L\sqrt{\beta_1(t,\frac{1}{\delta^2})}}{T}+\frac{8}{5} \frac{(T+1)^\frac{5}{8}}{T}.
\end{align*}
Hence when we choose $\delta=\frac{1}{T}$, 
\begin{align*}
    \lim_{T\to \infty}\max_{\lambda \in \Delta_{\mathcal{X}}} \frac{1}{T}\sum_{t=1}^T\mathbb{E}_{\theta \sim {p}_{t}} 
\frac{\|\theta_1 - \theta^\mathrm{r}\|_{A(\lambda)}^2-\|\theta_1 - \hthetart\|_{A(\lambda)}^2}{\sigma^2} \le o(1).
\end{align*}

Similarly  we can also prove
\begin{align*}
    \lim_{T\to \infty}\max_{\lambda \in \Delta_{\mathcal{X}}} \frac{1}{T}\sum_{t=1}^T\mathbb{E}_{\theta \sim {p}_{t}} 
 \frac{\|\theta_2 - \theta^\mathrm{c}\|_{{A}(\lambda)}^2-\|\theta_2 - \hthetact\|_{{A}(\lambda)}^2}{\gamma^2}\le\leo(1).
\end{align*}
Hence
\begin{align}
    F_1\le o(1)~. \nonumber
\end{align}
\end{proof}

\begin{lemma}
\label{lem:Regret_F2}
With probability $1-\frac{12}{T}$,
    \begin{align*}
        F_2\le o(1).
    \end{align*}
\end{lemma}
\begin{proof}
Consider,
    \begin{align*}
    F_2:&=\max_{\lambda \in \Delta_{\mathcal{X}}} \frac{1}{T} 
\sum_{t=1}^{T} \mathbb{E}_{(\theta_1, \theta_2) \sim p_t} 
\left[ \frac{\|\theta_1 - \hthetart\|_{A(\lambda)}^2}{\sigma^2} 
+ \frac{\|\theta_2 - \hthetact\|_{A(\lambda)}^2}{\gamma^2} \right] \\ 
&\qquad 
- \frac{1}{T} \sum_{t=1}^{T} \mathbb{E}_{(\theta_1, \theta_2) \sim p_t} 
\left[ \frac{\|\theta_1 - \hthetart\|_{A(\tilde{\lambda}_t)}^2}{\sigma^2} 
+ \frac{\|\theta_2 - \hthetact\|_{A(\tilde{\lambda}_t)}^2}{\gamma^2} \right] 
\\&=\max_{\lambda \in \Delta_{\mathcal{X}}} \frac{1}{T} 
\sum_{t=1}^{T} \mathbb{E}_{(\theta_1, \theta_2) \sim p_t, x_t\sim \lambda} 
\left[ \frac{\|\theta_1 - \hthetart\|_{x_t x_t^\top}^2}{\sigma^2} 
+ \frac{\|\theta_2 - \hthetact\|_{x_t x_t^\top}^2}{\gamma^2} \right] \\* 
&\qquad 
- \frac{1}{T} \sum_{t=1}^{T} \mathbb{E}_{(\theta_1, \theta_2) \sim p_t, x_t\sim \tilde{\lambda}_t} 
\left[ \frac{\|\theta_1 - \hthetart\|_{x_t x_t^\top}^2}{\sigma^2} 
+ \frac{\|\theta_2 - \hthetact\|_{x_t x_t^\top}^2}{\gamma^2} \right]. \nonumber
\end{align*}

For any $\lambda$, 
\begin{align*}
     &\frac{1}{T} 
\sum_{t=1}^{T} \mathbb{E}_{(\theta_1, \theta_2) \sim p_t, x_t\sim \lambda} 
\left[\frac{\|\theta_1 - \hthetart\|_{x_t x_t^\top}^2}{\sigma^2} 
+ \frac{\|\theta_2 - \hthetact\|_{x_t x_t^\top}^2}{\gamma^2} \right] 
- \frac{1}{T} \sum_{t=1}^{T} \mathbb{E}_{(\theta_1, \theta_2) \sim p_t, x_t\sim \tilde{\lambda}_t} 
\left[\frac{\|\theta_1 - \hthetart\|_{x_t x_t^\top}^2}{\sigma^2} 
+ \frac{\|\theta_2 - \hthetact\|_{x_t x_t^\top}^2}{\gamma^2} \right]
\\&=\frac{1}{T} 
\sum_{t=1}^{T} \mathbb{E}_{(\theta_1, \theta_2) \sim p_t, x_t\sim \lambda} 
\left[\frac{\|\theta_1 \!-\! \hthetart\|_{x_t x_t^\top}^2}{\sigma^2} 
\!+\! \frac{\|\theta_2\! -\! \hthetact\|_{x_t x_t^\top}^2}{\gamma^2} \right] 
\!-\! \frac{1}{T} \sum_{t=1}^{T} \mathbb{E}_{(\theta_1, \theta_2) \sim p_t, x_t\sim \tilde{\lambda}_t} 
\left[\frac{\|\theta_1 \!- \!\hthetart\|_{x_t x_t^\top}^2}{\sigma^2} 
\!+\! \frac{\|\theta_2 \!-\! \hthetact\|_{x_t x_t^\top}^2}{\gamma^2} \right]
\\&-\frac{1}{T} 
\sum_{t=1}^{T} \mathbb{E}_{ x_t\sim \lambda} 
\left[\frac{\|\theta_t^\mathrm{r} - \hthetart\|_{x_t x_t^\top}^2}{\sigma^2} 
+ \frac{\|\theta_t^\mathrm{c} - \hthetact\|_{x_t x_t^\top}^2}{\gamma^2} \right] 
+ \frac{1}{T} \sum_{t=1}^{T} \mathbb{E}_{ x_t\sim \tilde{\lambda}_t} 
\left[\frac{\|\theta_t^\mathrm{r} - \hthetart\|_{x_t x_t^\top}^2}{\sigma^2} 
+ \frac{\|\theta_t^\mathrm{c} - \hthetact\|_{x_t x_t^\top}^2}{\gamma^2} \right]
\\&+\frac{1}{T} 
\sum_{t=1}^{T} \mathbb{E}_{ x_t\sim \lambda} 
\left[\frac{\|\theta_t^\mathrm{r} - \hthetart\|_{x_t x_t^\top}^2}{\sigma^2} 
+ \frac{\|\theta_t^\mathrm{c} - \hthetact\|_{x_t x_t^\top}^2}{\gamma^2} \right] 
- \frac{1}{T} \sum_{t=1}^{T} \mathbb{E}_{ x_t\sim \tilde{\lambda}_t} 
\left[\frac{\|\theta_t^\mathrm{r} - \hthetart\|_{x_t x_t^\top}^2}{\sigma^2} 
+ \frac{\|\theta_t^\mathrm{c} - \hthetact\|_{x_t x_t^\top}^2}{\gamma^2} \right]
\\&=\frac{1}{T} 
\sum_{t=1}^{T} \mathbb{E}_{(\theta_1, \theta_2) \sim p_t, x_t\sim \lambda} 
\left[\frac{\|\theta_1 \!-\! \hthetart\|_{x_t x_t^\top}^2}{\sigma^2} 
\!+\! \frac{\|\theta_2 \!-\! \hthetact\|_{x_t x_t^\top}^2}{\gamma^2} \right] 
\!-\! \frac{1}{T} \sum_{t=1}^{T} \mathbb{E}_{(\theta_1, \theta_2) \sim p_t, x_t\sim \tilde{\lambda}_t} 
\left[\frac{\|\theta_1 \!-\! \hthetart\|_{x_t x_t^\top}^2}{\sigma^2} 
\!+\! \frac{\|\theta_2 \!-\! \hthetact\|_{x_t x_t^\top}^2}{\gamma^2} \right]
\\&-\frac{1}{T} 
\sum_{t=1}^{T} \mathbb{E}_{ x_t\sim \lambda} 
\left[\frac{\|\theta_t^\mathrm{r} - \hthetart\|_{x_t x_t^\top}^2}{\sigma^2} 
+ \frac{\|\theta_t^\mathrm{c} - \hthetact\|_{x_t x_t^\top}^2}{\gamma^2} \right] 
+ \frac{1}{T} \sum_{t=1}^{T} \mathbb{E}_{ x_t\sim \tilde{\lambda}_t} 
\left[\frac{\|\theta_t^\mathrm{r} - \hthetart\|_{x_t x_t^\top}^2}{\sigma^2} 
+ \frac{\|\theta_t^\mathrm{c} - \hthetact\|_{x_t x_t^\top}^2}{\gamma^2} \right]
\\&+\frac{1}{T} 
\sum_{t=1}^{T} \mathbb{E}_{ x_t\sim \lambda} 
\left[\frac{\|\theta_t^\mathrm{r} - \hthetart\|_{x_t x_t^\top}^2}{\sigma^2} 
+ \frac{\|\theta_t^\mathrm{c} - \hthetact\|_{x_t x_t^\top}^2}{\gamma^2}\right] 
- \frac{1}{T} 
\sum_{t=1}^{T} \mathbb{E}_{ x_t\sim \lambda_t} 
\left[\frac{\|\theta_t^\mathrm{r} - \hthetart\|_{x_t x_t^\top}^2}{\sigma^2} 
+ \frac{\|\theta_t^\mathrm{c} - \hthetact\|_{x_t x_t^\top}^2}{\gamma^2} \right]
\\&+\frac{1}{T} 
\sum_{t=1}^{T} \mathbb{E}_{ x_t\sim \lambda_t} 
\left[\frac{\|\theta_t^\mathrm{r} - \hthetart\|_{x_t x_t^\top}^2}{\sigma^2} 
+ \frac{\|\theta_t^\mathrm{c} - \hthetact\|_{x_t x_t^\top}^2}{\gamma^2} \right]
-
\frac{1}{T} \sum_{t=1}^{T} \mathbb{E}_{ x_t\sim \tilde{\lambda}_t} 
\left[\frac{\|\theta_t^\mathrm{r} - \hthetart\|_{x_t x_t^\top}^2}{\sigma^2} 
+ \frac{\|\theta_t^\mathrm{c} - \hthetact\|_{x_t x_t^\top}^2}{\gamma^2} \right]
\\&=F_{2,1}+F_{2,2}+F_{2,3}
\end{align*}
where
\begin{align*}
    F_{2,1}&\!:=\!\frac{1}{T} 
\sum_{t=1}^{T} \mathbb{E}_{(\theta_1, \theta_2) \sim p_t, x_t\sim \lambda} 
\left[\frac{\|\theta_1 \!-\! \hthetart\|_{x_t x_t^\top}^2}{\sigma^2} 
\!+\! \frac{\|\theta_2 \!-\! \hthetact\|_{x_t x_t^\top}^2}{\gamma^2} \right] 
\!-\! \frac{1}{T} \sum_{t=1}^{T} \mathbb{E}_{(\theta_1, \theta_2) \sim p_t, x_t\sim \tilde{\lambda}_t} 
\left[\frac{\|\theta_1 \!-\! \hthetart\|_{x_t x_t^\top}^2}{\sigma^2} 
\!+\! \frac{\|\theta_2 \!-\! \hthetact\|_{x_t x_t^\top}^2}{\gamma^2} \right]
\\&-\frac{1}{T} 
\sum_{t=1}^{T} \mathbb{E}_{ x_t\sim \lambda} 
\left[\frac{\|\theta_t^\mathrm{r} - \hthetart\|_{x_t x_t^\top}^2}{\sigma^2} 
+ \frac{\|\theta_t^\mathrm{c} - \hthetact\|_{x_t x_t^\top}^2}{\gamma^2} \right] 
+ \frac{1}{T} \sum_{t=1}^{T} \mathbb{E}_{ x_t\sim \tilde{\lambda}_t} 
\left[\frac{\|\theta_t^\mathrm{r} - \hthetart\|_{x_t x_t^\top}^2}{\sigma^2} 
+ \frac{\|\theta_t^\mathrm{c} - \hthetact\|_{x_t x_t^\top}^2}{\gamma^2} \right]
\end{align*}
and
\begin{align*}
    F_{2,2}:=\frac{1}{T} 
\sum_{t=1}^{T} \mathbb{E}_{ x_t\sim \lambda} 
\left[\frac{\|\theta_t^\mathrm{r} - \hthetart\|_{x_t x_t^\top}^2}{\sigma^2} 
+ \frac{\|\theta_t^\mathrm{c} - \hthetact\|_{x_t x_t^\top}^2}{\gamma^2}\right] 
- \frac{1}{T} 
\sum_{t=1}^{T} \mathbb{E}_{ x_t\sim \lambda_t} 
\left[\frac{\|\theta_t^\mathrm{r} - \hthetart\|_{x_t x_t^\top}^2}{\sigma^2} 
+ \frac{\|\theta_t^\mathrm{c} - \hthetact\|_{x_t x_t^\top}^2}{\gamma^2}\right]
\end{align*}
and
\begin{align*}
    F_{2,3}:=\frac{1}{T} 
\sum_{t=1}^{T} \mathbb{E}_{ x_t\sim \lambda_t} 
\left[\frac{\|\theta_t^\mathrm{r} - \hthetart\|_{x_t x_t^\top}^2}{\sigma^2} 
+ \frac{\|\theta_t^\mathrm{c} - \hthetact\|_{x_t x_t^\top}^2}{\gamma^2} \right]
-
\frac{1}{T} \sum_{t=1}^{T} \mathbb{E}_{ x_t\sim \tilde{\lambda}_t} 
\left[\frac{\|\theta_t^\mathrm{r} - \hthetart\|_{x_t x_t^\top}^2}{\sigma^2} 
+ \frac{\|\theta_t^\mathrm{c} - \hthetact\|_{x_t x_t^\top}^2}{\gamma^2}\right]
\end{align*}

Then with Lemma~\ref{lem:Regret_F21},~\ref{lem:Regret_F22}, and~\ref{lem:Regret_F23}, we have for any $\lambda$, with probability $1-\frac{6}{T}$,
\begin{align*}
    &\frac{1}{T} 
\sum_{t=1}^{T} \mathbb{E}_{(\theta_1, \theta_2) \sim p_t, x_t\sim \lambda} 
\left[ \frac{\|\theta_1 - \hthetart\|_{x_t x_t^\top}^2}{\sigma^2} 
+ \frac{\|\theta_2 - \hthetact\|_{x_t x_t^\top}^2}{\gamma^2}  \right]  \\*
&\quad
- \frac{1}{T} \sum_{t=1}^{T} \mathbb{E}_{(\theta_1, \theta_2) \sim p_t, x_t\sim \tilde{\lambda}_t} 
\left[\frac{\|\theta_1 - \hthetart\|_{x_t x_t^\top}^2}{\sigma^2} 
+ \frac{\|\theta_2 - \hthetact\|_{x_t x_t^\top}^2}{\gamma^2}\right]
\\&=F_{2,1}+F_{2,2}+F_{2,3}
\\&\le o(1)~.
\end{align*}

Define $\overline{\lambda}:=\argmax_{\lambda \in \Delta_{\mathcal{X}}} \frac{1}{T} 
\sum_{t=1}^{T} \mathbb{E}_{(\theta_1, \theta_2) \sim p_t, x_t\sim \lambda} 
\left[ \frac{\|\theta_1 - \hthetart\|_{x_t x_t^\top}^2}{\sigma^2} 
+ \frac{\|\theta_2 - \hthetact\|_{x_t x_t^\top}^2}{\gamma^2} \right] $. Next, denote the $\varepsilon$-covering of $\Delta_{\mathcal{X}}$ as $\mathcal{B}_{\varepsilon}:=\cbr{\lambda_1 : \lVert \lambda_1-\overline{\lambda} \rVert_1\le \varepsilon}~.$
Then there exists $\lambda_2\in\mathcal{B}_\varepsilon$ such that 
\begin{align*}
    &\max_{\lambda \in \Delta_{\mathcal{X}}}  
\sum_{t=1}^{T} \mathbb{E}_{(\theta_1, \theta_2) \sim p_t, x_t\sim \lambda} 
\left[ \frac{\|\theta_1 - \hthetart\|_{x_t x_t^\top}^2}{\sigma^2} 
+ \frac{\|\theta_2 - \hthetact\|_{x_t x_t^\top}^2}{\gamma^2} \right] \\
&\qquad- \max_{\lambda \in \mathcal{B}_\varepsilon}  
\sum_{t=1}^{T} \mathbb{E}_{(\theta_1, \theta_2) \sim p_t, x_t\sim \lambda} 
\left[ \frac{\|\theta_1 - \hthetart\|_{x_t x_t^\top}^2}{\sigma^2} 
+ \frac{\|\theta_2 - \hthetact\|_{x_t x_t^\top}^2}{\gamma^2} \right]
\\
&\le\sum_{t=1}^{T} \mathbb{E}_{(\theta_1, \theta_2) \sim p_t, x_t\sim \overline{\lambda}} 
\left[ \frac{\|\theta_1 - \hthetart\|_{x_t x_t^\top}^2}{\sigma^2} 
+ \frac{\|\theta_2 - \hthetact\|_{x_t x_t^\top}^2}{\gamma^2}  \right]-\sum_{t=1}^{T} \mathbb{E}_{(\theta_1, \theta_2) \sim p_t, x_t\sim \lambda_2} 
\left[ \frac{\|\theta_1 - \hthetart\|_{x_t x_t^\top}^2}{\sigma^2} 
+ \frac{\|\theta_2 - \hthetact\|_{x_t x_t^\top}^2}{\gamma^2}\right]
\\&\stackrel{(a)}{\le} \sum_{t=1}^T \mathbb{E}_{(\theta_1, \theta_2) \sim p_t} \left(\frac{(LR_1+B_1)^2}{\sigma^2}+\frac{(LR_2+B_2)^2}{\gamma^2}\right)\lVert\overline{\lambda}-\lambda_2 \rVert_1 
\\&\le T\varepsilon \left(\frac{(LR_1+B_1)^2}{\sigma^2}+\frac{(LR_2+B_2)^2}{\gamma^2}\right)
\end{align*}
The term $(a)$ follows from the fact that 
\[
\mathbb{E}_{(\theta_1, \theta_2) \sim p_t, x_t \sim \lambda} \left[ \frac{\|\theta_1 - \hat{\theta}_t^{\mathrm{r}}\|_{x_t x_t^\top}^2}{\sigma^2} + \frac{\|\theta_2 - \hat{\theta}_t^{\mathrm{c}}\|_{x_t x_t^\top}^2}{\gamma^2} \right]
\]
is \( \Big(\frac{(LR_1 + B_1)^2}{\sigma^2} + \frac{(LR_2 + B_2)^2}{\gamma^2} \Big)\)-Lipschitz  for any \( t \) under the event \( E_{2, \delta} \).

Additionally, since \( |\mathcal{B}_\varepsilon| \le \left( \frac{3}{\varepsilon} \right)^{|\mathcal{X}|} \), with probability \( 1 - \frac{6}{T} \), and based on the proof of Lemmas~\ref{lem:Regret_F21}, \ref{lem:Regret_F22}, and \ref{lem:Regret_F23} with \( \delta = \frac{1}{T|\mathcal{B}_\varepsilon|} \) and \( \varepsilon = \frac{1}{\sqrt{T}} \), we can also conclude that
\begin{align*}
&\max_{\lambda \in \mathcal{B}_\varepsilon} \frac{1}{T} \sum_{t=1}^{T} \mathbb{E}_{(\theta_1, \theta_2) \sim p_t, x_t \sim \lambda} \left[ \frac{\|\theta_1 - \hat{\theta}_t^{\mathrm{r}}\|_{x_t x_t^\top}^2}{\sigma^2} + \frac{\|\theta_2 - \hat{\theta}_t^{\mathrm{c}}\|_{x_t x_t^\top}^2}{\gamma^2} \right] \\
&\qquad- \frac{1}{T} \sum_{t=1}^{T} \mathbb{E}_{(\theta_1, \theta_2) \sim p_t, x_t \sim \tilde{\lambda}_t} \left[ \frac{\|\theta_1 - \hat{\theta}_t^{\mathrm{r}}\|_{x_t x_t^\top}^2}{\sigma^2} + \frac{\|\theta_2 - \hat{\theta}_t^{\mathrm{c}}\|_{x_t x_t^\top}^2}{\gamma^2} \right]
\le o(1)
\end{align*}

Hence with probability $1-\frac{12}{T}$,
\begin{align*}
    F_2\le \frac{1}{\sqrt{T}}\left(\frac{(LR_1+B_1)^2}{\sigma^2}+\frac{(LR_2+B_2)^2}{\gamma^2}\right)+o(1)
    =o(1)~.
\end{align*}

\end{proof}

\begin{lemma}
\label{lem:Regret_F21}
With probability $1-\frac{2}{T}$,
    \begin{align*}
        F_{2,1}\le o(1)~.
    \end{align*}
\end{lemma}
\begin{proof}
    Denote $\mathcal{F}_t$ as the history up to $t$ as in the algorithm. Since 
\begin{align*}
    F_{2,1}:&=\frac{1}{T} 
\sum_{t=1}^{T} \Bigg( \mathbb{E}_{(\theta_1, \theta_2) \sim p_t, x_t\sim \lambda} 
\left[ \frac{\|\theta_1 - \hthetart\|_{x_t x_t^\top}^2}{\sigma^2} 
+ \frac{\|\theta_2 - \hthetact\|_{x_t x_t^\top}^2}{\gamma^2}  \right] \\
&\qquad- \mathbb{E}_{(\theta_1, \theta_2) \sim p_t, x_t\sim \tilde{\lambda}_t} 
\left[ \frac{\|\theta_1 - \hthetart\|_{x_t x_t^\top}^2}{\sigma^2} 
+ \frac{\|\theta_2 - \hthetact\|_{x_t x_t^\top}^2}{\gamma^2}  \right] \\
 &\qquad-\mathbb{E}_{ x_t\sim \lambda} 
\left[ \frac{\|\theta_t^\mathrm{r} - \hthetart\|_{x_t x_t^\top}^2}{\sigma^2} 
 + \frac{\|\theta_t^\mathrm{c} - \hthetact\|_{x_t x_t^\top}^2}{\gamma^2}  \right]  \\
&\qquad+ \mathbb{E}_{ x_t\sim \tilde{\lambda}_t} 
\left[\frac{\|\theta_t^\mathrm{r} - \hthetart\|_{x_t x_t^\top}^2}{\sigma^2} 
+ \frac{\|\theta_t^\mathrm{c} - \hthetact\|_{x_t x_t^\top}^2}{\gamma^2} \right] \Bigg)  =:\frac{1}{T} 
\sum_{t=1}^{T}  Y_t
\end{align*}
One can see that $\EE[Y_t \mid \mathcal{F}_{t-1}]=0$ is a martingale difference since $p_t$ is determined under $\mathcal{F}_{t-1}$. 

Furthermore, under the good event $E_{2,\delta}$, $$\|\theta_t^\mathrm{r} - \hthetart\|_{x_t x_t^\top}^2\le (B_1+LR_1)^2, \qquad\|\theta_t^\mathrm{c} - \hthetact\|_{x_t x_t^\top}^2\le (B_2+LR_2)^2~,$$ then 
\begin{align*}
    \lvert Y_t\rvert\le 4 \left(\frac{(B_1+LR_1)^2}{\sigma^2}+\frac{(B_2+LR_2)^2}{\gamma^2} \right)
\end{align*}
Then by Azuma--Hoeffding's inequality, with probabaility at least $1-\frac{2}{T}$, 
\begin{align*}
    F_{2,1} &\le \sqrt{\frac{8(\frac{(B_1+LR_1)^2}{\sigma^2}+\frac{(B_2+LR_2)^2}{\gamma^2}) \log(T)}{T}}
     \le o(1)~.
\end{align*}

\end{proof}

\begin{lemma}
\label{lem:Regret_F22}
With probability $1-\frac{2}{T}$,
    \begin{align*}
        F_{2,2}\le o(1)~.
    \end{align*}
\end{lemma}
\begin{proof}
    Apply the upper bound of the AdaHedge as in \cite{de2014follow}, we have
    \begin{align*}
        &\max_{\lambda\in\Delta_{\mathcal{X}}} 
\sum_{t=1}^{T} \mathbb{E}_{ x_t\sim \lambda} 
\left[ \frac{\|\theta_t^\mathrm{r} - \hthetart\|_{x_t x_t^\top}^2}{\sigma^2} 
+ \frac{\|\theta_t^\mathrm{c} - \hthetact\|_{x_t x_t^\top}^2}{\gamma^2}  \right] 
- \sum_{t=1}^{T} \mathbb{E}_{ x_t\sim \lambda_t} 
\left[ \frac{\|\theta_t^\mathrm{r} - \hthetart\|_{x_t x_t^\top}^2}{\sigma^2} 
+ \frac{\|\theta_t^\mathrm{c} - \hthetact\|_{x_t x_t^\top}^2}{\gamma^2}\right]
\\&\le 2 r_T\sqrt{T\log |\mathcal{X}|}+r_T\Big(\frac{16}{3}\log |\mathcal{X}|+2\Big)
    \end{align*}
    where $r_T:=\max_{t\in [T]} \max_{x\in\mathcal{X}}\frac{\|\theta_t^\mathrm{r} - \hthetart\|_{x_t x_t^\top}^2}{\sigma^2} 
+ \frac{\|\theta_t^\mathrm{c} - \hthetact\|_{x_t x_t^\top}^2}{\gamma^2}~.$

Then with the choice $\delta=\frac{1}{T}$, with probability $1-\frac{2}{T}$, good events $E_{1,\delta}, E_{2,\delta}$ hold for any $t\in [T]$, hence by Lemma~\ref{lem:Good_Event_1} and Lemma~\ref{lem:Good_Event_2}, 
\begin{align*}
    r_T\le \frac{(LR_1+B_1)^2}{\sigma^2}+\frac{(LR_2+B_2)^2}{\gamma^2}=O(\log T)
\end{align*}

Hence
\begin{align*}
    F_{2,2}&\le \frac{1}{T}\max_{\lambda\in\Delta_{\mathcal{X}}} 
\sum_{t=1}^{T} \mathbb{E}_{ x_t\sim \lambda} 
\left[ \frac{\|\theta_t^\mathrm{r} - \hthetart\|_{x_t x_t^\top}^2}{\sigma^2} 
+ \frac{\|\theta_t^\mathrm{c} - \hthetact\|_{x_t x_t^\top}^2}{\gamma^2}  \right] 
- \frac{1}{T}\sum_{t=1}^{T} \mathbb{E}_{ x_t\sim \lambda_t} 
\left[ \frac{\|\theta_t^\mathrm{r} - \hthetart\|_{x_t x_t^\top}^2}{\sigma^2} 
+ \frac{\|\theta_t^\mathrm{c} - \hthetact\|_{x_t x_t^\top}^2}{\gamma^2}\right]
\\&\le O\left(  \frac{2 \log T\sqrt{T\log |\mathcal{X}|}+\log T(\frac{16}{3}\log |\mathcal{X}|+2)}{T}\right)
 \le o(1)~.
\end{align*}

\end{proof}

\begin{lemma}
\label{lem:Regret_F23}
With probability $1-\frac{2}{T}$,
    \begin{align*}
        F_{2,3}\le o(1)~.
    \end{align*}
\end{lemma}

\begin{proof}
Consider,
    \begin{align*}
        F_{2,3}:=\frac{1}{T} 
\sum_{t=1}^{T} \mathbb{E}_{ x_t\sim \lambda_t} 
\left[ \frac{\|\theta_t^\mathrm{r} - \hthetart\|_{x_t x_t^\top}^2}{\sigma^2} 
+ \frac{\|\theta_t^\mathrm{c} - \hthetact\|_{x_t x_t^\top}^2}{\gamma^2} \right]
-
\frac{1}{T} \sum_{t=1}^{T} \mathbb{E}_{ x_t\sim \tilde{\lambda}_t} 
\left[ \frac{\|\theta_t^\mathrm{r} - \hthetart\|_{x_t x_t^\top}^2}{\sigma^2} 
+ \frac{\|\theta_t^\mathrm{c} - \hthetact\|_{x_t x_t^\top}^2}{\gamma^2} \right]
    \end{align*}
By the definition of $\tilde{\lambda}_t:=(1-\gamma_t)\lambda_t+\gamma_t\lambda^{G}$, we have
\begin{align*}
    F_{2,3}&=\frac{1}{T} \sum_{t=1}^T \gamma_t\left( \mathbb{E}_{ x_t\sim \lambda_t} 
\left[ \frac{\|\theta_t^\mathrm{r} - \hthetart\|_{x_t x_t^\top}^2}{\sigma^2} 
+ \frac{\|\theta_t^\mathrm{c} - \hthetact\|_{x_t x_t^\top}^2}{\gamma^2}  \right]-\mathbb{E}_{ x_t\sim \lambda^G} 
\left[ \frac{\|\theta_t^\mathrm{r} - \hthetart\|_{x_t x_t^\top}^2}{\sigma^2} 
+ \frac{\|\theta_t^\mathrm{c} - \hthetact\|_{x_t x_t^\top}^2}{\gamma^2} \right]\right).
\end{align*}

Furthermore, under good event $E_{2,\delta}$, 
$$\|\theta_t^\mathrm{r} - \hthetart\|_{x_t x_t^\top}^2\le (B_1+LR_1)^2, \|\theta_t^\mathrm{c} - \hthetact\|_{x_t x_t^\top}^2\le (B_2+LR_2)^2~,$$
Hence with $\gamma_t= t^{-\frac{1}{4}}$, $\delta:=\frac{1}{T}$,
\begin{align*}
    F_{2,3}&\le \frac{2}{T} \left(\frac{(LR_1+B_1)^2}{\sigma^2}+\frac{(LR_2+B_2)^2}{\gamma^2}\right)\sum_{t=1}^T t^{-\frac{1}{4}}
    \\&\le \frac{8}{3T} \left(\frac{(LR_1+B_1)^2}{\sigma^2}+\frac{(LR_2+B_2)^2}{\gamma^2}\right) T^{\frac{3}{4}}
    \\&\le o(1)~.
\end{align*}
\end{proof}

\begin{lemma}
With probability $1 - \frac{3}{T}$,
\label{lem:Regret_F3}
\begin{align*}
    F_3 \le o(1).
\end{align*}
\end{lemma}

\begin{proof}
    When the good events \( E_{2, \delta} \) and \( E_{4, \delta} \) both hold, by the definition of \( \tilde{p}_t \), for \( t > \max \{T_1(\delta) + 1, T_2(\delta)\} \), we have \( \hat{z}_t = z^* \), and thus \( \tilde{p}_t = p_t \).
    
    Hence, with probability \( 1 - 3\delta \) and \( \delta = \frac{1}{T} \), we can write:
    \begin{align*}
        F_3 &:= \frac{1}{T} \sum_{t=1}^{T} \mathbb{E}_{(\theta_1, \theta_2) \sim p_t} 
        \left[  \frac{\|\theta_1 - \hat{\theta}_t^{\mathrm{r}}\|_{A(\tilde{\lambda}_t)}^2}{\sigma^2} + \frac{\|\theta_2 - \hat{\theta}_t^{\mathrm{c}}\|_{A(\tilde{\lambda}_t)}^2}{\gamma^2}  \right] \\
        &\quad - \frac{1}{T} \sum_{t=1}^{T} \mathbb{E}_{(\theta_1, \theta_2) \sim \tilde{p}_t} 
        \left[  \frac{\|\theta_1 - \hat{\theta}_t^{\mathrm{r}}\|_{A(\tilde{\lambda}_t)}^2}{\sigma^2} + \frac{\|\theta_2 - \hat{\theta}_t^{\mathrm{c}}\|_{A(\tilde{\lambda}_t)}^2}{\gamma^2}  \right] \\
        &= \frac{1}{T} \sum_{t=1}^{\max \{T_1(\delta) + 1, T_2(\delta)\}} \mathbb{E}_{(\theta_1, \theta_2) \sim p_t} 
        \left[  \frac{\|\theta_1 - \hat{\theta}_t^{\mathrm{r}}\|_{A(\tilde{\lambda}_t)}^2}{\sigma^2} + \frac{\|\theta_2 - \hat{\theta}_t^{\mathrm{c}}\|_{A(\tilde{\lambda}_t)}^2}{\gamma^2} \right] \\
        &\quad - \frac{1}{T} \sum_{t=1}^{\max \{T_1(\delta) + 1, T_2(\delta)\}} \mathbb{E}_{(\theta_1, \theta_2) \sim \tilde{p}_t} 
        \left[  \frac{\|\theta_1 - \hat{\theta}_t^{\mathrm{r}}\|_{A(\tilde{\lambda}_t)}^2}{\sigma^2} + \frac{\|\theta_2 - \hat{\theta}_t^{\mathrm{c}}\|_{A(\tilde{\lambda}_t)}^2}{\gamma^2}  \right] \\
        &\le \frac{2 \max \{T_1(\delta) + 1, T_2(\delta)\} \left( \frac{(LR_1 + B_1)^2}{\sigma^2} + \frac{(LR_2 + B_2)^2}{\gamma^2} \right)}{T} \\
        &= o(1).
    \end{align*}
\end{proof}

\begin{lemma}
With probability \( 1 - \frac{2}{T} \),
\label{lem:Regret_F4}
\begin{align*}
    F_4 \le o(1).
\end{align*}
\end{lemma}

\begin{proof}
    We define \( F_4 \) as follows:
    \begin{align*}
        F_4 &:= \frac{1}{T} \sum_{t=1}^{T} \mathbb{E}_{(\theta_1, \theta_2) \sim \tilde{p}_t} 
        \left[ \frac{\|\theta_1 - \hat{\theta}_t^{\mathrm{r}}\|_{A(\tilde{\lambda}_t)}^2}{\sigma^2} + \frac{\|\theta_2 - \hat{\theta}_t^{\mathrm{c}}\|_{A(\tilde{\lambda}_t)}^2}{\gamma^2}  \right] \\*
        &\quad - \frac{1}{T} \sum_{t=1}^{T} \mathbb{E}_{(\theta_1, \theta_2) \sim \tilde{p}_t} 
        \left[ \frac{\|\theta_1 - \hat{\theta}_t^{\mathrm{r}}\|_{x_t x_t^\top}^2}{\sigma^2} + \frac{\|\theta_2 - \hat{\theta}_t^{\mathrm{c}}\|_{x_t x_t^\top}^2}{\gamma^2}  \right].
    \end{align*}

    Since $$ \mathbb{E}_{x_t} \left[ \mathbb{E}_{(\theta_1, \theta_2) \sim \tilde{p}_t} \left[  \frac{\|\theta_1 - \hat{\theta}_t^{\mathrm{r}}\|_{x_t x_t^\top}^2}{\sigma^2} + \frac{\|\theta_2 - \hat{\theta}_t^{\mathrm{c}}\|_{x_t x_t^\top}^2}{\gamma^2}  \right] \,\bigg|\,  \mathcal{F}_{t-1} \right] 
    = \mathbb{E}_{(\theta_1, \theta_2) \sim \tilde{p}_t} \left[ \frac{\|\theta_1 - \hat{\theta}_t^{\mathrm{r}}\|_{A(\tilde{\lambda}_t)}^2}{\sigma^2} + \frac{\|\theta_2 - \hat{\theta}_t^{\mathrm{c}}\|_{A(\tilde{\lambda}_t)}^2}{\gamma^2}  \right], $$  
    we observe that:
    \[
    \mathbb{E}_{(\theta_1, \theta_2) \sim \tilde{p}_t} \left[ \frac{\|\theta_1 - \hat{\theta}_t^{\mathrm{r}}\|_{A(\tilde{\lambda}_t)}^2}{\sigma^2} + \frac{\|\theta_2 - \hat{\theta}_t^{\mathrm{c}}\|_{A(\tilde{\lambda}_t)}^2}{\gamma^2}  \right]
    - \mathbb{E}_{(\theta_1, \theta_2) \sim \tilde{p}_t} \left[ \frac{\|\theta_1 - \hat{\theta}_t^{\mathrm{r}}\|_{x_t x_t^\top}^2}{\sigma^2} + \frac{\|\theta_2 - \hat{\theta}_t^{\mathrm{c}}\|_{x_t x_t^\top}^2}{\gamma^2}  \right]
    \]
    is a martingale difference sequence.

   When the good event \( E_{2, \delta} \) holds, we have:
\[
\left| \mathbb{E}_{(\theta_1, \theta_2) \sim \tilde{p}_t} \left[ \frac{\|\theta_1 - \hat{\theta}_t^{\mathrm{r}}\|_{A(\tilde{\lambda}_t)}^2}{\sigma^2} + \frac{\|\theta_2 - \hat{\theta}_t^{\mathrm{c}}\|_{A(\tilde{\lambda}_t)}^2}{\gamma^2}  \right]
- \mathbb{E}_{(\theta_1, \theta_2) \sim \tilde{p}_t} \left[  \frac{\|\theta_1 - \hat{\theta}_t^{\mathrm{r}}\|_{x_t x_t^\top}^2}{\sigma^2} + \frac{\|\theta_2 - \hat{\theta}_t^{\mathrm{c}}\|_{x_t x_t^\top}^2}{\gamma^2}  \right] \right|
\]
\[
\le 2 \left( \frac{(LR_1 + B_1)^2}{\sigma^2} + \frac{(LR_2 + B_2)^2}{\gamma^2} \right).
\]

Then, with probability at least \( 1 - \frac{2}{T} \) and \( \delta = \frac{1}{T} \), we have:
\begin{align*}
    F_4 &\le \sqrt{\frac{2 \left( \frac{(LR_1 + B_1)^2}{\sigma^2} + \frac{(LR_2 + B_2)^2}{\gamma^2} \right) \log T}{T}}  \le o(1).
\end{align*}

\end{proof}

\begin{lemma}
    \label{lem:Regret_F5}
    With probability $1-\frac{4}{T}$,
    \begin{align*}
        F_5 &:= \frac{1}{T} \sum_{t=1}^{T} \mathbb{E}_{(\theta_1, \theta_2) \sim \tilde{p}_t} 
        \left[  \frac{\|\theta_1 - \hat{\theta}_t^{\mathrm{r}}\|_{X_t X_t^\top}^2}{\sigma^2} 
        + \frac{\|\theta_2 - \hat{\theta}_t^{\mathrm{c}}\|_{X_t X_t^\top}^2}{\gamma^2}  \right] \\
        &\quad - \frac{1}{T} \inf_{(\theta_1, \theta_2) \in \overline{\Theta}_{z^*}} 
        \left( \frac{\|\theta_1 - \hat{\theta}^{\mathrm{r}}_{T + 1}\|_{V_{T}}^2}{\sigma^2} 
        + \frac{\|\theta_2 - \hat{\theta}^{\mathrm{c}}_{T + 1}\|_{V_{T}}^2}{\gamma^2} \right) 
        \le o(1).
    \end{align*}
\end{lemma}

\begin{proof}
    Since $\tilde{p}_t := \mathcal{N}(\hat{\theta}^{\mathrm{r}}_t, \eta_r^{-1} V_{t-1}^{-1}) \otimes\mathcal{N}(\hat{\thetac_t}, \eta_c^{-1} V_{t-1}^{-1} ) | \overline{\Theta}_{z^*}$, we define the normalization term at time $t$ as 
    \begin{align*}
        W_t:=\int_{(\theta_1, \theta_2) \in \overline{\Theta}_{z^*}}\exp\left(-\frac{1}{2}\left[\eta_r\lVert\theta_1-\hat{\theta}_t^\mathrm{r}\rVert^2_{V_{t-1}}+\eta_c\lVert\theta_2-\hat{\theta}_t^\mathrm{c}\rVert^2_{V_{t-1}}\right]\right) \,\mathrm{d}\theta_1 \mathrm{d}\theta_2
    \end{align*}

Then we have for $t>1$,
\begin{align*}
    &\log \frac{W_{t}}{W_{t-1}}:=\log \frac{\int_{(\theta_1, \theta_2) \in \overline{\Theta}_{z^*}}\exp\left(-\frac{1}{2}\left[\eta_r\lVert\theta_1-\hat{\theta}_{t}^\mathrm{r}\rVert^2_{V_{t-1}}+\eta_c\lVert\theta_2-\hat{\theta}_{t}^\mathrm{c}\rVert^2_{V_{t-1}}\right]\right) \,\mathrm{d}\theta_1 \mathrm{d}\theta_2}{W_{t-1}}
    \\&=\log \frac{\int_{(\theta_1, \theta_2) \in \overline{\Theta}_{z^*}}\exp\left(-\frac{1}{2}\left[\eta_r\lVert\theta_1-\hat{\theta}_{t}^\mathrm{r}\rVert^2_{V_{t-1}} +\eta_c\lVert\theta_2-\hat{\theta}_{t}^\mathrm{c}\rVert^2_{V_{t-1}}+w_{t-1}-w_{t-1}\right]\right) \,\mathrm{d}\theta_1 \mathrm{d}\theta_2}{W_{t-1}}
\end{align*}
where $w_t:=\eta_r\lVert\theta_1-\hat{\theta}_t^\mathrm{r}\rVert^2_{V_{t-1}}+\eta_c\lVert\theta_2-\hat{\theta}_t^\mathrm{c}\rVert^2_{V_{t-1}}$

Then
\begin{align*}
    \log \frac{W_{t}}{W_{t-1}}&=\log {\EE_{(\theta_1, \theta_2) \sim\tilde{p}_t}\exp\left(-\frac{1}{2}\left[\eta_r\lVert\theta_1-\hat{\theta}_{t}^\mathrm{r}\rVert^2_{V_{t-1}} +\eta_c\lVert\theta_2-\hat{\theta}_{t}^\mathrm{c}\rVert^2_{V_{t-1}}-w_{t-1}\right]\right) }
    \\&=\log \EE_{(\theta_1, \theta_2)\sim \tilde{p}_t}\exp \left(-\frac{1}{2}\Gamma_1+\frac{1}{2}\Gamma_2\right)
    \\&\stackrel{(a)}{\le}\frac{1}{2}\log \EE_{(\theta_1, \theta_2)\sim\tilde{p}_t}\exp \left(-\Gamma_1\right)+\frac{1}{2}\log \EE_{(\theta_1, \theta_2) \in \tilde{p}_t}\exp \left(\Gamma_2\right)
\end{align*}

where 
\begin{align*}
    \Gamma_1:&=\eta_r \| \hat{\theta}^\mathrm{r}_{t-1} - \theta_1 \|^2_{x_{t-1} x_{t-1}^\top}+ \eta_c \| \hat{\theta}^\mathrm{c}_{t-1} - \theta_2 \|^2_{x_{t-1} x_{t-1}^\top}
   \\ \Gamma_2:&=\eta_r \| \hat{\theta}^\mathrm{r}_{t-1} - \theta_1 \|_{V_{t-1}}^2 - \eta_r \| \hat{\theta}^\mathrm{r}_{t} - \theta_1 \|_{V_{t-1}}^2+\eta_c \| \hat{\theta}^\mathrm{c}_{t-1} - \theta_2 \|_{V_{t-1}}^2 - \eta_c \| \hat{\theta}^\mathrm{c}_{t} - \theta_2 \|_{V_{t-1}}^2
\end{align*}
and $(a)$  follows from the Cauchy--Schwarz inequality. Next,
\begin{align*}
    -{\| \hat{\theta}_{t-1}^\mathrm{r} - \theta_1 \|^2_{x_{t-1} x_{t-1}^\top}} &\leq - \| \truethetar - \theta_1 \|^2_{x_{t-1} x_{t-1}^\top} + 2 (\truethetar - \hat{\theta}_{t-1}^\mathrm{r})^\top (x_{t-1} x_{t-1}^\top) (\truethetar - \theta_1) \\
    &\overset{}{\leq} - \| \truethetar - \theta_1 \|^2_{x_{t-1} x_{t-1}^\top} + 2 \| \truethetar - \hat{\theta}_{t-1}^\mathrm{r} \|_{V_{t-2}} \| x_{t-1} \|_{V_{t-2}^{-1}} \| \truethetar - \theta_1 \|_{x_{t-1} x_{t-1}^\top} \\
    &\overset{}{\leq} - \| \truethetar - \theta_1 \|^2_{x_{t-1} x_{t-1}^\top} + 2 \| \truethetar - \hat{\theta}_{t-1}^\mathrm{r} \|_{V_{t-2}} \| x_{t-1} \|_{V_{t-2}^{-1}} D_1 \\
    &= - \| \truethetar - \theta_1 \|^2_{x_{t-1} x_{t-1}^\top} + a_1
\end{align*}
where $a_1 := 2 \| \truethetar - \hat{\theta}_{t-1}^\mathrm{r} \|_{V_{t-2}} \| x_{t-1} \|_{V_{t-2}^{-1}} D_1$ and $D_1:=\max_{x\in\mathcal{X}} x^\top (\theta_1-\truethetar)$.

On the other hand, since
\begin{align*}
	 &-{\left \| \truethetar - \theta_1 \right\|^2_{x_{t-1} x_{t-1}^\top}}   \leq  - \| \hat{\theta}_{t-1}^\mathrm{r} - \theta_1 \|^2_{x_{t-1} x_{t-1}^\top} - 2 (\truethetar - \hat{\theta}_{t-1}^\mathrm{r})^\top (x_{t-1} x_{t-1}^\top) (\hat{\theta}_{t-1}^\mathrm{r} - \theta_1) \\
	 &\leq - \| \hat{\theta}_{t-1}^\mathrm{r} - \theta_1 \|^2_{x_{t-1} x_{t-1}^\top}  + 2 \| \truethetar - \hat{\theta}_{t-1}^\mathrm{r} \|_{V_{t-2}} \| x_{t-1} \|_{V_{t-2}^{-1}} \| \hat{\theta}_{t-1}^\mathrm{r} - \theta_1 \|_{x_{t-1} x_{t-1}^\top}.
     \\&\leq - \| \hat{\theta}_{t-1}^\mathrm{r} - \theta_1 \|^2_{x_{t-1} x_{t-1}^\top}  + 2 \| \truethetar - \hat{\theta}_{t-1}^\mathrm{r} \|_{V_{t-2}} \| x_{t-1} \|_{V_{t-2}^{-1}}(\| \truethetar - \hat{\theta}_{t-1}^\mathrm{r} \|_{x_{t-1} x_{t-1}^\top} + \| \truethetar - \theta_1 \|_{x_{t-1} x_{t-1}^\top} )
     \\&\leq - \| \hat{\theta}_{t-1}^\mathrm{r} - \theta_1 \|^2_{x_{t-1} x_{t-1}^\top}  + 2 \| \truethetar - \hat{\theta}_{t-1}^\mathrm{r} \|_{V_{t-2}} \| x_{t-1} \|_{V_{t-2}^{-1}}( \| \truethetar - \hat{\theta}_{t-1}^\mathrm{r} \|_{x_{t-1} x_{t-1}^\top} + D_1)
     \\&\leq - \| \hat{\theta}_{t-1}^\mathrm{r} - \theta_1 \|^2_{x_{t-1} x_{t-1}^\top}  + 2 \| \truethetar - \hat{\theta}_{t-1}^\mathrm{r} \|_{V_{t-2}} \| x_{t-1} \|_{V_{t-2}^{-1}}(\| \truethetar - \hat{\theta}_{t-1}^\mathrm{r} \|_{V_{t-2}} \| x_{t-1} \|_{V_{t-2}^{-1}} + D_1)
     \\&=- \| \hat{\theta}_{t-1}^\mathrm{r} - \theta_1 \|^2_{x_{t-1} x_{t-1}^\top}+a_1+\frac{a_1^2}{2D_1^2}
\end{align*}
then
\begin{align*}
    -{\left \| \truethetar - \theta_1 \right\|^2_{x_{t-1} x_{t-1}^\top}}  \leq - \| \hat{\theta}_{t-1}^\mathrm{r} - \theta_1 \|^2_{x_{t-1} x_{t-1}^\top}+2a_1+\frac{a_1^2}{2D_1^2}.
\end{align*}

Similarly we also have
\begin{align*}
    -{\left \| \truethetac - \theta_2 \right\|^2_{x_{t-1} x_{t-1}^\top}}  \leq - \| \hat{\theta}_{t-1}^\mathrm{c} - \theta_2 \|^2_{x_{t-1} x_{t-1}^\top}+2a_2+\frac{a_2^2}{2D_2^2},
\end{align*}
where $a_2 := 2 \| \truethetac - \hat{\theta}_{t-1}^\mathrm{c} \|_{V_{t-2}} \| x_{t-1} \|_{V_{t-2}^{-1}} D_1$ and $ D_2:=\max_{x\in\mathcal{X}} x^\top (\theta_2-\truethetac)$.

Furthermore, from Corollary~1 in \citet{kone2024pareto}, when we choose
\begin{align}
    \eta_r\le \frac{1}{8L^2 R_1^2}\quad \mbox{and}\quad \eta_c\le\frac{1}{8L^2 R_2^2}, \nonumber
\end{align}
whence, $\exp(-\Gamma_1)$ is concave in terms of $(\theta_1, \theta_2)$.

Then
\begin{align*}
    \log \frac{W_t}{W_{t-1}}&\stackrel{}{\le}\frac{1}{2}\log \EE_{(\theta_1, \theta_2) \sim\tilde{p}_t}\exp \left(-\Gamma_1\right)+\frac{1}{2}\log \EE_{(\theta_1, \theta_2) \sim\tilde{p}_t}\exp \left(\Gamma_2\right)
    \\&\le \frac{1}{2}\left(-\EE_{(\theta_1, \theta_2) \sim \tilde{p}_t}\eta_r\| \hat{\theta}_{t-1}^\mathrm{r} - \theta_1 \|^2_{x_{t-1} x_{t-1}^\top}+\eta_c\| \hat{\theta}_{t-1}^\mathrm{c} - \theta_2 \|^2_{x_{t-1} x_{t-1}^\top}+\eta_r\Big(2a_1+\frac{a_1^2}{2D_1^2}\Big)+ \eta_c\Big(2a_2+\frac{a_2^2}{2D_2^2}\Big)\right)
    \\*
    &\qquad+\frac{1}{2}\log \EE_{(\theta_1, \theta_2) \in \tilde{p}_t}\exp \left(\Gamma_2\right)
\end{align*}

Then by telescoping
\begin{align}
    \log \frac{W_{T+1}}{W_1}&=\sum_{t=2}^{T+1} \log \frac{W_{t}}{W_{t-1}}
   \nonumber\\&\le \frac{1}{2}\sum_{t=2}^{T+1}\left(-\EE_{(\theta_1, \theta_2) \sim \tilde{p}_t}\eta_r\| \hat{\theta}_{t-1}^\mathrm{r} - \theta_1 \|^2_{x_{t-1} x_{t-1}^\top} \!+\!\eta_c\| \hat{\theta}_{t-1}^\mathrm{c} \!- \!\theta_2 \|^2_{x_{t-1} x_{t-1}^\top}\!+\!\eta_r\Big(2a_1\!+\!\frac{a_1^2}{2D_1^2}\Big)\!+\! \eta_c\Big(2a_2\!+\!\frac{a_2^2}{2D_2^2}\Big)\right) \nonumber\\* 
&\qquad +\frac{1}{2}\sum_{t=2}^{T+1}\log \EE_{(\theta_1, \theta_2) \sim \tilde{p}_t}\exp \left(\Gamma_2\right)\label{eq:logless2}
\end{align}

{
    \let\OldGamma\gamma
    \renewcommand{\gamma}{z}
On the other hand, let $\gamma > 0$, and $(\tilde{\theta}_T^\mathrm{r}, \tilde{\theta}_T^\mathrm{c}) = \arg\min_{(\theta_1, \theta_2) \in \overline{\Theta}_{z^*}} 
        \left( \eta_r{\|\theta_1 - \hat{\theta}^{\mathrm{r}}_{T + 1}\|_{V_{T}}^2} 
        + \eta_c {\|\theta_2 - \hat{\theta}^{\mathrm{c}}_{T + 1}\|_{V_{T}}^2}\right)$.

Given that \( \overline{\Theta}_{z^*} \) constitutes a union of convex sets, there exists a convex subset \( \mathcal{C} \subseteq \overline{\Theta}_{z^*} \) such that \( (\tilde{\theta}_T^\mathrm{r}, \tilde{\theta}_T^\mathrm{c}) \in \mathcal{C} \) and $\text{vol}(\mathcal{C})>0$.

Next, define the set \( \mathcal{N}_\gamma \) as follows:
\[
\mathcal{N}_\gamma := \big\{ (1-\gamma)(\tilde{\theta}_T^\mathrm{r}, \tilde{\theta}_T^\mathrm{c}) + \gamma (\theta_1, \theta_2) \mid (\theta_1, \theta_2) \in \mathcal{C} \big\} = (1-\gamma)(\tilde{\theta}_T^\mathrm{r}, \tilde{\theta}_T^\mathrm{c}) + \gamma \mathcal{C},
\]
which represents a convex combination of the point \( (\tilde{\theta}_T^\mathrm{r}, \tilde{\theta}_T^\mathrm{c}) \) and elements from \( \mathcal{C} \), parameterized by \( \gamma \in [0, 1] \).

Hence,
\begin{align*}
    \log\frac{W_{T+1}}{W_{1}} &\geq  \log \frac{\int_{\mathcal{C}} \exp\left( - \frac{\eta_r\lVert\theta_1-\hat{\theta}_{T+1}^\mathrm{r}\rVert^2_{V_{T}}+\eta_c\lVert\theta_2-\hat{\theta}_{T+1}^\mathrm{c}\rVert^2_{V_{T}}}{2} \right) \,\mathrm{d}\theta_1 \,\mathrm{d}\theta_2}{W_{1}} \\
    &\geq \log \frac{\int_{\mathcal{N}_\gamma} \exp\left( - \frac{\eta_r\lVert\theta_1-\hat{\theta}_{T+1}^\mathrm{r}\rVert^2_{V_{T}}+\eta_c\lVert\theta_2-\hat{\theta}_{T+1}^\mathrm{c}\rVert^2_{V_{T}}}{2} \right) \,\mathrm{d}\theta_1\, \mathrm{d}\theta_2}{W_{1}} \\
    &= \log \frac{\int_{\gamma \mathcal{C}} \exp\left( - \frac{\eta_r\lVert\hat{\theta}_{T+1}^\mathrm{r}-(1-\gamma)\tilde{\theta}_T^\mathrm{r}-\theta_1\rVert^2_{V_{T}}+\eta_c\lVert\hat{\theta}_{T+1}^\mathrm{c}-\tilde{\theta}_T^\mathrm{c}-\theta_2\rVert^2_{V_{T}}}{2} \right) \,\mathrm{d}\theta_1\,\mathrm{d}\theta_2}{W_{1}} \\
    &= \log \frac{\int_{\mathcal{C}} \gamma^{2d} \exp\left( - \frac{\eta_r \| (1-\gamma)(\hat{\theta}_{T+1}^\mathrm{r}-\tilde{\theta}_T^\mathrm{r}) + \gamma \hat{\theta}_{T+1}^\mathrm{r} - \gamma \theta_1 \|^2_{V_T}+\eta_c \| (1-\gamma)(\hat{\theta}_{T+1}^\mathrm{c}-\tilde{\theta}_T^\mathrm{c}) + \gamma \hat{\theta}_{T+1}^\mathrm{c} - \gamma \theta_2 \|^2_{V_T}}{2} \right) \,\mathrm{d}\theta_1\, \mathrm{d}\theta_2}{W_{1}} \\
    &\geq \log \frac{\int_{\mathcal{C}} \gamma^{2d} \exp\left( - \frac{\eta_r \| (1-\gamma)(\hat{\theta}_{T+1}^\mathrm{r}-\tilde{\theta}_T^\mathrm{r}) \|^2_{V_T}+\gamma \lVert\hat{\theta}_{T+1}^\mathrm{r} - \theta_1 \rVert^2_{V_T} + \eta_c \| (1-\gamma)(\hat{\theta}_{T+1}^\mathrm{c}-\tilde{\theta}_T^\mathrm{c})  \|^2_{V_T}+\gamma \lVert\hat{\theta}_{T+1}^\mathrm{c} - \theta_2\rVert^2_{V_T}}{2} \right) \,\mathrm{d}\theta_1\, \mathrm{d}\theta_2}{W_{1}} \\
    &=2d\log(\gamma) - \frac{\eta_r(1-\gamma)}{2} \| \hat{\theta}_{T+1}^\mathrm{r} - \tilde{\theta}_T^\mathrm{r} \|_{V_T}^2 -\frac{\eta_c(1-\gamma)}{2} \| \hat{\theta}_{T+1}^\mathrm{c} - \tilde{\theta}_T^\mathrm{c} \|_{V_T}^2  \\ 
    &\qquad+ \log \frac{\int_{\mathcal{C}} \exp\left( - \frac{\eta_r \gamma \| \theta_1 - \hat{\theta}_{T+1}^\mathrm{r} \|^2_{V_T}+\eta_c \gamma \| \theta_2 - \hat{\theta}_{T+1}^\mathrm{c} \|^2_{V_T}}{2} \right) \,\mathrm{d}\theta_1 \,\mathrm{d}\theta_2}{W_{1}} \\
    &\geq 2d\log(\gamma) - \frac{\eta_r(1-\gamma)}{2} \| \hat{\theta}_{T+1}^\mathrm{r} - \tilde{\theta}_T^\mathrm{r} \|_{V_T}^2 -\frac{\eta_c(1-\gamma)}{2} \| \hat{\theta}_{T+1}^\mathrm{c} - \tilde{\theta}_T^\mathrm{c} \|_{V_T}^2 \\* 
    &\qquad- \EE_{(\theta_1, \theta_2)\sim \text{Uni}(\mathcal{C})} \frac{\eta_r \gamma \| \theta_1 - \hat{\theta}_{T+1}^\mathrm{r} \|^2_{V_T}+\eta_c \gamma \| \theta_2 - \hat{\theta}_{T+1}^\mathrm{c} \|^2_{V_T}}{2}  + \log \frac{\text{vol}(\mathcal{C})}{W_{1}}.
\end{align*}
Then with Eqn.~\eqref{eq:logless2}, we have
\begin{align*}
   &\frac{1}{2}\sum_{t=1}^T \EE_{(\theta_1, \theta_2)\sim \tilde{p}_t}\left(\eta_r\| \hat{\theta}_{t}^\mathrm{r} - \theta_1 \|^2_{x_{t} x_{t}^\top}+\eta_c\| \hat{\theta}_{t}^\mathrm{c} - \theta_2 \|^2_{x_{t} x_{t}^\top}\right)-\frac{1-\gamma}{2}\inf_{(\theta_1, \theta_2) \in \overline{\Theta}_{z^*}} 
        \left( \eta_r{\|\theta_1 - \hat{\theta}^{\mathrm{r}}_{T + 1}\|_{V_{T}}^2} 
        + \eta_c {\|\theta_2 - \hat{\theta}^{\mathrm{c}}_{T + 1}\|_{V_{T}}^2}\right)
        \\&\le \frac{1}{2}\sum_{t=2}^{T+1}\left[\eta_r\Big(2a_1+\frac{a_1^2}{2D_1^2}\Big)+ \eta_c\Big(2a_2+\frac{a_2^2}{2D_2^2}\Big)\right]+\EE_{(\theta_1, \theta_2)\sim \text{Uni}(\mathcal{C})} \frac{\eta_r \gamma \| \theta_1 - \hat{\theta}_{T+1}^\mathrm{r} \|^2_{V_T}+\eta_c \gamma \| \theta_2 - \hat{\theta}_{T+1}^\mathrm{c} \|^2_{V_T}}{2}
        \\&-2d\log \gamma -\log \frac{\text{vol}(\mathcal{C})}{W_{1}}+\frac{1}{2}\sum_{t=2}^{T+1}\log \EE_{(\theta_1, \theta_2) \in \tilde{p}_t}\exp \left(\Gamma_2\right).
\end{align*}
  \let\gamma\OldGamma
}

Then with the choice $z=\frac{1}{T}$, 
\begin{align*}
     \eta_r:=\frac{\eta}{\sigma^2},\quad
    \eta_c:=\frac{\eta}{\gamma^2}\quad\mbox{and}\quad
    \eta:=\min \cbr{\frac{\sigma^2}{8L^2R_1^2},\frac{\gamma^2}{8L^2R_2^2}}.
\end{align*}
From Lemma~12 in \cite{kone2024pareto},  we have with probability $1-\frac{4}{T}$,
\begin{align*}
    \frac{1}{2}\sum_{t=2}^{T+1}\left[\eta_r \Big(2a_1+\frac{a_1^2}{2D_1^2}\Big)+ \eta_c\Big(2a_2+\frac{a_2^2}{2D_2^2}\Big)\right] &=O(\sqrt{T\log^2 T}),\quad\mbox{and}\\
      \frac{1}{2}\sum_{t=2}^{T+1}\log \EE_{(\theta_1, \theta_2)\sim\tilde{p}_t}\exp \left(\Gamma_2\right)&=O(\sqrt{T\log^2 T})~.
\end{align*}

Also, under good event $E_{1, \frac{1}{T}}$, we have 
\begin{align*}
     \EE_{(\theta_1, \theta_2)\sim \text{Uni}(\mathcal{C})} \frac{\eta_r \gamma \| \theta_1 - \hat{\theta}_{T+1}^\mathrm{r} \|^2_{V_T}+\eta_c \gamma \| \theta_2 - \hat{\theta}_{T+1}^\mathrm{c} \|^2_{V_T}}{2}&=O(1),\quad\mbox{and}\\
    -2d\log z=2d\log T&=O(\log T).
\end{align*}

Since $\overline{\Theta}_{z^*}$ is finite and $\text{vol}(\mathcal{C})>0$, observe from   the definition of $W_t$ that $W_1\le \text{vol}(\overline{\Theta}_{z^*})$. Thus,   $-\log \frac{\text{vol}(\mathcal{C})}{W_{1}}=O(1)$.

To summarize, with probability $1-\frac{4}{T}$,
\begin{align*}
        F_5 &:= \frac{1}{T} \sum_{t=1}^{T} \mathbb{E}_{(\theta_1, \theta_2) \sim \tilde{p}_t} 
        \left[ \frac{\|\theta_1 \!-\! \hat{\theta}_t^{\mathrm{r}}\|_{X_t X_t^\top}^2}{\sigma^2} 
       \! + \!\frac{\|\theta_2 \!-\! \hat{\theta}_t^{\mathrm{c}}\|_{X_t X_t^\top}^2}{\gamma^2} \right]   - \frac{1}{T} \inf_{(\theta_1, \theta_2) \in \overline{\Theta}_{z^*}} 
        \left( \frac{\|\theta_1 \! -\! \hat{\theta}^{\mathrm{r}}_{T + 1}\|_{V_{T}}^2}{\sigma^2} 
        \!+\! \frac{\|\theta_2 \!-\! \hat{\theta}^{\mathrm{c}}_{T + 1}\|_{V_{T}}^2}{\gamma^2} \right) \\
        &\le O\Big(\sqrt{\frac{\log^2 T}{T}}\Big) = o(1).
    \end{align*}

\end{proof}

\begin{lemma}
\label{lem:Regret_F6}
With probability \( 1 - \frac{4}{T} \),
\begin{align*}
F_6 
&:= \frac{1}{T} \inf_{(\theta_1, \theta_2) \in \overline{\Theta}_{z^*}} 
\left( \frac{\|\theta_1 - \hat{\theta}_{T+1}\|_{V_{T}}^2}{\sigma^2} 
+ \frac{\|\theta_2 - \hat{\theta}_{T+1}\|_{{V}_{T}}^2}{\gamma^2} \right)   - \inf_{(\theta_1, \theta_2) \in \overline{\Theta}_{z^*}} 
\left( \frac{\|\theta_1 - \thetar\|_{\overline{V}_T}^2}{\sigma^2} 
+ \frac{\|\theta_2 - \thetac\|_{\overline{V}_T}^2}{\gamma^2} \right) \le o(1).
\end{align*}
\end{lemma}

\begin{proof}
    Denote \( (\theta_3, \theta_4) := \arg\min_{(\theta_1, \theta_2) \in \overline{\Theta}_{z^*}} \left( \frac{\|\theta_1 - \thetar\|_{V_T}^2}{\sigma^2} 
    + \frac{\|\theta_2 - \thetac\|_{V_T}^2}{\gamma^2} \right) \). Then,
    \begin{align*}
        F_6 &= \frac{1}{T} \left[ \inf_{(\theta_1, \theta_2) \in \overline{\Theta}_{z^*}} 
        \left( \frac{\|\theta_1 - \hat{\theta}_{T+1}\|_{V_T}^2}{\sigma^2} + \frac{\|\theta_2 - \hat{\theta}_{T+1}\|_{{V}_T}^2}{\gamma^2} \right) \right. \\*
        &\quad \left. - \inf_{(\theta_1, \theta_2) \in \overline{\Theta}_{z^*}} 
        \left( \frac{\|\theta_1 - \thetar\|_{{V}_T}^2}{\sigma^2} + \frac{\|\theta_2 - \thetac\|_{{V}_T}^2}{\gamma^2} \right) \right] \\
        &\le \frac{1}{T} \left[ \left( \frac{\|\theta_3 - \hat{\theta}_{T+1}\|_{V_T}^2}{\sigma^2} 
        + \frac{\|\theta_4 - \hat{\theta}_{T+1}\|_{{V}_T}^2}{\gamma^2} \right) - \left( \frac{\|\theta_3 - \thetar\|_{{V}_T}^2}{\sigma^2} 
        + \frac{\|\theta_4 - \thetac\|_{{V}_T}^2}{\gamma^2} \right) \right] \\
        &= \frac{1}{T} \Bigg[ \frac{(\lVert \theta_3 - \hat{\theta}_{T+1} \rVert_{V_T} - \lVert \theta_3 - \thetar \rVert_{V_T})(\lVert \theta_3 - \hat{\theta}_{T+1} \rVert_{V_T} + \lVert \theta_3 - \thetar \rVert_{V_T})}{\sigma^2} \\
        &\quad + \frac{(\lVert \theta_4 - \hat{\theta}_{T+1} \rVert_{V_T} - \lVert \theta_4 - \thetac \rVert_{V_T})(\lVert \theta_4 - \hat{\theta}_{T+1} \rVert_{V_T} - \lVert \theta_4 + \thetac \rVert_{V_T})}{\gamma^2} \Bigg] \\
        &= \frac{1}{T} \Bigg[ \frac{\lVert \thetar - \hat{\theta}_{T+1} \rVert_{V_T} (\lVert \theta_3 - \hat{\theta}_{T+1} \rVert_{V_T} + \lVert \theta_3 - \thetar \rVert_{V_T})}{\sigma^2} \\
        &\quad + \frac{\lVert \thetac - \hat{\theta}_{T+1} \rVert_{V_T} (\lVert \theta_4 - \hat{\theta}_{T+1} \rVert_{V_T} - \lVert \theta_4 + \thetac \rVert_{V_T})}{\gamma^2} \Bigg] \\
        &\stackrel{(a)}{\le} \frac{1}{T} \Bigg[ \frac{\lVert \thetar - \hat{\theta}_{T+1} \rVert_{V_T} (3 L R_1 + B_1) \sqrt{T}}{\sigma^2} \\
        &\quad + \frac{\lVert \thetac - \hat{\theta}_{T+1} \rVert_{V_T} (3 L R_2 + B_2) \sqrt{T}}{\gamma^2} \Bigg] \\
        &\stackrel{(b)}{\le} \frac{1}{\sqrt{T}} \Bigg[ \frac{\sqrt{\beta_1(T+1, \frac{1}{\delta^2})}(3 L R_1 + B_1)}{\sigma^2} \\
        &\quad + \frac{\sqrt{\beta_2(T+1, \frac{1}{\delta^2})}(3 L R_2 + B_2)}{\gamma^2} \Bigg]
    \end{align*}
    where (a) comes from the good event \( E_{2, \delta} \), and (b) comes from Lemmas~\ref{Ellip_potential_reward} and~\ref{Ellip_potential_cost}.

    Then, with probability \( 1 - 4\delta \) and with the choice \( \delta = \frac{1}{T} \),
    \begin{align*}
        F_6 &\le \frac{1}{\sqrt{T}} \Bigg[ \frac{\sqrt{\beta_1(T+1, T^2)}(3 L R_1 + B_1)}{\sigma^2}   + \frac{\sqrt{\beta_2(T+1, T^2)}(3 L R_2 + B_2)}{\gamma^2} \Bigg]  = o(1).
    \end{align*}
\end{proof}

\section{Useful Lemmas}
In this section we present some useful existing lemmas for our proof.
\begin{lemma}[Theorem~2 in \citet{abbasi2011improved}]
    \label{Ellip_potential_reward}
    For all $t>0$, with probability $1-\delta$,
    \begin{align*}
        \lVert \hthetart-\thetar  \rVert_{V_{t-1}}\le \sqrt{\beta_1(t, \frac{1}{\delta^2})}
    \end{align*} 
    where $\beta_1(t,\frac{1}{\delta^2}):= (S_1 + \sigma\sqrt{2 \log(\frac{1}{\delta^2}) + d \log\left(\frac{d + t L^2}{d}\right)})^2~.$
\end{lemma}

\begin{lemma}[Theorem~2 in \citet{abbasi2011improved}]
    \label{Ellip_potential_cost}
    For all $t>0$, with probability $1-\delta$,
    \begin{align*}
        \lVert \hthetact-\thetac  \rVert_{V_{t-1}}\le \sqrt{\beta_2(t, \frac{1}{\delta^2})}
    \end{align*}    
    where $\beta_2(t,\frac{1}{\delta^2}):= (S_2 + \gamma\sqrt{2 \log(\frac{1}{\delta^2}) + d \log\left(\frac{d + t L^2}{d}\right)})^2~.$
\end{lemma}

\begin{lemma}[Laplace Approximation]
\label{lem:GaussianApproximation}
Denote $\overline{V}_T:=I+\frac{1}{T} \sum_{t=1}^{T} X_t X_t^\top$, $\text{for any bounded open sets } \Theta_1, \Theta_2 \subseteq \mathbb{R}^d, \text{ we have}$
\[
\int_{\Theta_1} \int_{\Theta_2} \exp\left( -\frac{T}{2} (\frac{\left\| \thetar - \theta_1 \right\|^2_{\overline{V}_T}}{\sigma^2}+\frac{\left\| \thetac - \theta_2 \right\|^2_{\overline{V}_T}}{\gamma^2}) \right) \,\mathrm{d}\theta_1 \, \mathrm{d}\theta_2 
\overset{.}{=} \exp\left( -\frac{T}{2} \inf_{\theta_1 \in \Theta_1, \theta_2 \in \Theta_2} \frac{\left\| \thetar - \theta_1 \right\|^2_{\overline{V}_T}}{\sigma^2}+\frac{\left\| \thetac - \theta_2 \right\|^2_{\overline{V}_T}}{\gamma^2} \right).
\]

\end{lemma}

\begin{proof}
    First, since $\Theta_1$ and $\Theta_2$ are bounded, 
    \begin{align}
        \mathrm{LHS}\le \text{Vol}(\Theta_1) \cd \text{Vol}(\Theta_2) \exp\left( -\frac{T}{2} \inf_{\theta_1 \in \Theta_1, \theta_2 \in \Theta_2} \frac{\left\| \truethetar - \theta_1 \right\|^2_{\overline{V}_T}}{\sigma^2}+\frac{\left\| \truethetac - \theta_2 \right\|^2_{\overline{V}_T}}{\gamma^2} \right).\nonumber
    \end{align}
    Second, define $(\theta^*_1, \theta^*_2):=\arg\min
_{(\theta_1,\theta_2) \in \Theta_1\times \Theta_2} \frac{\left\| \truethetar - \theta_1 \right\|^2_{\overline{V}_T}}{\sigma^2}+\frac{\left\| \truethetac - \theta_2 \right\|^2_{\overline{V}_T}}{\gamma^2} $

For any $\theta_1\in \Theta_1, \theta_2 \in \Theta_2$, since
\begin{align*}
    &\Bigg| \left(\frac{\left\| \truethetar - \theta_1 \right\|^2_{\overline{V}_T}}{\sigma^2} 
    + \frac{\left\| \truethetac - \theta_2 \right\|^2_{\overline{V}_T}}{\gamma^2}\right) 
    - \left(\frac{\left\| \truethetar - \theta^*_1 \right\|^2_{\overline{V}_T}}{\sigma^2} 
    + \frac{\left\| \truethetac - \theta^*_2 \right\|^2_{\overline{V}_T}}{\gamma^2}\right) \Bigg| \\
    &\le \frac{1}{\sigma^2} \Big| \left\| \truethetar - \theta_1 \right\|^2_{\overline{V}_T} 
    - \left\| \truethetar - \theta^*_1 \right\|^2_{\overline{V}_T} \Big| 
    + \frac{1}{\gamma^2} \Big| \left\| \truethetac - \theta_2 \right\|^2_{\overline{V}_T} 
    - \left\| \truethetac - \theta^*_2 \right\|^2_{\overline{V}_T} \Big| \\
    &\le \frac{4L^2 R_1}{\sigma^2} \lVert \theta_1 - \theta^*_1 \rVert_2 
    + \frac{4L^2 R_2}{\gamma^2} \lVert \theta_2 - \theta^*_2 \rVert_2
\end{align*}
then when $\lVert \theta_1-\theta^*_1\rVert_2\le \epsilon_1, \lVert \theta_2-\theta^*_2\rVert_2\le \epsilon_2$,
\begin{align*}
    &\Bigg| \left( \frac{\left\| \truethetar - \theta_1 \right\|^2_{\overline{V}_T}}{\sigma^2} 
    + \frac{\left\| \truethetac - \theta_2 \right\|^2_{\overline{V}_T}}{\gamma^2} \right) 
    - \left( \frac{\left\| \truethetar - \theta^*_1 \right\|^2_{\overline{V}_T}}{\sigma^2} 
    + \frac{\left\| \truethetac - \theta^*_2 \right\|^2_{\overline{V}_T}}{\gamma^2} \right) \Bigg| \\
    &\le \frac{4L^2 R_1}{\sigma^2} \epsilon_1 
    + \frac{4L^2 R_2}{\gamma^2} \epsilon_2.
\end{align*}
Define $\Theta^*_1:=\cbr{\theta_1\mid \lVert \theta_1-\theta^*_1\rVert\le \epsilon_1}, \Theta^*_2:=\cbr{\theta_2\mid \lVert \theta_2-\theta^*_2\rVert\le \epsilon_2}$, then
\begin{align*}
    \mathrm{LHS}&\ge \int_{\Theta^*_1} \int_{\Theta^*_2} \exp\left( -\frac{T}{2} (\frac{\left\| \truethetar - \theta_1 \right\|^2_{\overline{V}_T}}{\sigma^2}+\frac{\left\| \truethetac - \theta_2 \right\|^2_{\overline{V}_T}}{\gamma^2}) \right) \,\mathrm{d}\theta_1 \, \mathrm{d}\theta_2 
    \\&\ge \text{Vol}(\Theta^*_1) \cd \text{Vol}(\Theta^*_2) \exp \bigg(-\frac{T}{2} \Big(\frac{4L^2 R_1}{\sigma^2} \epsilon_1+ \frac{4L^2 R_2}{\gamma^2} \epsilon_2\Big)\bigg) \exp\Big(-\frac{T}{2}\inf_{\theta_1 \in \Theta_1, \theta_2 \in \Theta_2} \frac{\left\| \truethetar - \theta_1 \right\|^2_{\overline{V}_T}}{\sigma^2}+\frac{\left\| \truethetac - \theta_2 \right\|^2_{\overline{V}_T}}{\gamma^2}\Big)
\end{align*}
Then when we choose $\epsilon_1 \to 0, \epsilon_2 \to 0$, we can have
\begin{align*}
    \int_{\Theta_1} \int_{\Theta_2} \exp\left( -\frac{T}{2} (\frac{\left\| \thetar - \theta_1 \right\|^2_{\overline{V}_T}}{\sigma^2}+\frac{\left\| \thetac - \theta_2 \right\|^2_{\overline{V}_T}}{\gamma^2}) \right) \,\mathrm{d}\theta_1 \, \mathrm{d}\theta_2 
\overset{.}{=} \exp\left( -\frac{T}{2} \inf_{\theta_1 \in \Theta_1, \theta_2 \in \Theta_2} \frac{\left\| \thetar - \theta_1 \right\|^2_{\overline{V}_T}}{\sigma^2}+\frac{\left\| \thetac - \theta_2 \right\|^2_{\overline{V}_T}}{\gamma^2} \right).
\end{align*}
\end{proof}

\section{More Empirical Plots}
\label{sec:empirical_plots}
\subsection{More plots for "End of Optimism" Instance} \label{sec:eoo}

\begin{figure}[H]
    \centering
    \includegraphics[width=.75\linewidth]{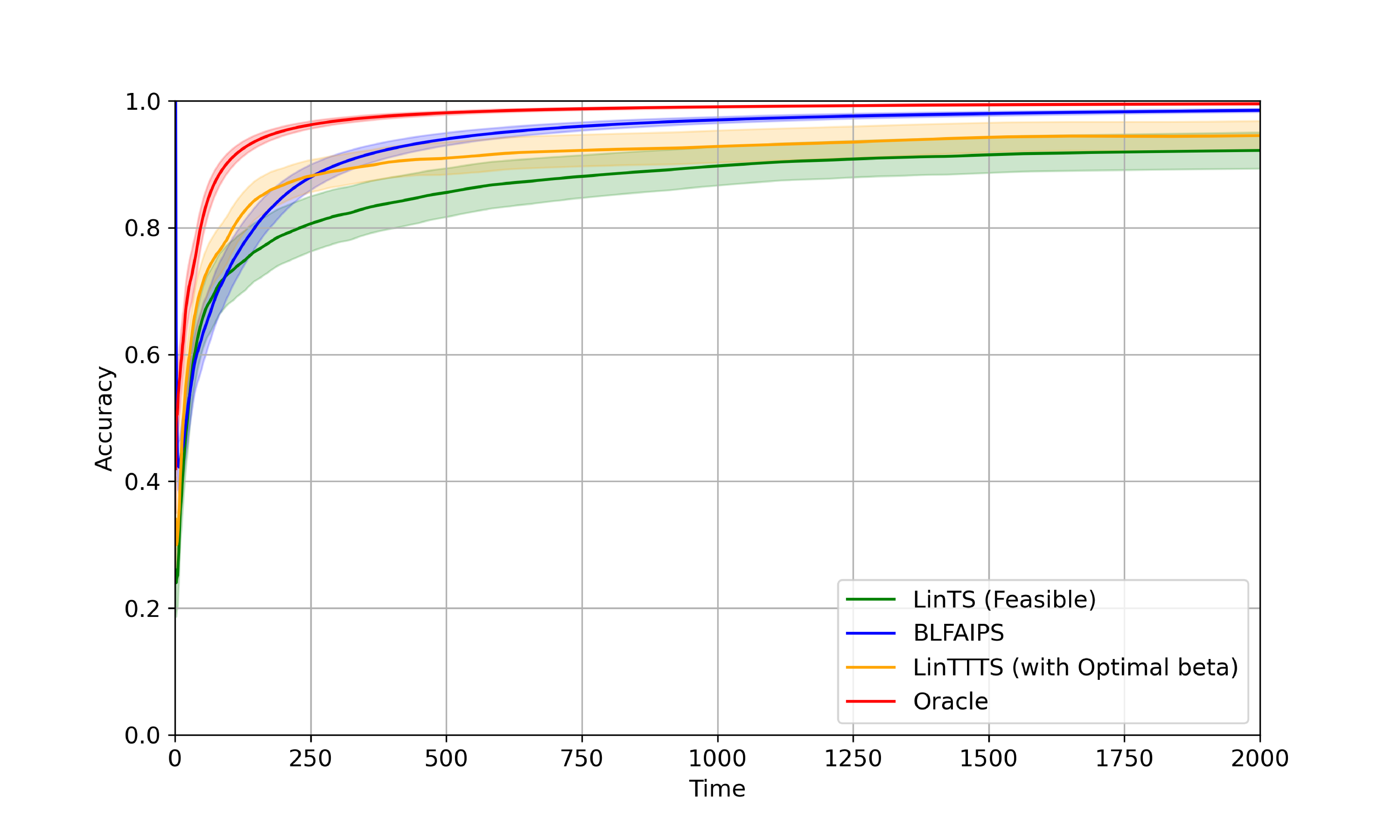}
    
    \caption{$\alpha=0.2$}
    
    \label{fig:2}
    
    \includegraphics[width=.75\linewidth]{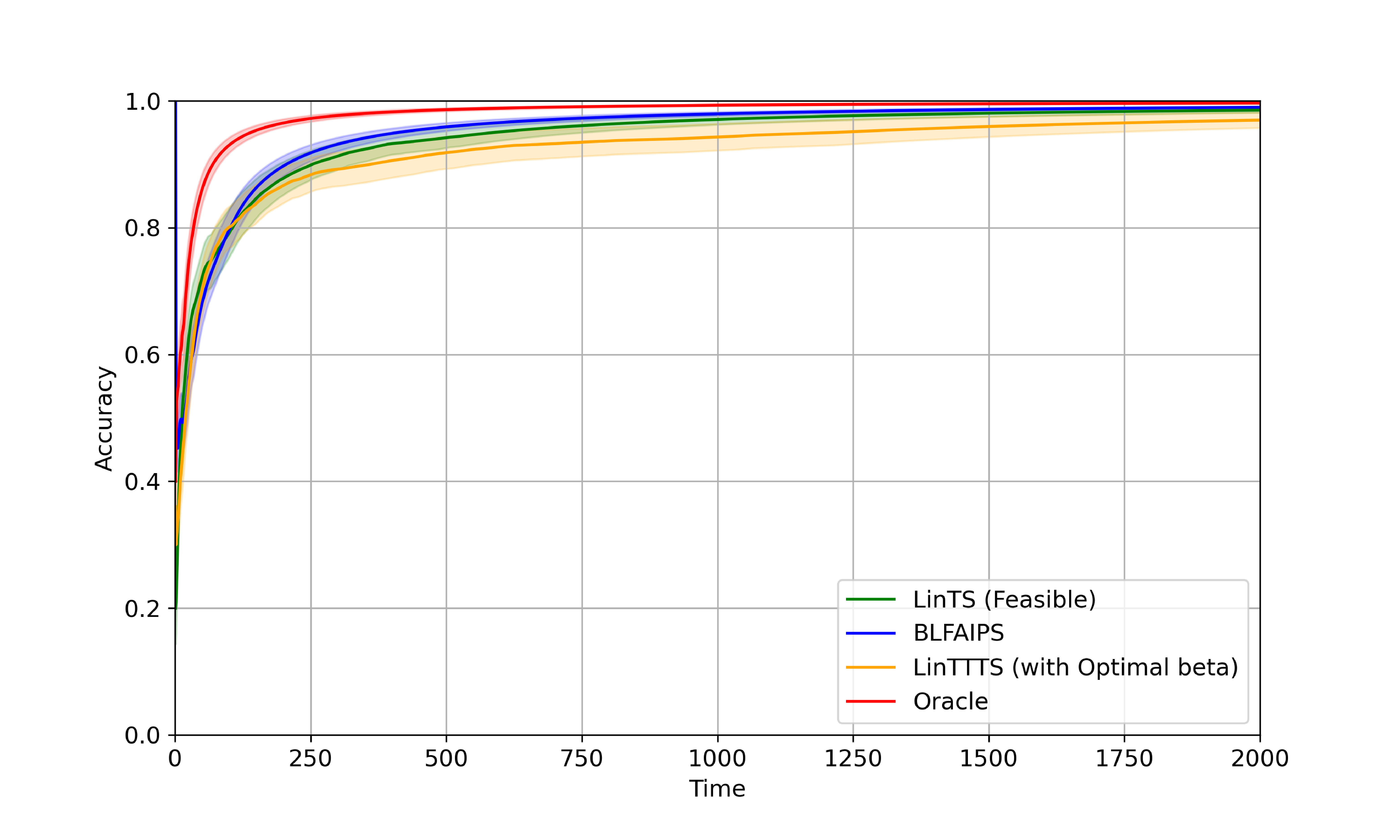}
    
    \caption{$\alpha=0.3$}
    \label{fig:3}

\end{figure}

\subsection{More plots for Random Instance}\label{sec:random}
\begin{figure}[H]
    \centering
    \includegraphics[width=.75\linewidth]{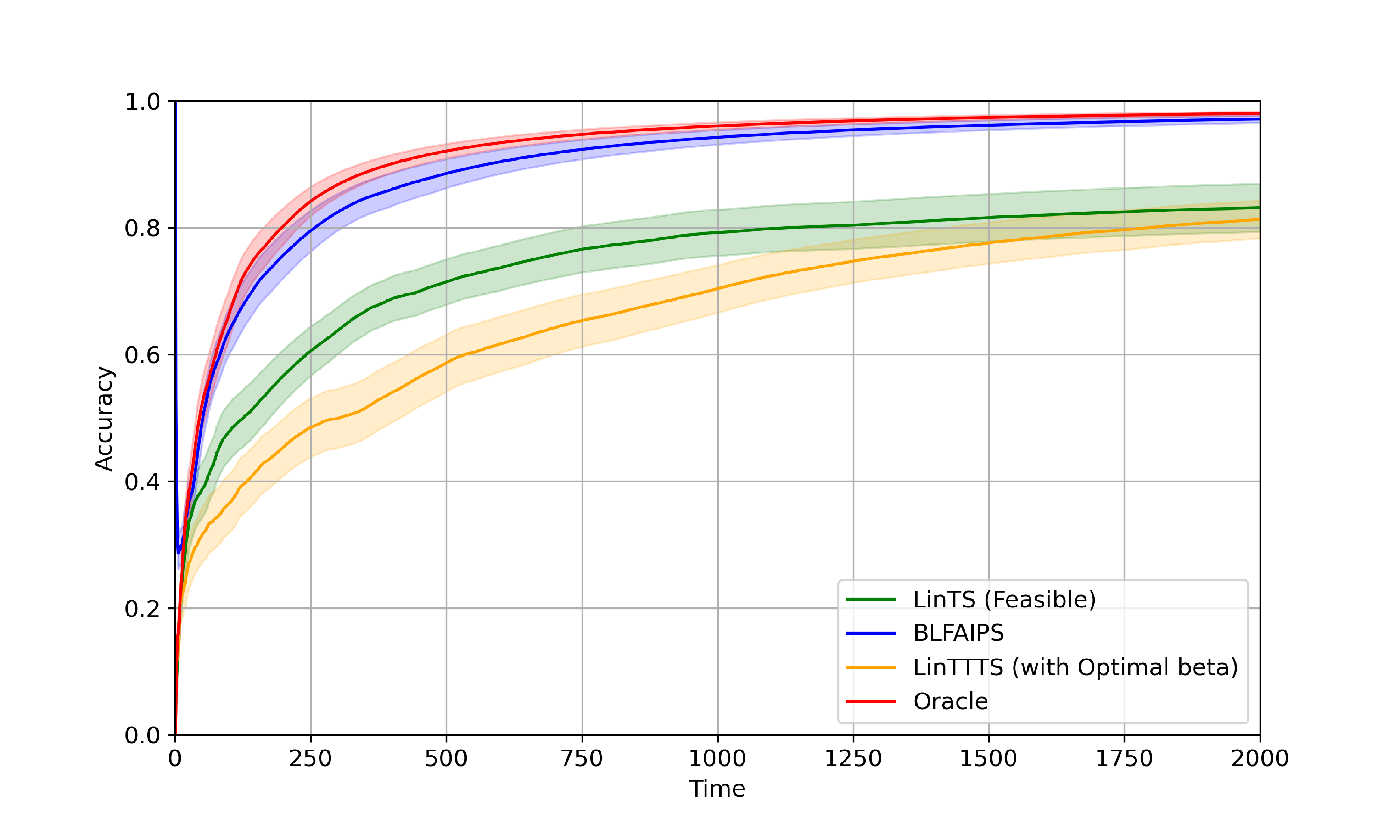}
    \caption{$d=2, K=20$}
    \label{fig:d2K20}
    \includegraphics[width=.7\linewidth]{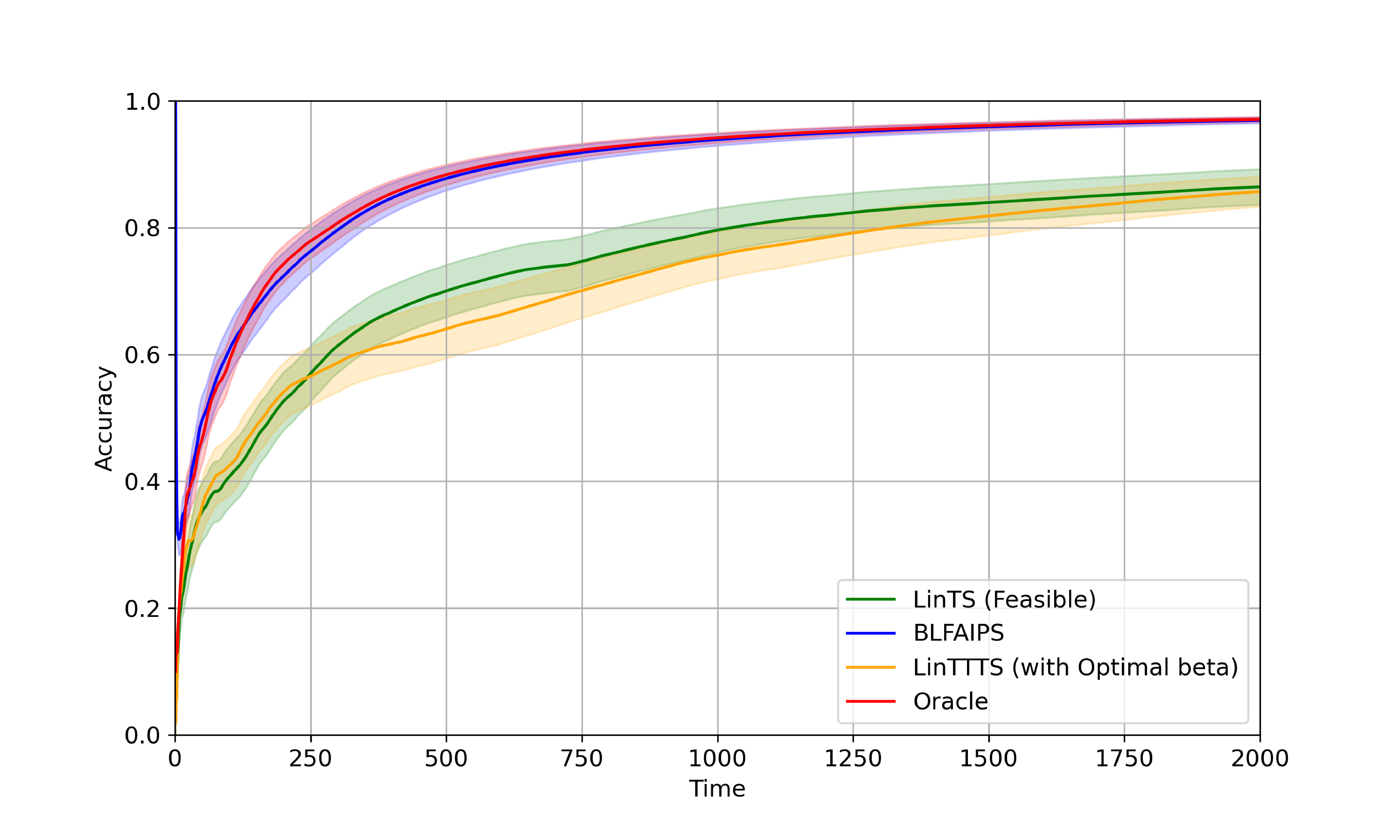}
    \caption{$d=20, K=20$}
    \label{fig:d20K20}

\end{figure}
\end{document}